\documentclass[10pt]{article}

\usepackage{comment}
\usepackage{cite}
\usepackage{amsmath,amssymb,amsfonts}
\usepackage{graphicx}
\usepackage{textcomp}
\usepackage{xcolor}
\usepackage{xspace}
\usepackage{float}
\usepackage{pifont}
\usepackage[inline]{enumitem}
\usepackage{tikz}
\usetikzlibrary{shapes}
\usepackage{amsthm}
\usepackage{algorithm,algpseudocode}
\usepackage{dsfont}
\usepackage{subfigure}
\usepackage{wrapfig}
\usepackage{lipsum}
\usepackage{framed}
\usepackage{booktabs}
\usepackage{soul}
\usepackage{authblk}
\usepackage[square,sort,comma,numbers]{natbib}
\usepackage[inline]{enumitem}
\newcommand{\cmark}{\ding{51}}%
\newcommand{\xmark}{\ding{55}}%

\newcommand{\compactcaption}[1]{\vspace{-0.5em}\caption{#1}\vspace{-0.5em}}
\usepackage[compact, medium]{titlesec}
\titlespacing*{\section}{1pt}{3.5pt}{2pt}
\titlespacing*{\subsection}{1pt}{3pt}{1.5pt}
\titlespacing*{\subsubsection}{1pt}{3pt}{1.5pt}

\usepackage{etoolbox}

\makeatletter
\patchcmd{\ttlh@hang}{\parindent\z@}{\parindent\z@\leavevmode}{}{}
\patchcmd{\ttlh@hang}{\noindent}{}{}{}
\makeatother

\usepackage{hyperref} % SHOULD BE LOADED LAST
\hypersetup{colorlinks = true}

\title{Exploring Connections Between Active Learning and Model Extraction}
\author[1]{Varun Chandrasekaran\thanks{Corresponding author: chandrasekaran@cs.wisc.edu}}
\author[3]{Kamalika Chaudhuri}
\author[2]{Irene Giacomelli}
\author[1]{Somesh Jha}
\author[3]{Songbai Yan}
\affil[1]{University of Wisconsin-Madison}
\affil[2]{Protocol Labs}
\affil[3]{University of California San Diego}
\date{\today}

\begin{document}
\maketitle
\newcommand{\myparagraph}[1]{\vspace{0.1in}\noindent\textbf{#1}}
\newcommand{\myparnovspace}[1]{\vspace{0.0in}\noindent\textbf{#1}}
\newcommand{\myvspacenopar}{\vspace{0.1in}\noindent }

\newtheorem{prop}{Proposition}
\newtheorem{fact}{Fact}
\newtheorem{lem}{Lemma}
\newtheorem*{notation*}{Notation}
\newtheorem*{rem*}{Remark}
\newtheorem{rem}{Remark}
\newcommand{\One}{\mathds{1}}

\newtheorem{theorem}{Theorem}
\newtheorem{proposition}{Proposition}
\newtheorem{lemma}{Lemma}

\theoremstyle{definition}
\newtheorem{definition}{Definition}
\newtheorem{experiment}{Experiment}

\theoremstyle{remark}
\newtheorem{example}{Example}
\newtheorem{remark}{Remark}

\newcommand{\ie}{\textit{i.e.}\@\xspace}
\newcommand{\eg}{\textit{e.g.}\@\xspace}
\newcommand{\etal}{\textit{et al.}\@\xspace}
\newcommand{\R}{\ensuremath{\mathbb{R}}}
\newcommand{\Z}{\ensuremath{\mathbb{Z}}}
\newcommand{\E}{\ensuremath{\mathbb{E}}}
\newcommand{\D}{\ensuremath{\mathcal{D}}}
\newcommand{\Q}{\ensuremath{\mathbb{Q}}}
\newcommand{\Y}{\ensuremath{\mathbf{Y}}}
\newcommand{\X}{\ensuremath{\mathbf{X}}}
\newcommand{\F}{\ensuremath{\mathcal{F}}}
\renewcommand{\L}{\ensuremath{\mathcal{L}}}
\renewcommand*{\O}{\ensuremath{\mathcal{O}}}
\newcommand{\A}{\ensuremath{\mathcal{A}}}
\newcommand{\Ser}{\ensuremath{S}}
\newcommand{\Exp}{\ensuremath{\textup{\texttt{Exp}}}}
\newcommand{\uErr}{\ensuremath{\textup{Err}_u}}
\newcommand{\gErr}{\ensuremath{\textup{Err}_2}}
\newcommand{\Err}{\ensuremath{\textup{Err}}}
\newcommand{\sign}{\ensuremath{\textup{sign}}}
\newcommand{\negl}{\ensuremath{\textup{negl}}}
\newcommand{\tr}{\top}
\newcommand{\at}{\makeatletter @\makeatother}
\newcommand{\qquote}[1]{``#1''}

\newcommand{\Irene}[1]{\textcolor{blue}{[\textbf{Irene}: #1]}}
\newcommand{\Varun}[1]{\textcolor{green}{[\textbf{Varun}: #1]}}
\newcommand{\Somesh}[1]{\textcolor{purple}{[\textbf{Somesh}: #1]}}
\newcommand{\kc}[1]{\textcolor{orange}{[\textbf{Kamalika}: #1]}}
\newcommand{\Songbai}[1]{\textcolor{brown}{[\textbf{Songbai}: #1]}}
\newcommand{\todo}[1]{\textcolor{red}{[\textbf{ToDo}: #1]}}

%%% Local Variables:
%%% mode: latex
%%% TeX-master: "main"
%%% End:

\begin{abstract}	

Machine learning is being increasingly used by individuals, research
institutions, and corporations. This has resulted in the surge of
Machine Learning-as-a-Service (MLaaS) - cloud services that provide
(a) tools and resources to learn the model, and (b) a user-friendly
query interface to access the model. However, such MLaaS systems raise
privacy concerns such as \emph{model extraction}. In model
extraction attacks, adversaries maliciously exploit the query
interface to {\em steal} the model. More precisely, in a model
extraction attack, a good approximation of a sensitive or proprietary
model held by the server is extracted (\ie learned) by a dishonest
user who interacts with the server only via the query interface. This
attack was introduced by Tram\`{e}r \etal at the 2016 USENIX Security
Symposium, where practical attacks for various models were shown. We
believe that better understanding the efficacy of model extraction
attacks is paramount to designing secure MLaaS systems. To that end,
we take the first step by (a) formalizing model extraction and
discussing possible defense strategies, and (b) drawing parallels
between model extraction and established area of \emph{active
  learning}. In particular, we show that recent advancements in the
active learning domain can be used to implement powerful model extraction
attacks and investigate possible defense strategies.

\end{abstract}

\section{Introduction}

Advancements in various facets of machine learning has made it an
integral part of our daily life. However, most real-world machine
learning tasks are resource intensive. To that end, several cloud
providers, such as Amazon, Google, Microsoft, and BigML offset the
storage and computational requirements by providing {\em Machine
  Learning-as-a-Service (MLaaS)}. A MLaaS server offers support for
both the training phase, and a query interface for accessing the
trained model. The trained model is then queried by other users on
chosen instances (refer Fig.~\ref{fig:1}). Often, this is implemented
in a pay-per-query regime \ie the server, or the model owner via the
server, charges the the users for the queries to the model. Pricing
for popular MLaaS APIs is given in Table~\ref{aux}. 

Current research is focused at improving the performance of training
algorithms and of the query interface, while little emphasis is placed
on the related security aspects. For example, in many real-world
applications, the trained models are privacy-sensitive - a model can
{\sf (a)} leak sensitive information about training data
\cite{AMSVVF15} during/after training, and {\sf (b)} can itself have
commercial value or can be used in security applications that assume
its secrecy (\eg, spam filters, fraud detection
etc. \cite{LM05,HJNRT11,SL14}). To keep the models private, there has
been a surge in the practice of \emph{oracle access}, or black-box
access. Here, the trained model is made available for prediction but
is kept secret. MLaaS systems use oracle access to balance the
trade-off between privacy and usability.

\begin{table}[H]
\centering
\small\addtolength{\tabcolsep}{-1pt}
\begin{tabular}{p{2.5cm}p{1.5cm}p{1.5cm}p{1.5cm}}
\toprule
{\bf Models}	  &{\bf Google} & {\bf Amazon} & {\bf Microsoft} \\
\midrule
		  {\bf \textbullet \  DNNs}  & {\footnotesize Confidence Score} & \xmark & {\footnotesize Confidence Score}\\
		  {\bf \textbullet \  Regression} & {\footnotesize Confidence Score} & {\footnotesize Confidence Score} & {\footnotesize Confidence Score}\\
		  {\bf \textbullet \  Decision trees} & {\footnotesize Leaf Node} & \xmark & {\footnotesize Leaf Node} \\
		  {\bf \textbullet \  Random forests} & {\footnotesize Leaf Node} & \xmark & {\footnotesize Leaf Node} \\
		  {\bf \textbullet \ Binary $\&$ n-ary classification} & {\footnotesize Confidence Score} & {\footnotesize Confidence Score} & {\footnotesize Confidence Score}\\ 
		  {\bf } & & & \\
\midrule 
		  {\bf \textbullet \  Batch}    & $\$0.093^*$        & $\$0.1$        &  $\$0.5$  \\
		  {\bf \textbullet \  Online}   & $\$0.056^*$        & $\$0.0001$        &  $\$0.0005$   \\ 
\bottomrule
\end{tabular}
\compactcaption{\footnotesize Pricing, and auxiliary information shared. $*$ Google’s pricing model is per node per hour. Leaf node denotes the exact leaf (and not an internal node) where the computation halts, and \xmark indicates the absence of support for the associated model.}
\label{aux}
\end{table}

%: the machine learning model is widely accessible for usage, yet the model itself remains confidential.

Despite providing oracle access, a broad suite of attacks continue to
target existing MLaaS systems \cite{brendel2017decision}. For
example, membership inference attacks attempt to determine if a given
data-point is included in the model's training dataset only by
interacting with the MLaaS interface (\eg
\cite{shokri2017membership}). In this work, we focus on {\em model
  extraction attacks}, where an adversary makes use of the MLaaS query
interface in order to {\em steal} the proprietary model (\ie learn the
model or a good approximation of it). In an interesting paper,
Tram\`{e}r \etal \cite{TZJRR16}, show that many commonly used MLaaS
interfaces can be exploited using only few queries to recover a
model's secret parameters. Even though model extraction attacks are
empirically proven to be feasible, their work consider interfaces that
reveal auxiliary information, such as confidence values together with
the prediction output. Additionally, their work does not formalize
model extraction. We believe that such formalization is paramount for
designing secure MLaaS that are resilient to aforementioned
threats. In this paper, we take the first step in this direction. 
The main contributions of the paper appear in boldfaced captions.

\noindent
{\bf Model Extraction $\approx$ Active Learning.}  The key observation
guiding our formalization is that the process of model extraction is
very similar to \emph{active learning} \cite{settlesactive}, a special
case of semi-supervised machine learning. An active learner learns an
approximation of a labeling function $f^*$ through repetitive
interaction with an oracle, who is assumed to know $f^*$. These
interactions typically involve the learner sending an instance $x$ to
the oracle, and the oracle returning the label $y=f^*(x)$ to the
learner. Since the learner can choose the instances to be labeled, the
number of data-points needed to learn the labeling function is often
much lower than in the normal supervised case.  Similarly, in model
extraction, the adversary uses a strategy to query a MLaaS server with
the following goals: {\sf (a)} to successfully steal (\ie learn) the
model (\ie labeling function) known by the server (\ie oracle), in
such a way as to {\sf (b)} minimize the number of queries made to the
MLaaS server, as each query costs the adversary a fixed dollar
value.

While the overall process of active learning mirrors the
general description of model extraction, the entire spectrum of active
learning can not be used to study model extraction. Indeed, some
scenarios (eg, PAC active learning) assume that the query instances
are sampled from the actual input distribution. However, an attacker
is not restricted to such a condition and can query any instance. For
this reason, we believe that the query synthesis framework of
active learning, where the learner has the power to generate arbitrary
query instances best replicates the capabilities of the adversary in
the model extraction framework. Additionally, the query synthesis
scenario ensures that we make no assumptions about the adversary's
prior knowledge.

%THE PREVIOUS PARAGRAPH MODIFIED AS:

%\textcolor{blue}{While the overall process of active learning mirrors the general description of model extraction, the entire spectrum of active learning can not be used to study model extraction without making assumption on the adversary's prior knowledge.  Indeed, some scenarios (\eg, PAC active learning) assume that the query instances are sampled from the actual input distribution. Active learning algorithms in these scenarios can still used to produce attack, but only for adversaries that know the original distribution used to learn target model. However, in a more general  setting an attacker is not restricted to such condition and can query any instance. For this reason, we believe that the {\em query synthesis} framework of active learning, where the learner has power to generate arbitrary query instances best illustrate model extraction in its generality (\ie where no assumptions about the adversary's prior knowledge are made).}

\noindent
{\bf Powerful attacks with no auxiliary information.}  By casting
model extraction as query synthesis active learning, we are able to
draw concrete similarities between the two. Consequently, we are able
to use algorithms and techniques from the active learning community to
perform powerful model extraction attacks, and investigate possible
defense strategies. In particular, we show that: query synthesis
active learning algorithms can be used to perform model extraction on both
linear and non-linear classifiers with {\em no auxiliary information}.  Moreover, our
evaluation shows that our attacks are better than the classic attacks,
such as by Lowd and Meek~\cite{LM05}, which have been widely used in
the security community. Our approach based on active selection also improves upon existing approaches to extract kernel SVMs (see Section~\ref{sec:implementation}).

\noindent
{\bf No ``free lunch'' for defense.}  Simple defense strategies
such as changing the prediction output with constant and small
probability are not effective. However, defense strategies that
change the prediction output depending on the instances that are being
queried, such as the work of Alabdulmohsin \etal \cite{CIKM-14}, are
more robust to extraction attacks implemented using existing query
synthesis active learning algorithms. However, in
Algorithm~\ref{alg:passive} of Section~\ref{sec:implementation}, we
show that this defense is not secure against traditional passive
learning algorithms. This suggests that there is ``no free lunch'' --
accuracy might have to be sacrificed to prevent model extraction. An
in-depth investigation of such a result will be interesting avenue for
future work.

% More precisely, we have the following goals in mind:
% \begin{enumerate}
% \item Give a formal definition of model extraction;
% \item Draw concrete similarities between model extraction and active learning;
% \item Use known lower bounds for query complexity from the active learning literature to classify which models are "hard to steal";
% \item Use known upper bounds (and algorithm to achieve them) from the active learning literature to steal models in an efficient way (using only membership queries).
% \end{enumerate}

\begin{figure} 
	\begin{framed}
	\vspace*{-.3cm}
		\begin{center}
			\begin{tikzpicture}
			%\node [fill=gray!10, cloud, draw=gray, thick, align=center, cloud puffs=15, cloud puff arc=120, aspect=2] (cloud) {MLaaS Server\\(oracle) with $f^*$};
			\node (cloud) {\includegraphics[height=1cm]{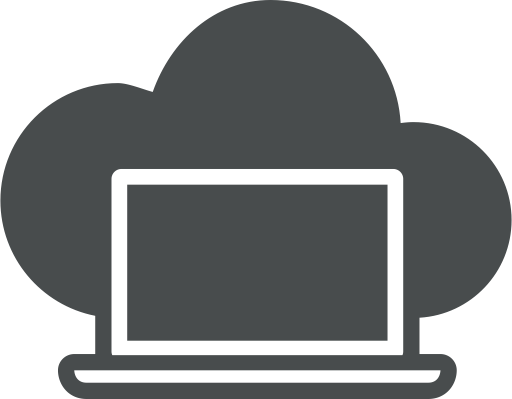}};
			\node[below of=cloud, align=center, node distance=1.1cm]{MLaaS Server\\(oracle)};
			\node[left of=cloud, node distance=3.3cm](owner){\includegraphics[height=1cm]{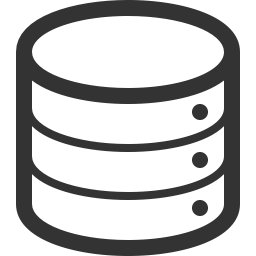}};
			\node[above of=owner, align=right]{Data owner};
			\node[right of=cloud, node distance=3.3cm](user){\includegraphics[height=1.5cm]{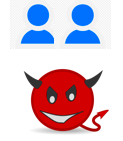}};
			\node[above of=user, align=center]{};
			\draw[-latex]([yshift=.2cm]user.west) -- node[above] {$x_1,\dots,x_q$} ([yshift=.2cm]cloud.east);		
			\draw[-latex]([yshift=-.4cm]cloud.east) -- node[above] {$y_1,\dots,y_q$} ([yshift=-.4cm]user.west);
			
			\draw[-latex](cloud.west) -- node[above] {training} (owner.east);		
			\draw[-latex](owner.east) -- node[above] {} (cloud.west);
			\end{tikzpicture}
		\end{center}
	\vspace*{-.6cm}
	\end{framed}
	\vspace*{-.2cm}
	\compactcaption{\footnotesize Model extraction can be envisioned as active
          learning. A data owner, with the help of a MLaaS server,
          trains a model $f^*$ on its data. The proprietary model is
          stored by the server, which also answers to queries from
          users (\ie, $y_i=f^*(x_i)$). In a model extraction attack, a dishonest user tries to
          exploit this interface to \qquote{steal} $f^*$ in the same
          way as a learner uses answer from a machine-learning oracle
          in order to learn $f^*$. }
	\label{fig:1}
\end{figure}

\noindent
\emph{Paper structure.}  We begin with a brief comparison between
passive machine learning and active learning in
Section~\ref{def:AL}. This allows us to introduce the notation used in
this paper, and review the state-of-the-art for active
learning. Section~\ref{sec:ME} focuses on the formalization of model
extraction attacks, casting it into the query synthesis active
learning framework. Section~\ref{sec:non-linear} discusses our algorithms used to extract non-linear classifiers (\ie kernel SVMs, decision trees, and random forests). Section~\ref{sec:defense} discusses possible defenses
strategies. Section~\ref{subsec:linear} reports our experimental
findings and demonstrates that query synthesis active learning can be
used to successfully perform model extraction, and evaluates different
defense strategies. Specifically, we observe that \$0.09 worth Amazon
queries are needed to extract most halfspaces when the MLaaS server
does not deploy any defense, and \$3.65 worth of queries are required
to learn a halfspace when it uses data-independent
randomization. Furthermore, our experiments in Section~\ref{subsec:nonlinear} show modifying adaptive retraining proposed by Tram\`{e}r \etal results in efficienct extraction attacks for non-linear models; we obtain 5$\times$-224$\times$ improvement for kernel SVMs, and comparable extraction efficiency for discrete models such as decision trees with {\em no auxiliary information}. Finally, we discuss some open issues in Section~\ref{sec:discussion}, which provides avenue for future work. Related work is discussed in Section~\ref{sec:related}, and we end the paper with some concluding remarks.

%Finally, we discuss some open issues in Section~\ref{sec:discussion}, which provides avenue for future work.

%%% OLD figure 
%\begin{figure}
%	\begin{center}
%		\begin{tikzpicture}
%		\node [fill=gray!10, cloud, draw=gray, thick, align=center, cloud puffs=15, cloud puff arc=120, aspect=2] (cloud) {MLaaS Server\\(oracle) with $f^*$};
%		\node[right of=cloud, node distance=5.5cm](user){\includegraphics[scale=0.7]{Images/user.png}};
%		\node[below of=user, align=center]{User\\(learner)};
%		\draw[-latex]([yshift=.2cm]user.west) -- node[above] {$x$} ([yshift=.2cm]cloud.east);		
%		\draw[-latex]([yshift=-.4cm]cloud.east) -- node[above] {$f^*(x)$} ([yshift=-.4cm]user.west);
%		\end{tikzpicture}
%	\end{center}
%	\compactcaption{\footnotesizeModel extraction can be envisioned as active learning: an active learner (the adversary) issues queries to the MLaaS server (the oracle) to extract the hidden model.}
%	\label{fig:1}
%\end{figure}

\section{Machine Learning}
\label{sec:ML}

In this section, we give a brief overview of machine learning, and
terminology we use throughout the paper. In particular, we summarize
the passive learning framework in subsection~\ref{sec:PL}, and focus
on active learning algorithms in subsection~\ref{sec:AL}.  A review of
the state-of-the-art of active learning algorithms is needed to
explicitly link model extraction to active learning and is presented
in Section~\ref{sec:ME}.

\subsection{Passive learning}
\label{sec:PL}
In the standard, passive machine learning setting, the learner has
access to a large labeled dataset and uses it in its entirety to learn
a predictive model from a given class. Let $\X$ be an instance space,
and $\Y$ be a set of labels. For example, in object recognition, $\X$
can be the space of all images, and $\Y$ can be a set of objects that
we wish to detect in these images. We refer to a pair
$(x,y)\in\X\times\Y$ as a {\em data-point or labeled instance} ($x$ is
the instance, $y$ is the label). Finally, there is a class of
functions $\F$ from $\X$ to $\Y$ called the \emph{hypothesis space}
that is known in advance. The learner's goal is to find a function
$\hat{f}\in\F$ that is a good predictor for the label $y$ given the
instance $x$, with $(x,y)\in\X\times\Y$. To measure how well $\hat{f}$
predicts the labels, a loss function $\ell$ is used. Given a
data-point $z=(x,y)\in\X\times\Y$, $\ell(\hat{f},z)$ measures the
``difference'' between $\hat{f}(x)$ and the true label $y$. When the
label domain $\Y$ is finite (classification problem), the $0$-$1$ loss
function is frequently used:
\begin{equation*}
\ell(\hat{f},z)=\begin{cases}
0, &\text{ if } \hat{f}(x)=y\\
1, &\text{ otherwise}
\end{cases}
\end{equation*}
If the label domain $\Y$ is continuous, one can use the square loss:
$\ell(\hat{f},z)=(\hat{f}(x)-y)^2$.

In the passive setting, the PAC (\emph{probably approximately
  correct}) learning \cite{valiant1984theory} framework is
predominantly used. Here, we assume that there is an underlying distribution
$\D$ on $\X \times \Y$ that describes the data; the learner has no
direct knowledge of $\D$ but has access to a set of training data $D$
drawn from it. The main goal in passive PAC learning is to use the
labeled instances from $D$ to produce a hypothesis
%\footnote{Used interchangeably with predictive model, predictor, and labeling function.} 
$\hat{f}$ such that its expected loss with respect to the probability
distribution $\D$ is low.  This is often measured through the
\emph{generalization error} of the hypothesis $\hat{f}$, defined by

\begin{equation} 
\Err_\D(\hat{f})=\E_{z\sim \D} [\ell(\hat{f},z)] \label{eq:gen_error}
\end{equation} 

More precisely, we have the following definition.
\begin{definition}[\underline{PAC passive learning \cite{valiant1984theory}}]
\label{def:passive}
An algorithm $A$ is a PAC passive learning algorithm for the
hypothesis class $\F$ if the following holds for any $\D$ on
$\X\times\Y$ and any $\varepsilon,\delta\in (0,1)$: If $A$ is given
$s_A(\varepsilon,\delta)$ i.i.d. data-points generated by $\D$, then
$A$ outputs $\hat{f}\in\F$ such that $\Err_\D(\hat{f})\leq
\min_{f\in\F}\Err_\D(f)+ \varepsilon$ with probability at least
$1-\delta$. We refer to $s_A(\varepsilon,\delta)$ as the \emph{sample
  complexity} of algorithm $A$.
\end{definition}

\begin{remark}[\underline{Realizability assumption}]
\label{rem:realizability}
In the general case, the labels are given together with the instances, and the factor $\min_{f\in\F}\Err_\D(f)$ depends on the hypothesis class. Machine learning literature refers to this as \emph{agnostic learning} or the non-separable case of PAC learning. However, in some applications, the labels themselves can be described using a labeling function $f^*\in \F$. In this case (known as \emph{realizable learning}), $\min_{f\in\F}\Err_\D(f)=0$ and the distribution $\D$ can be described by its marginal over $\X$. 
A  PAC passive learning algorithm $A$ in the realizable case takes $s_A(\varepsilon,\delta)$ i.i.d. instances generated by $\D$ and the corresponding labels generated using $f^*$, and outputs $\hat{f}\in\F$ such that $\Err_\D(\hat{f})\leq \varepsilon$ with probability at least $1-\delta$.
\end{remark}

%%%%%%%%%%%%%%%%%%%%%%%%%%%%%%%%%%%%%%%%%%%%%%%
\subsection{Active learning}
\label{sec:AL}
%%%%%%%%%%%%%%%%%%%%%%%%%%%%%%%%%%%%%%%%%%%%%%%
In the passive setting, learning an accurate model (\ie learning
$\hat{f}$ with low generalization error) requires a large number of
data-points. Thus, the labeling effort required to produce an accurate
predictive model may be prohibitive. In other words, the sample
complexity of many learning algorithms grows rapidly as $\varepsilon
\rightarrow 0$ (refer Example~\ref{ex:hs}). This has spurred interest
in learning algorithms that can operate on a smaller set of labeled
instances, leading to the emergence of \emph{active learning}. In active
learning, the learning algorithm is allowed to select a subset of
unlabeled instances, query their corresponding labels from an
annotator (a.k.a oracle) and then use it to construct or update a
model. How the algorithm chooses the instances varies widely. However,
the common underlying idea is that by actively choosing the
data-points used for training, the learning algorithm can
drastically reduce sample complexity.

Formally, an active learning algorithm is an interactive process
between two parties - the oracle $\O$ and the learner $\L$. The only
interaction allowed is through \emph{queries} - $\L$ chooses $x\in \X$
and sends it to $\O$, who responds with $y\in\Y$ (\ie, the oracle
returns the label for the chosen unlabeled instance). This value of
$(x,y)$ is then used by $\L$ to infer some information about the
labeling procedure, and to choose the next instance to query. Over
many such interactions, $\L$ outputs $\hat{f}$ as a predictor for
labels. We can use the generalization error~\eqref{eq:gen_error} to
evaluate the accuracy of the output $\hat{f}$. However, depending on
the query strategy chosen by $\L$, other types of error can be used.

There are two distinct scenarios for active learning: \emph{PAC active
  learning} and \emph{Query Synthesis (QS) active learning}.  In literature, QS active learning is also known as \emph{Membership Query
  Learning}, and we will use the two terms synonymously. 

\subsubsection{PAC active learning}
\label{sec:activePAC}
This scenario was introduced by Dasgupta in 2005
\cite{dasgupta2005coarse} in the realizable context and then
subsequently developed in following works (\eg,
\cite{angluin1987learning,DHM07,H07}).  In this scenario, the instances are sampled
according to the marginal of $\D$ over $\X$, and the learner, after
seeing them, decides whether to query for their labels or not. Since
the data-points seen by $\L$ come from the actual underlying
distribution $\D$, the accuracy of the output hypothesis $\hat{f}$ is
measured using the generalization error~\eqref{eq:gen_error}, as in
the classic (\ie, passive) PAC learning.

There are two options to consider for sampling data-points. In {\em
  stream-based sampling} (also called \emph{selective sampling}) , the
instances are sampled one at a time, and the learner decides whether
to query for the label or not on a per-instance
basis. \emph{Pool-based sampling} assumes that all of the instances
are collected in a static pool $S\subseteq \X$ and then the learner
chooses specific instances in $S$ and queries for their
labels. Typically, instances are chosen by $\L$ in a greedy fashion
using a metric to evaluate all instances in the pool. This is not
possible in stream-based sampling, where $\L$ goes through the data
sequentially, and has to therefore make decisions to query
individually. Pool-based sampling is extensively studied since it has
applications in many real-world problems, such as text classification,
information extraction, image classification and retrieval,
etc. \cite{McCallumN98}. Stream-based sampling represents scenarios
where obtaining unlabeled data-points is easy and cheap, but obtaining
their labels is expensive (\eg, stream of data is collected by a
sensor, but the labeling needs to be performed by an expert).

Before describing query synthesis active learning, we wish to
highlight the advantage of PAC active learning over passive PAC
learning (\ie the reduced sample complexity) for some hypothesis class
through Example~\ref{ex:hs}. Recall that this advantage comes from the
fact that an active learner is allowed to adaptively choose the data
from which it learns, while a passive learning algorithm learns from a
static set of data-points.

\begin{example}[\underline{PAC learning for halfspaces}]
\label{ex:hs}
Let $\F_{d,HS}$ be the hypothesis class of \emph{$d$-dimensional halfspaces}, used for binary classification. A function in $f_w\in\F_{d,HS}$ is described by a normal vector $w\in\R^d$ (\ie, $||w||_2=1$) and is defined by 
\begin{equation*}
f_w(x)=\sign(\langle w, x\rangle) \text{ for any }x\in\R^d
\end{equation*} 
where given two vectors $a,b\in\R^d$, then their product is defined as
$\langle a, b\rangle=\sum_{i=1}^{d}a_ib_i$. Moreover, if $x\in\R$,
then $\sign(x)=1$ if $x\geq 0$ and $\sign(x)=-1$ otherwise. A classic
result in passive PAC learning states that
$O(\frac{d}{\varepsilon}\log(\frac{1}{\varepsilon})+\frac{1}{\varepsilon}\log(\frac{1}{\delta}))$
data-points are needed to learn $f_w$
\cite{valiant1984theory}. On the other hand, several works propose
active learning algorithms for $\F_{d,HS}$ with sample
complexity\footnote{The $\tilde{O}$ notation ignores logarithmic factors
  and terms dependent on $\delta$.}
$\tilde{O}(d\log(\frac{1}{\varepsilon}))$ (under certain
distributional assumptions). For example, if the underlying
distribution is log-concave, there exists an active learning algorithm
with sample complexity $\tilde{O}(d \log (\frac{1}{\varepsilon}))$
\cite{BalcanBZ07,BalcanL13,ZC14}.  This general reduction in the
sample complexity for $\F_{d,HS}$ is easy to infer when $d=1$. In this
case, the data-points lie on the real line and their labels are a
sequence of $-1$'s followed by a sequence of $+1$'s. The goal is to
discover a point $w$ where the change from $-1$ to $+1$ happens. PAC
learning theory states that this can be achieved with
$\tilde{O}(\frac{1}{\varepsilon})$\footnote{More generally,
  $\tilde{O}(\frac{d}{\varepsilon})$ points.} points i.i.d.\ sampled
from $\D$. On the other hand, an active learning algorithm that uses a
simple binary search can achieve the same task with
$O(\log(\frac{1}{\varepsilon}))$ queries \cite{dasgupta2005coarse}
(refer Figure~\ref{fig:bs}).
\end{example} 

\begin{figure} 
	\begin{framed}
	\begin{equation*}
	f_w(x)=\begin{cases}
	-1 \text{ if } \langle w, x\rangle < -1\\
	+1 \text{ otherwise}
	\end{cases}
	\end{equation*}
	\begin{center}
		\begin{tikzpicture}
		\draw[-latex] (-3,0) -- (3,0) node[below right] {$\R$};
		\filldraw
		(-2.5,0) circle (1.5pt) node[font=\footnotesize,above] {$-1$} --
		(-1.3,0) circle (1.5pt) node[font=\footnotesize,above] {$-1$} --
		(-0.8,0) circle (1.5pt) node[font=\footnotesize,above] {$-1$}--
		(0,0) circle (1.5pt) node[font=\footnotesize,above] {$+1$} --
		(.7,0) circle (1.5pt) node[font=\footnotesize,above] {$+1$} --
		(1.2,0) circle (1.5pt) node[font=\footnotesize,above] {$+1$}--
		(2,0) circle (1.5pt) node[font=\footnotesize,above] {$+1$}--
		(2.5,0) circle (1.5pt) node[font=\footnotesize,above] {$+1$};
		\draw (-.5,.2) -- (-.5,-.2) node[anchor=north,fill=white] {$w^*$};
		%\node[above right=10pt of {(1,1)}, outer sep=2pt,fill=white] {P};
		\end{tikzpicture}
	\end{center}
	\end{framed}
	\compactcaption{\footnotesize Halfspace classification in dimension $1$.}
	\label{fig:bs}
	\vspace{-4mm}
\end{figure}

\subsubsection{Query Synthesis (QS) active learning}
\label{sec:QS}

In this scenario, the learner can request labels for any instance in
the input space $\X$, including points that the learner generates {\it de
novo}, independent of the distribution $\D$ (\eg, $\L$ can ask for
labels for those $x$ that have zero-probability of being sampled
according to $\D$). Query synthesis is reasonable for many problems,
but labeling such arbitrary instances can be awkward if the oracle is
a human annotator. Thus, this scenario better represents real-world
applications where the oracle is automated (\eg, results from
synthetic experiments \cite{king2009automation}). Since the
data-points are independent of the distribution, generalization error
is not an appropriate measure of accuracy of the hypothesis $\hat{f}$,
and other types of error are typically used. These new error
formulations depend on the concrete hypothesis class $\F$
considered. For example, if $\F$ is the class of boolean functions
from $\{0,1\}^n$ to $\{0,1\}$, then the \emph{uniform error} is
used. Assume that the oracle $\O$ knows $f^*\in\F$ and uses it as
labeling function (realizable case), then the uniform error of the
hypothesis $\hat{f}$ is defined as
\begin{equation*}
	\uErr(\hat{f})=\Pr_{x\sim \{0,1\}^n}[\hat{f}(x)\neq f^*(x)]
\end{equation*}
where $x$ is sampled uniformly at random from the instance space $\{0,1\}^n$.
Recent work~\cite{alabdulmohsin2015efficient,chen2017near}, for the
class of halfspaces $\F_{d,HS}$ (refer to Example~\ref{ex:hs}) use 
\emph{geometric error}. Assume that the true labeling function used by
the oracle is $f_{w^*}$, then the geometric error of the hypothesis
$f_w\in\F_{d,HS} $ is defined as
\begin{equation*}
\gErr(f_w)=||w^*-w||_2
\end{equation*}
where $||\cdot||_2$ is the 2-norm.

\vspace{2mm} 
In both active learning scenarios (PAC and QS), the learner needs to evaluate the ``usefulness'' of an
unlabeled instance $x$, which can either be generated de novo or
sampled from the given distribution, in order to decide whether to
query the oracle for the corresponding label. In the state of the art,
we can find many ways of formulating such \emph{query
  strategies}. Most of existing literature presents strategies where
efficient search through the hypothesis space is the goal (refer the
survey by Settles \cite{settlesactive}).
%Another  argument used for query strategies is exploiting cluster structure in the data. Assume that the unlabeled points form clusters and these clusters are pure in their class labels, then only one label per cluster is necessary. This is a simplistic example,  but it illustrates the basic idea of a query strategy that is able to detect and exploit a cluster structure \cite{dasgupta2008hierarchical,dasgupta2011two}. 
Another point of consideration for an active learner $\L$ is to decide
when to stop. This is essential as active learning is geared at
improving accuracy while being sensitive to new data acquisition cost
(\ie, reducing the query complexity). While one school of thought
relies on the {\em stopping criteria} based on the intrinsic measure
of stability or self-confidence within the learner, another believes
that it is based on economic or other external factors (refer
\cite[Section 6.7]{settlesactive}).

Given this large variety within active learning, we enhance the
standard definition of a learning algorithm and propose the definition
of an active learning system, which is geared towards model
extraction. Our definition is informed by the MLaaS APIs that we
investigated (more details are present in Table~\ref{aux}). 

\begin{definition}[\underline{Active learning system}]
	\label{def:AL}
	Let $\F$ be a hypothesis class with instance space $\X$ and
        label space $\Y$. An active learning system for $\F$ is given
        by two entities, the learner $\L$ and the oracle $\O$,
        interacting via membership queries: $\L$ sends to $\O$ an instance
        $x\in\X$; $\O$ answers with a label $y\in\Y$. We indicate via
        the notation $\O_{f^*}$  the realizable case
        where $\O$ uses a specific labeling function $f^*\in \F$, \ie
        $y=f^*(x)$. The behavior of $\L$ is described by the following
        parameters:
		
\begin{enumerate}
\itemsep0em
\item \emph{Scenario}: this is the rule that describes the generation
  of the input for the querying process (\ie which instances $x\in\X$
  can be queried). In the PAC scenario, the instances are sampled from
  the underlying distribution $\D$. In the query synthesis (QS)
  scenario, the instances are generated by the learner $\L$;

\item \emph{Query strategy}: given a specific scenario, the query strategy is the algorithm that adaptively decides if the label for a given instance $x_i$ is queried for, given that the queries $x_1,\dots, x_{i-1}$ have been answered already. In the query synthesis scenario, the query strategy also describes the procedure for instance generation.
\item \emph{Stopping criteria}: this is a set of considerations used by $\L$ to decide when it must stop asking queries. 
\end{enumerate}
Any system $(\L,\O)$ described as above is an active learning system for $\F$ if one of the following holds:
\begin{itemize}[label=-]
\item  {\it (PAC scenario)} For any $\D$ on $\X\times\Y$ and any $\varepsilon,\delta\in (0,1)$, if $\L$ is allowed to interact with $\O$ using  $q_\L(\varepsilon,\delta)$  queries, then $\L$ outputs $\hat{f}\in\F$ such that $\Err_\D(\hat{f})\leq \min_{f\in\F}\Err_\D(f)+\varepsilon$ with probability at least $1-\delta$. 
\item {\it (QS scenario)} Fix an error measure $\Err$ for the functions in $\F$. For any $f^*\in \F$, if $\L$ is allowed to interact with $\O_{f^*}$ using  $q_\L(\varepsilon,\delta)$  queries, then $\L$ outputs $\hat{f}\in\F$ such that $\Err(\hat{f})\leq \varepsilon$ with probability at least $1-\delta$. 
\end{itemize}
We refer to $q_\L(\varepsilon,\delta)$ as the \emph{query complexity} of $\L$.
\end{definition}

As we will show in the following section (in particular, refer subsection~\ref{sec:ALandME}), the query synthesis scenario is more appropriate in casting model extraction attack as active learning
\textcolor{black}{when we make no assumptions about the adversary's prior knowledge.}

 Note that, other types queries have been studied in literature. This includes the \emph{equivalence query}~\cite{angluin1987learning}. Here the learner can verify if a hypothesis is correct or not. We do not consider equivalence queries in our definition because we did not see any of the MLaaS APIs support them.

%%%%%%%%%%%%%%%%%%%%%%%%%%%%%%%%%%%%%%%%%%%%%%
\section{Model Extraction}
\label{sec:ME}
%%%%%%%%%%%%%%%%%%%%%%%%%%%%%%%%%%%%%%%%%%%%%%

In subsection~\ref{sec:MED}, we begin by formalizing the process of model
extraction. We then draw parallels between model
extraction and active learning in subsection~\ref{sec:ALandME}. We proceed to provide insight about extracting non-linear models in Section~\ref{}. %We finally discuss possible defense strategies based on noisy answers in subsection~\ref{sec:defense}.

\subsection{Model Extraction Definition}
\label{sec:MED}

We begin by describing the operational ecosystem of model extraction
attacks in the context of MLaaS systems. An entity learns a private
model $f^*$ from a public class $\F$, and provides it to the MLaaS
server. The server provides a client-facing query interface for
accessing the model for prediction. 
%For each query issued by a client, the server responds with the corresponding prediction. 
For example, in the case of logistic regression, the MLaaS server knows a
model represented by parameters $a_0,a_1,\cdots,a_d$. The client
issues queries of the form $x=(x[1],\cdots,x[d])\in\R^d$, and the
MLaaS server responds with $0$ if $(1+e^{-a(x)})^{-1} \leq 0.5$ and
$1$ otherwise, with $a(x)=a_0+\sum_{i=1}^{d} a_i x[i]$.

Model extraction is the process where an adversary exploits this
interface to learn more about
the proprietary model $f^*$. The adversary can be interested in
defrauding the description of the model $f^*$ itself (\ie, stealing
the parameters $a_i$ as in a reverse engineering attack), or in
obtaining an approximation of the model, say $\hat{f}\in\F$, that he
can then use for free for the same task as the original $f^*$ was
intended for. To capture the different goals of an adversary, we say
that the attack is successful if the extracted model is ``close enough''
to $f^*$ according to an \emph{error function} on $\F$ that is context
dependent. Since many existing MLaaS providers operate in a
pay-per-query regime, we use query complexity as a measure of
efficiency of such model extraction attacks.

Formally, consider the following experiment: an adversary $\A$, who
knows the hypothesis class $\F$, has oracle access to a proprietary
model $f^*$ from $\F$. This can be thought of as $\A$ interacting with
a server $\Ser$ that safely stores $f^*$. The interaction has several
rounds. In each round, $\A$ chooses an instance $x$ and sends it to
$\Ser$. The latter responds with $f^*(x)$. After a few rounds, $\A$
outputs a function $\hat{f}$ that is the adversary's candidate
approximation of $f^*$; the experiment considers $\hat{f}$ a good
approximation if its error with respect to the true function $f^*$
held by the server is less then a fixed threshold $\varepsilon$. The
error function $\Err$ is defined a priori and fixed for the extraction
experiment on the hypothesis class~$\F$.

\begin{experiment}[\underline{Extraction experiment}] 
	\label{exp:extractionC} Given a hypothesis class $\F=\{f:\X \rightarrow\Y\}$, fix an error function $\Err:\F \rightarrow \R$. Let $\Ser$ be a MLaaS server with the knowledge of a specific $f^*\in\F$, denoted by $\Ser(f^*)$. Let $\A$ be an adversary interacting with $\Ser$ with a maximum budget of $q$ queries. The extraction experiment $\Exp^\varepsilon_{\F}(\Ser(f^*),\A,q)$ proceeds as follows
\begin{enumerate}
	\item $\A$ is given a description of $\F$ and oracle access to $f^*$ through the query interface of $\Ser$. That is, if  $\A$ sends $x\in \X$ to $\Ser$, it gets back $y=f^*(x)$. After at most $q$ queries, $\A$ eventually outputs $\hat{f}$;
%	\item  Sample $m$ i.i.d instances $x_i$ from $\D$, the output of the experiment is 1 if and only if $\frac{1}{m} \sum_{i=1}^m\ell(\hat{f},x_i)\leq \varepsilon$.
	\item  The output of the experiment is $1$ if 
	 $\Err(\hat{f})\leq \varepsilon$.	 Otherwise the output is $0$.
\end{enumerate}
\end{experiment}

\noindent
Informally, an adversary $\A$ is successful if with high probability the output of the extraction experiment is $1$ for a small value of $\varepsilon$ and a fixed query budget $q$. This means that $\A$ likely learns a good approximation of $f^*$ by only asking $q$ queries to the server. More precisely, we have the following definition.

\begin{definition}[\underline{Extraction attack}]
	\label{def:ext}
	Let $\F$ be a public hypothesis class and $\Ser$ an MLaaS server as explained before. We say that an adversary $\A$,  which interacts with $\Ser$, implements an \emph{$\varepsilon$-extraction attack} of complexity $q$ and confidence $\gamma$ against the class $\F$ if 
	\begin{equation*}
		\Pr[\Exp^\varepsilon_{\F}(\Ser(f^*),\A,q)=1]\geq \gamma
	\end{equation*}
	for any $f^*\in\F$. The probability is over the randomness of $\A$.
\end{definition}

%\Somesh{should we say this is from Tramer et al \cite{TZJRR16}?} \Irene{Why? \cite{TZJRR16} implements model extraction but does not give a general definition of it...}\\

%\Irene{I added the following brief discussion about Definition \ref{def:ext}. Let me know if you agree with it!} 
In other words, in Definition~\ref{def:ext} the success probability of an adversary constrained by a fixed budget for queries  is explicitly lower bounded by the quantity $\gamma$. %, that we call confidence.
% This concrete approach is motivated by real-world applications where the adversary operates in a pay-per-query regime, and has a fixed budget for the attack.

%An alternative, asymptotic approach, is one where the query complexity and the success probability of $\A$ are functions of some security parameter (\eg, the error level $\varepsilon$). We leave determining these relations for future investigation. \Songbai{The last sentence is a little bit redundant (since it is unclear whether such a difference really matters) and may confuse the readers.}
%\Irene{Ok, we can remove everything from \qquote{An alternative, ....}}

\vspace{1mm}
Before discussing the connection between model extraction and active learning, we provide an example of a hypothesis class that is easy to extract.

\begin{example}[\underline{Equation-solving attack for linear regression}]
\label{ex:eq}
Let $\F_{d,R}$ be the hypothesis class of regression models from $\R^d$ to $\R$. A function $f_a$ in this  class is described by $d+1$ parameters $a_0,a_1,\dots,a_d$ from $\R$ and defined by: for any $x\in\R^d$,
\begin{equation*}
	f_a(x)=a_0+\sum_{i=1}^d a_ix_i\,.
\end{equation*} 
Consider the adversary $\A_{ES}$ that queries  $x^1,\dots,x^{d+1}$ ($d+1$ instances from $\R^d$) chosen in such a way that the set of vectors  $\{(1,x^i)\}_{i=1,\dots,d+1}$ is linearly independent in $\R^{d+1}$. $\A_{ES}$ receives the corresponding $d+1$ labels, $y_1,\dots,y_{d+1}$, and can therefore solve the linear system given by the equations $f_a(x^i)=y_i$. Assume that $f_{a^*}$ is the function known by the MLaaS server (\ie, $y_i=f_{a^*}(x^i)$). It is easy to see that if we fix $\Err(f_a)=||a^*-a||_1$, then  $\Pr[\Exp^0_{\F_{d,R}}(\Ser(f_{a^*}),\A_{ES},d+1)=1]=1$. That is, $\A_{ES}$ implements $0$-extraction of complexity $d+1$ and confidence~$1$.

%Tram\`{e}r \etal \cite{TZJRR16} introduce the equation solving attack for logistic regression and multilayer perceptron models. However, they consider a different attack model, one where the server returns the label $y$ and auxiliary information about it (\eg, the probability of $y$ being correct) for each instance $x$. We discuss such attack models in more detail in Remark~\ref{rem:non-blackbox}.
While our model operates in the black-box setting, we discuss other attack models in more detail in Remark~\ref{rem:non-blackbox}
\end{example}

%%%%%%%%%%%%%%%%%%%%%%%%%%%%%%%%%%%%%%%%%%%%%%%
\subsection{Active Learning and Extraction}
\label{sec:ALandME}
%%%%%%%%%%%%%%%%%%%%%%%%%%%%%%%%%%%%%%%%%%%%%%%
%\Somesh{Should we say how we modify Experiment \ref{exp:extractionC} to introduce auxiliary information?}

From the description presented in the Section~\ref{sec:ML}, it is
clear that model extraction in the MLaaS system context closely
resembles active learning. The survey of active learning in
subsection~\ref{sec:AL} contains a variety of algorithms and scenarios
which can be used to implement model extraction attacks (or to study
its impossibility).

%\st{However, not all possible scenarios of active
%learning are interesting for model extraction. 
%We notice that in the case of model extraction, an adversary $\A$ has no knowledge of the
%data distribution $\D$. Additionally, such an adversary is not restricted
%to only considering instances $x\sim\D$ to query.}

%THE PREVIOUS PARAGRAPH MODIFIED AS:

\textcolor{black}{However, different scenarios of active learning impose different assumptions on the adversary's prior knowledge. Here, we focus on the general case of model extraction with an adversary $\A$ that has no knowledge of the	data distribution $\D$. In particular, such an adversary is not restricted to only considering instances $x\sim\D$ to query.} For this reason, we
believe that query synthesis (QS) is the right active learning
scenario to investigate in order to draw a meaningful parallelism with
model extraction. Recall that the query synthesis is the only
framework where the query inputs can be generated de novo (\ie, they
do not conform to a distribution).

 %Our earlier discussion leads to  the following observation.
 
\noindent
{\bf \underline{Observation 1}:} Given a hypothesis class $\F$ and an error function
$\Err$, let $(\L,\O$) be an active learning system for $\F$ in the QS
scenario (Definition~\ref{def:AL}). If the query complexity of $\L$ is
$q_\L(\varepsilon,\delta)$, then there exists and adversary $\A$ that
implements $\varepsilon$-extraction with complexity
$q_\L(\varepsilon,\delta)$ and confidence $1-\delta$ against the
class~$\F$.

The reasoning for this observation is as follows:
Consider the adversary $\A$ that is the learner $\L$ (\ie, $\A$ deploys the query strategy procedure and the stopping criteria that describe $\L$). This is possible because $(\L,\O)$ is in the QS scenario and $\L$ is independent of any underlying (unknown) distribution. Let  $q=q_\L(\varepsilon,\delta)$ and observe that
\begin{align*}
	&\Pr[\Exp^\varepsilon_{\F}(\Ser(f^*),\A,q)=1]=\\
	&\Pr[\A \text{ outputs } \hat{f}\text{ and } \Err(\hat{f})\leq \varepsilon]=\\
	&\Pr[\L \text{ outputs } \hat{f}\text{ and } \Err(\hat{f})\leq \varepsilon]\geq 1-\delta
\end{align*} 

Our observation states that any active learning algorithm in the QS
scenario can be used to implement a model extraction
attack. Therefore, in order to study the security of a given
hypothesis class in the MLaaS framework, we can use known techniques
and results from the active learning literature. Two examples of this
follow. 
% - we describe the active learning algorithms presented in \cite{alabdulmohsin2015efficient} and in \cite{Kushilevitz-Mansour} as model extraction attacks.

\begin{example}[\underline{Decision tree extraction via QS active learning}]
Let $\F_{n,BF}$ denote the set of boolean functions with domain $\{
0,1 \}^n$ and range $\{-1,1\}$. The reader can think of $-1$ as $0$ and $+1$ as
$1$. Using the range of $\{ -1, +1 \}$ is very common in the
literature on learning boolean functions. An
interesting subset of $\F_{n,BF}$ is given by the functions that can
be represented as a boolean decision tree.  A {\it boolean decision
  tree} $T$ is a labeled binary tree, where each node $v$ of the tree
is labeled by $L_v \subseteq \{ 1, \cdots, n\}$ and has two outgoing
edges. Every leaf in this tree is labeled either $+1$ or $-1$. Given
an $n$-bit string $x = (b_1,\cdots,b_n), b_i \in \{0,1\}$ as input,
the decision tree defines the following computation: the computation
starts at the root of the tree $T$. When the computation arrives at an
internal node $v$ we calculate the parity of $\sum_{i \in L_v}b_i$ and
go left if the parity is $0$ and go right otherwise. The value of the
leaf that the computation ends up in is the value of the function.  We
denote by $\F^m_{n,BT}$ the class of boolean decision trees with
$n$-bit input and $m$ nodes.  Kushilevitz and
Mansour~\cite{Kushilevitz-Mansour} present an active learning
algorithm for the class $\F_{n,BF}$ that works in the QS
scenario. This algorithm utilizes the uniform error to determine the
stopping condition (refer subsection~\ref{sec:AL}). The authors claim
that this algorithm has practical efficiency when restricted to the
classes $\F^m_{n,BT}\subset \F_{n,BF}$ for any $m$.  In particular, if
the active learner $\L$ of \cite{Kushilevitz-Mansour} interacts with
the oracle $\O_{T^*}$ where $T^*\in \F^m_{n,BT}$, then $\L$ learns
$g\in\F_{n,BF}$ such that $\Pr_{x\sim\{0,1\}^n} [g(x) \not=T^*(x)]
\leq \varepsilon$ with probability at least $1-\delta$ using a number
of queries polynomial in $n$, $m$, $\frac{1}{\varepsilon}$ and
$\log(\frac{1}{\delta})$. Based on Observation 1, this directly
translates to the existence of an adversary that implements
$\varepsilon$-extraction with complexity polynomial in $n$, $m$,
$\frac{1}{\varepsilon}$ and confidence $1-\delta$ against the class
$\F^m_{n,BT}$.

Moreover, the algorithm of \cite{Kushilevitz-Mansour} can be extended
to (a) boolean functions of the form $f:\{0,1,\dots,k-1\}^n\rightarrow
\{-1,+1\}$ that can be computed by a polynomial-size $k$-ary decision
tree\footnote{A $k$-ary decision tree is a tree in which each inner
  node $v$ has $k$ outgoing edges.}, and (b) regression trees (\ie,
the output is a real value from $[0,M]$). In the second case, the
running time of the learning algorithm is polynomial in $M$ (refer
Section~6 of~\cite{Kushilevitz-Mansour}).  Note that the attack model
considered here is a stronger model than that considered
by~\cite{TZJRR16} because the attacker/learner does not get any
information about the internal path of the decision tree (refer
Remark~\ref{rem:non-blackbox}).
\end{example}

\begin{example}[\underline{Halfspace extraction via QS active learning}]
\label{ex:hs2}
Let $\F_{d,HS}$ be the hypotheses class of $d$-dimensional halfspaces
defined in Example~\ref{ex:hs}. Alabdulmohsin
\etal~\cite{alabdulmohsin2015efficient} present a spectral algorithm
to learn a halfspace in the QS scenario that, in practice,
outperformed earlier active learning strategies in the PAC
scenario. They demonstrate, through several experiments that their
algorithm learns $f_w\in \F_{d,HS}$ such that $\|w-w^*\|_2\leq
\varepsilon$ with approximately $2d\log(\frac{1}{\varepsilon})$
queries, where $f_{w^*}\in \F_{d,HS}$ is the labeling function used by
$\O$.  It follows from Observation 1 that an adversary
utilizing this algorithm implements $\varepsilon$-extraction against
the class $\F_{d,HS}$ with complexity
$\O(d\log(\frac{1}{\varepsilon}))$ and confidence $1$. We validate the
practical efficacy of this attack in Section~\ref{sec:implementation}.
\end{example}

\begin{remark}[\underline{Extraction with auxiliary information}]
\label{rem:non-blackbox}

Observe that we define model extraction for only those MLaaS servers
that return {\em only the label value} $y$ for a {\em well-formed
  query} $x$ (\ie in the oracle access setting). A weaker model (\ie,
one where attacks are easier) considers the case of MLaaS servers
responding to a user's query $x$ even when $x$ is incomplete (\ie with
missing features), and returning the label $y$ along with some
auxiliary information. The work of Tram\`{e}r \etal~\cite{TZJRR16}
proves that model extraction attacks in the presence of such
\qquote{leaky servers} are feasible and efficient (\ie low query
complexity) for many hypothesis classes (\eg, logistic regression,
multilayer perceptron, and decision trees).  In particular, they
propose an \emph{equation solving attack} \cite[Section 4.1]{TZJRR16}
that uses the confidence values returned by the MLaaS server together
with the labels to steal the model parameters. For example, in the case of logistic
regression, the MLaaS server knows the parameters $a_0,a_1,\dots, a_d$
and responds to a query $x$ with the label $y$ ($y=0$ if
$(1+e^{-a(x)})\leq0.5$ and $y=1$ otherwise) \emph{and the value
  $a(x)$} as confidence value for $y$. Clearly, the knowledge of the
confidence values allows an adversary to implement the same attack we
describe in Example~\ref{ex:eq} for linear regression models.
In \cite[Section 4.2]{TZJRR16}, the authors describes a \emph{path-finding attack}
that use the leaf/node identifier returned by the server, even for
incomplete queries, to steal a decision tree. These attacks are very
efficient (\ie, $d+1$ queries are needed to steal a $d$-dimensional
logistic regression model); however, their efficiency heavily relies
on the presence of the various forms of auxiliary information provided
by the MLaaS server. 
%In Section 6 of \cite{TZJRR16}, attacks for
%logistic regression, neural networks, and kernel SVMs are proposed
%that do not use the confidences values. For binary logistic regression
%the best strategy proposed is the linear search of Lowd and Meek
%\cite{LM05}, while for multiclass logistic regression and RBF kernel
%SVMs, the best attack uses the PAC active learning algorithm of
%\cite{cohn1994improving} (stream-based sampling model). All these
%attacks are considered less feasible than the earlier proposals
%because of the higher query cost (\eg, for logistic regression,
%$100\times$ less queries are needed in the equation solving
%attack). 
While the work in \cite{TZJRR16} performs preliminary exploration of attacks in the black-box setting \cite{cohn1994improving, LM05}, 
it does not consider more recent, and efficient algorithms in the QS scenario. Our work explores
this direction through a formalization of the model extraction
framework that enables understanding the possibility of
extending/improving the active learning attacks presented in
\cite{TZJRR16}. Furthermore, having a better understanding of model
extraction attack and its unavoidable connection with active learning
is paramount for designing MLaaS systems that are resilient to model
extraction.

%\todo{Varun, should we add here something bout out DT extraction that use PAC and no assumption in the API?}

%\Irene{More general comment: should we rename this remark "extraction with no API assumptions"?}
\end{remark}

\section{Non-linear Classifiers}
\label{sec:non-linear}

This section focuses on model extraction for two important non-linear
classifiers: kernel SVMs and discrete models (\ie decision trees and
random forests). For kernel SVMs our method is a combination of the
adaptive-retraining algorithm introduced by Tram\`{e}r \etal and the
active selection strategy from classic literature on active
learning of kernel SVMs~\cite{LASVM}. For discrete models our
algorithm is based on the importance weighted active learning (IWAL)
as described in~\cite{iwal:2010}. Note that decision trees
for general labels (\ie non-binary case) and random forests was not
discussed in~\cite{iwal:2010}.

\subsection{Kernel SVMs}
\label{subsec:kSVM}

In kernel SVMs (kSVMs), there is a kernel $K:\X\times\X\rightarrow\R$ associated with the
SVM. Some of the common kernels are polynomials and radial-basis
functions (RBFs). If the kernel function $K(.,.)$ has some special
properties (required by classic theorem of Mercer~\cite{minh2006mercer}), then $K(.,.)$ can
be replaced with $\Phi (.)^T \Phi(.)$ for a projection/feature function $\Phi$. In
the feature space (the domain of $\Phi$) the optimization problem is
as follows\footnote{we are using the formulation for soft-margin kSVMs}:
\[
\begin{array}{c}
  \mbox{min}_{w,b} \| w \|^2 + C \sum_{i=1}^n \eta_i \\
  \mbox{such that for $1 \leq i \leq n$} \\
  y_i \hat{y}(x_i) \; \geq \; 1 - \eta_i \\
  \eta_i \; \geq 0
\end{array}
\]
In the formulation given above, $\hat{y}(x)$ is equal to $w^T \Phi(x) +
b$. Recall that prediction of the kSVM is the sign of $\hat{y}(x)$, so
$\hat{y}(x)$ is the ``pre sign'' value of the prediction. Note that
for some kernels (\eg RBF) $\Phi$ is infinite dimensional, so one
generally uses the ``kernel trick''\ie one solves the dual of
the above problem and obtains a {\it kernel expansion}, so that
\begin{eqnarray*}
  \hat{y}(x) & = & \sum_{i=1}^n \alpha_i K(x,x_i) \; + \; b
\end{eqnarray*}
The vectors $x_1,\cdots,x_n$ are called support vectors.
We assume that hyper-parameters of
the kernel ($C, \eta$) are known; one can extract the hyper-parameters for
the RBF kernel using the extract-and-test approach as Tram\`{e}r \etal
Note that if $\Phi$ is finite dimensional, we can use an algorithm (including active learning strategies)
for linear classifier by simply working in the feature space (\ie
extracting the domain of $\Phi(\cdot)$). However, there is a subtle issue here, which
was not addressed by Tram\`{e}r \etal We need to make sure that if a query
$y$ is made in the feature space, it is ``realizable'' (\ie there exists
a $x$ such that $\Phi (x) = y$). Otherwise the learning algorithm
is not sound. 

Next we describe our model-extraction algorithm for kSVMs with kernels
whose feature space is infinite dimension (\eg RBF or Laplace
kernels). Our algorithm is a modification of the adaptive training
approach from Tram\`{e}r \etal Our discussion is specialized to kSVMs with
RBFs, but our ideas are general and are applicable in other
contexts.

\noindent
%Adaptive training was introduced by Tram\`{e}r \etal~\cite{TZJRR16}. We extend their method (hence the name), and show that these extensions can improve query complexity. 
{\bf Extended Adaptive Training (EAT):} EAT proceeds in multiple rounds. In each round we
construct $h$ labeled instances. In the initial stage ($t=0$) we draw $r$
instances $x_1,\cdots,x_r$ from the uniform distribution, query their
labels, and create an initial model $M_0$. Assume that we are at round
$t$, where $t > 0$, and let $M_{t-1}$ be model at time $t-1$. Round
$t$ works as follows: create $h$ labeled instances using a strategy
$St^{\mathcal{T}} (M_{t-1},h)$ (note that the strategy $St$ is oracle
access to the teacher, and takes as parameters model from the previous
round and number of labeled instances to be generated). Now we train
$M_{t-1}$ on the instances generated by $St^{\mathcal{T}} (M_{t-1},h)$ and obtain
the updated model $M_t$. We keep iterating using the strategy
$St^{\mathcal{T}} (\cdot,\cdot)$ until the query budget is
satisfied. Ideally, $St^{\mathcal{T}} (M_{t-1},h)$ should be instances
that the model $M_{t-1}$ is {\em least confident} about or closest to the decision
boundary. %The challenge lies in finding these uncertain points. % close to the decision boundary? This is where we differ and generalize Tram\`{e}r's method.

Tram\`{e}r \etal use line search as their strategy $St^{\mathcal{T}}
(M_{t-1},h)$, which can lead to several queries (each step in the binary
search leads to a query). We generate the initial model $M_0$ as in
Tram\`{e}r \etal and then our strategy differs.  Our strategy
$St^{\mathcal{T}}(M_{t-1},1)$ (note that we only add one labeled
sample at each iteration) works as follows: we generate $k$ random
points $x_1,\cdots,x_k$ and then compute $\hat{y_i}(x_i)$ for each
$x_i$ (recall that $\hat{y_i}(x_i)$ is the ``pre sign'' prediction of
$x_i$ on the SVM $M_{t-1}$. We then pick $x_i$ with minimum $\mid
\hat{y_i}(x_i) \mid$ and query for its label and retrain the model
$M_{t-1}$ and obtain $M_t$. This strategy is called {\it active
  selection} and has been used for active learning of
SVMs~\cite{LASVM}. The argument for why this strategy finds the point
closest to the boundary is given in~\cite[\S 4]{LASVM}. There are
other strategies described in~\cite{LASVM}, but we found active
selection to perform the best.

\subsection{Decision Trees and Random Forests}
\label{subsec:iwal}

Next we will describe the idea of {\it importance weighted active
  learning (IWAL)}~\cite{iwal:2010}. Our discussion will be specialized to decision
trees and random forests, but the ideas that are described are
general.

Let $\mathcal{H}$ be the hypothesis class (\ie space of decision
trees or random forests), $\X$ is the space of data, and
$\Y$ is the space of labels.  The active learner has a pool
of unlabeled data $x_1,x_2,\cdots$. For $i > 1$, we denote by
$X_{1:i-1}$ the sequence $x_1,\cdots,x_{i-1}$. After having processed the sequence
$X_{1:i-1}$, a coin is flipped with probability $p_i \in [0,1]$ and
if it comes up heads, the label of $x_i$ is queried. We also define a set
$S_i$ ($S_0 = \emptyset$) recursively as follows: If the label for
$x_i$ is not queried, then $S_i \; = \; S_{i-1}$; otherwise $S_i \; =
\; S_{i-1} \cup {(x_i,y_i,p_i)}$. Essentially the set $S_i$ keeps the
information (\ie data, label, and probability of querying) for all
the datapoints whose label was queried. Given a hypothesis $h \in \mathcal{H}$,
we define $err(h,S_n)$ as follows:
\begin{eqnarray}
  err(h,S_n) & = & \frac{1}{n} \sum_{(x,y,p) \in S_n} \frac{1}{p} 1_{h(x) \not= y}
\end{eqnarray}
Next we define the following quantities (we assume $n \geq 1$):
\begin{eqnarray*}
  h_n & = & \mbox{argmin} \{ err(h,S_{n-1}) \; : \; h \in \mathcal{H} \} \\
  h'_n & = & \mbox{argmin} \{ err(h,S_{n-1}) \; : \; h \in \mathcal{H} \wedge h(X_n) \not= h_n (X_n) \} \\
  G_n & = & err(h'_n,S_{n-1}) - err(h_n,S_{n-1})
\end{eqnarray*}
Recall that $p_n$ is the probability of querying for the label for
$X_n$, which is defined as follows:
\[
p_n = \left\{ \begin{array}{ll}
  1 & \mbox{if $G_n \leq \mu(n)$} \\
  s(n) & \mbox{otherwise}
  \end{array}
  \right.
  \]
  Where $\mu (n) = \sqrt{ \frac{c_0 \log n}{n-1} } + \frac{c_0 \log n}{n-1}$, and $s(n) \in (0,1)$
  is the positive solution to the following equation:
  \begin{eqnarray*}
    G_n & = & \left( \frac{c_1}{\sqrt{s} - c_1 + 1} \right) \cdot   \sqrt{ \frac{c_0 \log n}{n-1} }  +
    \left( \frac{c_2}{\sqrt{s} - c_2 + 1} \right) \cdot  \frac{c_0 \log n}{n-1}
  \end{eqnarray*}
  Note the dependence on constants/hyperparameters $c_0$, $c_1$ and
  $c_2$, which are tuned for a specific problem (\eg in their
  experiments for decision trees~\cite[\S 6]{iwal:2010} the
  authors set $c_0=8$ and $c_1 = c_2 = 1$).

  \noindent
      {\bf Decision Trees:} Let $DT$ be any algorithm to create a
      decision tree. We start with an initial tree $h_0$ (this can
      constructed using a small, uniformly sampled dataset whose labels are queried). Let
      $h_n$ be the tree at step $n-1$. The question is: how to
      construct $h'_n$? Let $x_n$ be the $n^{th}$ datapoint and $\Y
      = \{ l_1,\cdots,l_r \}$ be the set of labels. Let $h_n (x_n) =
      l_j$. Let $h_n (l)$ be the modification of tree $h_n$ such that
      $h_n (l)$ produces label $l \not= h_n (x_n)$ on datapoint $x_n$. Let
      $h'_n$ be the tree in the set $\{ h_n (l) \; | \; l \in \Y-\{
      l_j \} \}$ that has minimum $err(\cdot,S_{n-1})$. Now we can
      compute $G_n$ and the algorithm can proceed as described before.

      \noindent
          {\bf Random Forests:} In this case we will restrict
          ourselves to binary classification, but the algorithm can
          readily extended to the case of multiple labels. As before
          $RF_0$ is the random forest trained on a small initial
          dataset. Since we are in the binary classification domain, the
          label set $\Y = \{ 1,-1 \}$. Assume that we have a
          random forest $RF = \{ RF[1],\cdots,RF[o] \}$ of trees
          $RF[i]$ and on a datapoint $x$ the label of the random forest $RF(x)$ is the
          majority of the label of the trees $RF[1] (x),\cdots,
          RF[o](x)$. Let $RF_n$ be the random forest at time step
          $n-1$. The question again is: how to construct $RF'_n$?
          Without loss of generality, let us say on $x_n$ $RF_n (x_n) = +1$ (the case when
          the label is $-1$ is symmetric) and there are $r$ trees in
          $RF_n$ (denoted by $RF_n^{+1} (x_n)$) such that their labels
          on $x_n$ are $+1$. Note that $r > \lfloor \frac{o}{2}
          \rfloor $ because the majority label was $+1$. Define $j =
          r-\lfloor \frac{o}{2} \rfloor+1$.  Note that if $j$ trees in
          $RF^{+1}_n (x_n)$ will ``flip'' their decision to $-1$ on $x_n$, then
          the decision on $x_n$ will be flipped to $-1$. This is the
          intuition we use to compute $RF'_n$. %We pick $j$ trees out of $RF_n^{+1} (x_n)$ and flip their decisions to $-1$. 
          There are ${r \choose j}$ choices of trees and we pick the one
          with minimum error on $S_{n-1}$, and that gives us $RF'_n$.
          Recall that $r \choose j$ is approximately $r^j$, but we can
          be approximate by randomly picking $j$ trees out of $RF_n^{+1}
          (x_n)$, and choosing the random draw with the minimum
          error to approximate $RF'_n$.

%%%%%%%%%%%%%%%%%%%%%%%%%%%%%%%%%%%%%%%%%%%%%%%%%
\section{Defense Strategies}
\label{sec:defense}
%%%%%%%%%%%%%%%%%%%%%%%%%%%%%%%%%%%%%%%%%%%%%%%%%

Our main observation is  that model extraction
in the context of MLaaS systems described at the beginning of
Section~\ref{sec:ME} (\ie, oracle access) is equivalent to QS active
learning.  Therefore, any advancement in the area of QS active
learning directly translates to a new threat for MLaaS systems. In
this section, we discuss strategies that could be used to make the
process of extraction more difficult.%, such as adding some noise to the returned label
We investigate the link between machine-learning
in the noisy setting and model extraction. The design of a
good defense strategy is an open problem; we believe this is an
interesting direction for future work where the machine learning and
the security communities can fruitfully collaborate.

\vspace{1mm} 
%For the remainder of this section, we focus on the classification problem where $\F$ is an hypothesis class of functions of the form $f:\X\rightarrow \Y$ and $\Y$ is finite. 
In this section, we assume that the MLaaS server $\Ser$ with the
knowledge of $f^*$, $\Ser(f^*)$, has the freedom to modify the
prediction before forwarding it to the client. More precisely, we
assume that there exists a (possibly) randomized procedure $D$ that
the server uses to compute the answer $\tilde{y}$ to a query $x$, and
returns that instead of $f^*(x)$.  We use the notation $\Ser_D(f^*)$
to indicate that the server $\Ser$ implements $D$ to protect $f^*$.
Clearly, the learner that interacts with $\Ser_D(f^*)$ can still try
to learn a function $f$ from the noisy answers from the server.
However, the added noise requires the process to make more queries,
or could produce a less accurate model than $f$. 

\subsection{Classification case}
%\label{subsec:Classification}
We focus on the binary classification problem where $\F$ is an
hypothesis class of functions of the form $f:\X\rightarrow \Y$ and
$\Y$ is binary, but our argument can be easily generalized to the
multi-class setting.

%\Songbai{The original sentence here was "$\Y$ is finite", but the argument in this section actually works for binary classification only. Generalizing them to multiclass is straightforward, but would add some complexity (for example, we may need to deal with different generalizations of the "VC dimension", and we need to consider $|\Y|$ for theoretical bounds.)}
%\Irene{Songbai, what are you exactly referring to with \qquote{the argument %in this section}?  Everything I wrote in the part \qquote{No easy defense} (especially Proposition \ref{prop}) works as it is for any size of $\Y$ (finite). Are you referring to the part \qquote{Theoretical general bound}?}.  

%\emph{Theoretical general bounds}.

First, in the following two remarks we recall two known results from
the literature~\cite{key-1} that establish information theoretic
bounds (i.e., the computational cost is ignored) for the number of
queries required to extract the model when any defense is
implemented. Let $\nu$ be the generalization error of the model $f^*$
known by the server $\Ser_D$ and $\mu$ be the generalization error of
the model $f$ learned by an adversary interacting with
$\Ser_D(f^*)$. Assume that the hypothesis class $\F$ has VC dimension
equal to $d$.  Recall that the VC dimension of a hypothesis class $\F$
is the largest number $d$ such that there exists a subset $X\subset\X$
of size $d$ which can be shattered by $\F$. A set $X=\{x_1,
\dots, x_d\}\subset\X$ is said to be shattered by $\F$ if $|\{(f(x_1),
f(x_2),\dots,f(x_d)):f\in\F\}|=2^d$.

\begin{remark}[\underline{Passive learning}]

Assume that the adversary uses a passive learning algorithm to compute
$f$, such as the Empirical Risk Minimization (ERM) algorithm, where given
  a labeled training set $\{(X_1,Y_1),\dots(X_n,Y_n)\}$, the ERM algorithm outputs
  $\hat{f}=\arg\min_{f\in\F}\frac1n \sum_{i=1}^n \One[f(X_i)\neq
    Y_i]$.  Then, the adversary can learn $\hat{f}$ with excess
error $\varepsilon$ (\ie, $\mu\leq\nu+\varepsilon$) with
$\tilde{O}(\frac{\nu+\varepsilon}{\varepsilon^{2}}d)$ examples. For
any algorithm, there is a distribution such that the algorithm needs
at least $\tilde{\Omega}(\frac{\nu+\varepsilon}{\varepsilon^{2}}d)$
samples to achieve an excess error of $\varepsilon$.
\end{remark}

\begin{remark}[\underline{Active learning}]
	Assume that the adversary uses an active learning algorithm to compute $f$, such as the disagreement-based active learning algorithm \cite{key-1}. Then, the adversary  achieves excess error $\varepsilon$ with $\tilde{O}(\frac{\nu^{2}}{\varepsilon^{2}}d\theta)$ queries (where $\theta$ is the disagreement coefficient~\cite{key-1}).
	%\footnote{For any $f\in\F$, $r>0$, define $B(f,r):=\{f'\in\F\mid \Pr(f(X)\neq f'(X))\leq r\}$ to be $r$-ball around $f$. For any $V\subseteq\F$, define the disagreement region $\text{DIS}(V):=\{x\in\mathcal{X}\mid\exists f_{1}\neq f_{2}\in V\text{ s.t. }f_{1}(x)\neq f_{2}(x)\}$. For any $r_{0}\geq2\nu$, define  $\theta(r_{0})=\sup_{r>r_{0}} \frac{1}{r}\Pr(\text{DIS}(B(f^{\star},r)))$. Finally, the  disagreement coefficient $\theta$ is defined as $\theta=\theta(2\nu)$.}
	For any active learning algorithm, there is a distribution such that it takes at least $\tilde{\Omega}(\frac{\nu^{2}}{\varepsilon^{2}}d)$ queries to achieve an excess error of $\varepsilon$.	
	%\Irene{@Songbai: the disagreement coefficient is not defined!}	
	%\Songbai{I think it might be better to give a reference of the definition instead of having exact definition in the paper since this disagreement coefficient is not used elsewhere and needs some additional notations to define.}
\end{remark}

%\emph{No easy defense.}
%Then, we discuss some common defense strategies that unfortunately do not define an effective defense. 
Observe that any defense strategy $D$ used by a server $\Ser$ to prevent the extraction of a model $f^*$ can be seen as a randomized procedure that outputs $\tilde{y}$ instead of $f^*(x)$ with a given probability over the random coins of $D$.  In the discrete case, we represent this with the notation
\begin{equation}
\label{eq:rho}
	\rho_D(f^*,x)=\Pr[Y_x\neq f^*(x)],
\end{equation} 
where $Y_x$ is the random variable that represents the answer of the server $\Ser_D(f^*)$ to the query $x$ (\eg, $\tilde{y}\leftarrow Y_x$).
%where $\tilde{y}$ is the response of $\Ser_D(f^*)$ to the query $x$.
When the function $f^*$ is fixed, we can consider the supremum of the function $\rho_D(f^*,x)$, which represents the upper bound for the probability that an answer from $S_D(f^*)$ is wrong:
\begin{equation*}
	\rho_D(f^*)=\sup_{x\in\X}\rho_D(f^*,x).
\end{equation*}	
Before discussing potential defense approaches, we first present a general negative result. The following proposition states that that any candidate defense $D$ that correctly responds to a query with probability greater than or equal to $\frac{1}{2}+c$ for some constant $c>0$ 
%\Songbai{More precisely, this should be "with probability greater than or equal to $\frac{1}{2}+c$ for some constant $c>0$"} 
for all instances can be easily broken. Indeed, an adversary that repetitively queries the same instance $x$ can figure out the correct label $f^*(x)$ by simply looking at the most frequent label that is returned from $\Ser_D(f^*)$. We prove that with this extraction strategy, the number of queries required increases by only a logarithmic multiplicative factor. %More precisely:
%However, condition~2 in Definition~\ref{def:defense} states that a good defense strategy increases the query complexity {\em at least quadratically}. 

\begin{proposition}
\label{prop}
	Let $\F$ be an hypothesis class used for classification and  $(\L,\O)$ be an active learning system for $\F$ in the QS scenario with query complexity $q(\varepsilon,\delta)$. For any $D$, randomized procedure for returning labels, such that there exists $f^*\in\F$ with $\rho_D(f^*)<\frac12$, there exists an adversary that, interacting with $\Ser_D(f^*)$, can implement an $\varepsilon$-extraction attack with confidence $1-2\delta$ and complexity $q=\frac{8}{(1-2\rho_D(f^*))^2}q(\varepsilon,\delta)\ln\frac{q(\varepsilon,\delta)}{\delta}$.
	
\end{proposition}

The proof of Proposition~\ref{prop} can be found in  Appendix~\ref{app:prop}.

\noindent
Proposition~\ref{prop} can be used to discuss the following two different defense strategies:

\emph{1. Data-independent randomization.} Let $\F$ denote a hypothesis class that is subject to an
        extraction attack using QS active learning. An intuitive
        defense for $\F$ involves adding noise to the query output
        $f^*(x)$ independent of the labeling function $f^*$ and the
        input query $x$. In other words, $\rho_D(f,x)=\rho$ for any
        $x\in\X$, $f\in\F$, and $\rho$ is a constant value in the
        interval $(0,1)$.  It is easy to see that this simple strategy
        cannot work. It follows from Proposition~\ref{prop} that if
        $\rho<\frac{1}{2}$, then $D$ is not secure. On the other hand,
        if $\rho \geq \frac{1}{2}$, then the server is useless since
        it outputs an incorrect label with probability at least $\frac{1}{2}$.
	
\begin{example}[\underline{Halfspace extraction under noise}]
	\label{ex:hs3}
	For example, we know that $\varepsilon$-extraction with any level of confidence can be implemented with complexity $q=O(d\log(\frac{1}{\varepsilon}))$ using QS active learning for the class $\F_{d,HS}$ \ie for binary classification via halfspaces (refer Example~\ref{ex:hs2}). %If we look for a defense to such an attack,
	It follows from the earlier discussion that any defense that flips labels with a constant flipping probability $\rho$ does not work.
	This defense approach is similar to the case of \qquote{noisy oracles} studied extensively in the active learning literature \cite{Kaariainen06,KarpK07,Nowak11}.
	For example, from the machine-learning literature we know that if the flipping probability is exactly $\rho$ ($\rho\leq\frac12$),
	the AVERAGE algorithm (similar to our Algorithm \ref{alg:passive}, defined in Section \ref{sec:implementation}) $\varepsilon$-extracts $f^*$ with  $\tilde{O}(\frac{d^{2}}{(1-2\rho)^{2}}\log\frac{1}{\varepsilon})$
	labels \cite{key-3}. Under bounded noise where each label is flipped
	with probability at most $\rho$ ($\rho<\frac12$), the AVERAGE algorithm does not work
	anymore, but a modified Perceptron algorithm can learn with $\tilde{O}(\frac{d}{(1-2\rho)^{2}}\log\frac{1}{\varepsilon})$
	labels \cite{key-5} in a stream-based active learning setting, and a QS active learning algorithm proposed by Chen \etal~\cite{chen2017near} can also learn with the same number of labels. 
	%Finally, Chen \etal~\cite{chen2017near} present a query synthesis active learning algorithm that learns a halfspace when a noisy oracle flips the correct label response with constant probability $\rho$. If $\rho<\frac{1}{2}$, this algorithms learns  $f_{w}\in\F_{d,HS}$ such that $\|w-w^*\|_2\leq \varepsilon$ with probability at least $1-\delta$ using $\tilde{q}=d\cdot q_\rho(\varepsilon,\delta)$ queries where the value $q_\rho(\varepsilon,\delta)$ is $\tilde{O}(\log(\frac{1}{\varepsilon})+\log(\frac{1}{\delta}))$, and can be explicitly computed as function of $\rho$. %Under the adversarial noise where $\nu=\Omega(\varepsilon)$, \cite{key-6} proposes an algorithm that minimizes a sequence of hinge losses and uses $\tilde{O}(d\log\frac{1}{\varepsilon})$ labels.
	An adversary implementing the Chen \etal algorithm \cite{chen2017near} is even more efficient than the adversary $\tilde{A}$ defined in the proof of Proposition~\ref{prop} (\ie, the total number of queries only increases by a constant multiplicative factor instead of $\ln q(\epsilon,\delta)$). We validate the practical efficiency of this attack in Section~\ref{sec:implementation}. 		
	\end{example}

	\emph{2. Data-dependent randomization.}  Based on the
          outcome of the earlier discussion, we believe that a defense
          that aims to protect a hypothesis class against model
          extraction via QS active learning should implement
          data-dependent perturbation of the returned labels. That is,
          we are interested in a defense $D$ such that the probability
          $\rho_D(f^*,x)$ depends on the query input $x$ and
          the labeling function $f^*$.  For example, given
          a class $\F$ that can be extracted using an active learner
          $\L$ (in the QS scenario), if we consider a defense $D$ such
          that $\rho_D(f^*,x)\geq \frac{1}{2}$ for some instances,
          then the proof of Proposition~\ref{prop} does not work (the
          argument only works if there is a constant $c>0$ such that
          $\rho_D(f^*,x)\leq \frac{1}{2}-c$ for all $x$) and the
          effectiveness of the adversary $\tilde{A}$ is not guaranteed
          anymore\footnote{Intuitively, in the binary case if
            $\rho_D(f^*,x_i)\geq \frac{1}{2}$ then the definition of
            $y_i$ performed by $\tilde{A}$ in step 2 (majority vote) is
            likely to be wrong. However, notice that this is not
            always the case in the multiclass setting: For example,
            consider the case when the answer to query ${x_i}$ is
            defined to be wrong with probability $\geq\frac{1}{2}$
            and, when wrong, is sampled uniformly at random among the
            $k-1$ classes that are different to the true class
            $f^*(x)$, then if $k$ is large enough, $y_i$ defined via
            the majority vote is likely to be still correct.}.

\begin{example}[\underline{Halfspace extraction under noise}]
\label{ex:hs4}
For the case of binary classification via halfspaces, Alabdulmohsin
\etal~\cite{CIKM-14} design a system that follows this strategy. They
consider the class $\F_{d,HS}$ and design a learning rule that uses
training data to infer a distribution of models, as opposed to
learning a single model. To elaborate, the algorithm learns the mean
$\mu$ and the covariance $\Sigma$ for a multivariate Gaussian
distribution $\mathcal{N}(\mu,\Sigma)$ on $\F_{d,HS}$ such that any
model drawn from $\mathcal{N}(\mu,\Sigma)$ provides an accurate
prediction. The problem of learning such a distribution of classifiers
is formulated as a convex-optimization problem, which can be solved
quite efficiently using existing solvers. During prediction, when the
label for a instance $x$ is queried, a {\em new} $w$ is drawn at
random from the learned distribution $\mathcal{N}(\mu,\Sigma)$ and the
label is computed as $y=\sign(\langle w,x\rangle)$.  The authors show
that this randomization method can mitigate the risk of reverse
engineering without incurring any notable loss in predictive
accuracy. In particular, they use PAC active learning algorithms
\cite{cohn1994improving,BalcanBZ07} (assuming that the underlying
distribution $\D$ is Gaussian) to learn an approximation $\hat{w}$
from queries answered in three different ways: (a) with their
strategy, \ie using a new model for each query, (b) using a fixed
model to compute all labels, and (c) using a fixed model and adding
independent noise to each label, \ie $y=\sign(\langle
w,x\rangle+\eta)$ and $\eta\leftarrow[-1,+1]$. They show that the
geometric error of $\hat{w}$ with respect to the true model is higher
in the former setting (\ie in (a)) than in the others. %When the number of queries is fixed to 1000, 
On 15 different datasets from the UC
Irvine repository \cite{uci}, their strategy gives typically an order
of magnitude larger error.  We empirically evaluate this defense in
the context of model extraction using QS active learning algorithms in
Section~\ref{sec:implementation}.
		
		%Notice that the setting of our attack for the Alabdulmohsin \etal~\cite{CIKM-14}  paper presented  in Section~\ref{sec:implementation} is slightly different from the setting discussed here: here the performance metric is $\Err_\D(f)-\Err_\D(f^*)$, but in our attack, the metric is $\Pr(f(x)\neq f^*(x))$ ($x$ sampled according to the marginal of $\D$ over $\X$). We have $\Err_\D(f)-\Err_\D(f^*)\leq\Pr(f(x)\neq f^*(x))$, so $\Pr(f(x)\neq f^*(x))\leq\varepsilon$ implies an excess error $\varepsilon$. %\Songbai{I don't think we need to mention "but the converse direction requires additional conditions."}
		
		%Finally, \cite{key-8} considers a noise setting where the the oracle (\ie, server) answers with the label 1 with probability equals to $\phi(\mu^{\top}x)$, for some function $\phi$. This includes our attacking setting, however, theoretical results in \cite{key-8} are of a very general form, and it is unclear what its label complexity bound is under our setting.	
	\end{example}

%%%%%%%%%%%%%%%%%%%%%%%%%%%%%%%%%%%
\noindent
{\bf Continuous case:} Generalizing Proposition~\ref{prop} to the
continuous case does not seem straightforward, \ie when the target
model held by the MLaaS server is a real-valued function $f^*:\X
\rightarrow \R$; A detailed discussion about the continuous case 
appears in Appendix~\ref{subsec:continuous}.

\section{Implementation and Evaluation}
\label{sec:implementation}

For all experiments described below, we use an Ubuntu 16.04 server with 32 GB RAM, and an Intel i5-6600 CPU clocking 3.30GHz. We use a combination of datasets obtained from the \texttt{scikit-learn} library and the UCI machine learning repository \cite{uci}, as used by Tram\`{e}r \etal.

\subsection{Linear Models}
\label{subsec:linear}

We carried out experiments to validate our claims that query synthesis active learning can be used to successfully perform model extraction for linear models. Our experiments are designed to answer the following three questions: (1) Is active learning practically useful in settings without any auxiliary information, such as confidence values \ie in an oracle access setting?, (2) Is active learning useful in scenarios where the oracle is able to perturb the output \ie in a data-independent randomization setting?, and (3) Is active learning useful in scenarios where the oracle is able to perform more subtle perturbations \ie in a data-dependent randomization setting?

%\begin{enumerate}
%\itemsep0em
%\item Is active learning practically useful in settings without any auxiliary information, such as confidence values \ie in an oracle access setting?
%\item Is active learning useful in scenarios where the oracle is able to perturb the output \ie in a data-independent randomization setting?
%\item Is active learning useful in scenarios where the oracle is able to perform more subtle perturbations \ie in a data-dependent randomization %setting?
%\end{enumerate}

To answer these questions, we focused on learning the hypothesis class of $d$-dimensional half spaces. To perform model extraction, we implemented two QS algorithms \cite{alabdulmohsin2015efficient, chen2017near} to learn an approximation $w$, and terminate execution when $\lvert \lvert w^* - w \rvert \rvert_2 \leq \varepsilon$. The metric we use to capture efficiency is query complexity. To provide a monetary estimate of an attack, we borrow pricing information from the online pricing scheme of Amazon \ie $\$$0.0001 per query (more details are present in Table~\ref{aux}). We considered alternative stopping criteria, such as measuring the learned model's stability over the $N$ last iterations. Such a method resulted in comparable error and query complexity (refer Appendix \ref{app:alternate} for detailed results). For our experiments, the halfspace held by the server/oracle (\ie, the optimal hypothesis $w^*$) was learned using Python's \texttt{scikit-learn} library. Our experiments suggest that:

\begin{enumerate}
\itemsep0em
\item QS active learning algorithms are efficient for model extraction, with low query complexity and run-time. For the digits dataset ($d=64$), the dataset with the largest value of $d$ which we evaluated on, the active learning algorithm implemented required 900 queries to extract the halfspace with geometric error $\varepsilon \leq 10^{-4}$. This amounts to $\$0.09$ worth of queries. 
\item QS active learning algorithms are also efficient when the oracle flips the labels independently with constant probability $\rho$. This only moderately increases the query complexity (for low values of $\rho$). For the digits dataset of input dimensionality $d=64$, and a noise threshold $\rho=0.4$, our algorithm required 36546 queries (or $\$3.65$) to extract the halfspace with geometric error $\varepsilon \leq 10^{-4}$. 
\item State-of-the-art QS algorithms fail to recover the model when the oracle responds to queries using tailored model randomization techniques (refer subsection \ref{sec:defense}, specifically the algorithm by Alabdulmohsin \etal \cite{CIKM-14}). However, passive learning algorithms (refer Algorithm \ref{alg:passive}) are effective in this setting.
\end{enumerate}

In each figure, we plot the price (\ie $\$0.0001$ per query) for the most expensive attack we launch to serve as a baseline. We conclude by comparing our approach with the algorithm proposed by Lowd and Meek~\cite{LM05}.

%\vspace{2mm}
%\noindent{\bf Datasets Used:} We use a combination of datasets obtained from the \texttt{scikit-learn} library and the UCI machine learning repository \cite{uci}. These datasets are chosen for their diversity in number of features. Some of these datasets are traditionally used for training multiclass classifiers. However, all the datasets were modified such that each instance corresponds to one of two labels (-1 and +1) \ie these datasets are ideal for training binary classifiers. 

%\Irene{Varun, I uncommented the paragraph Datasets Used since I think we need it and the first sentence of this section actually says we are going to describe the dataset. But if this is not ok, then, feel free, comment it again. In this case, just remember to change the intro of this section too.}

\begin{figure}
        \centering
        \includegraphics[width=0.7\linewidth]{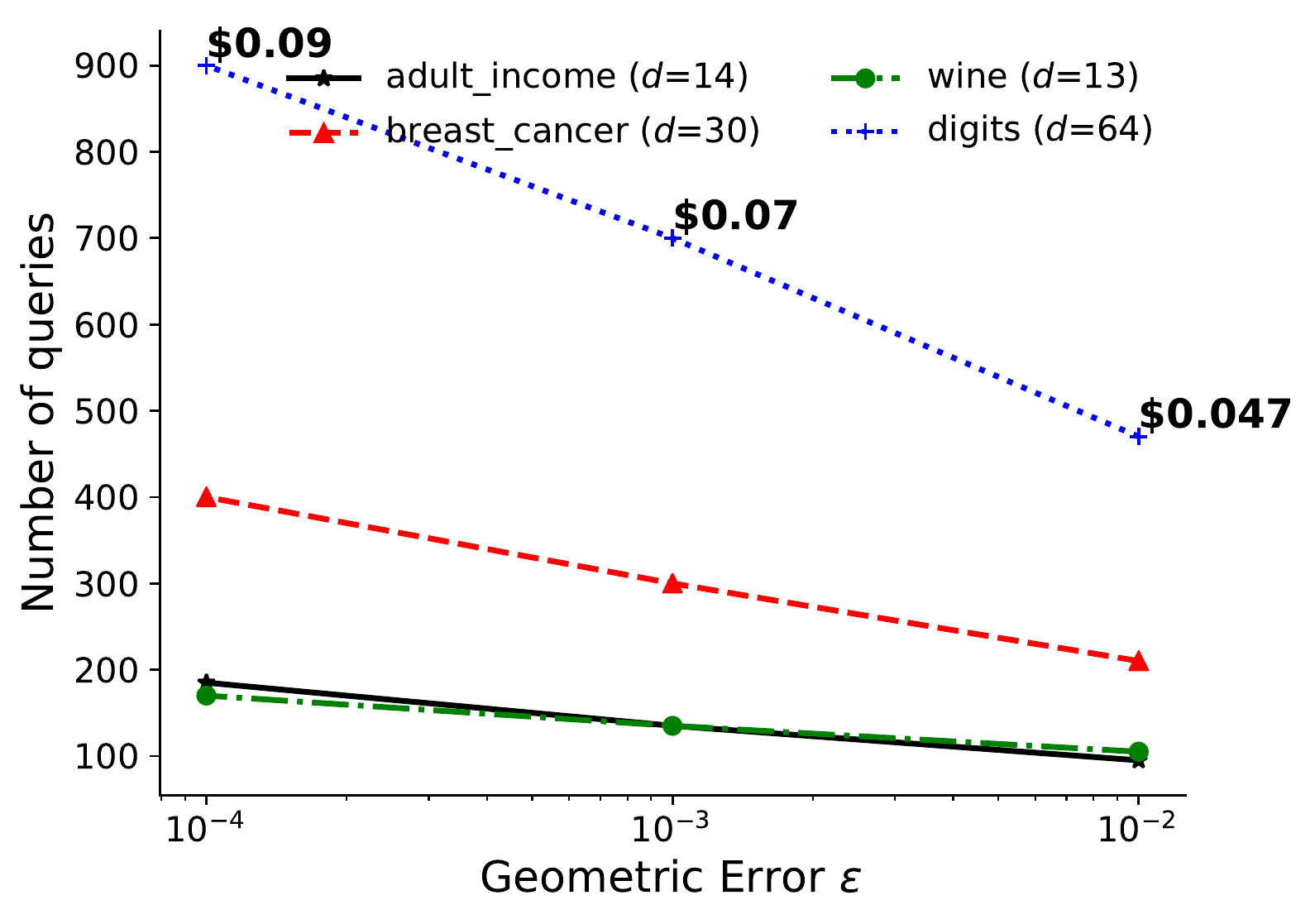}
        \compactcaption{\footnotesize \footnotesize Number of queries needed for halfspace extraction using the version space approximation algorithm. Note that the asymptotic query complexity for this algorithm is $O(d\log\frac{1}{\varepsilon})$. This explains the increase in query complexity as a function of $d$.}
        \label{fig:aaai15}
\vspace{-4mm}
\end{figure}

\vspace{2mm}
\noindent{\bf Q1. Usefulness in an oracle access setting:} We implemented Version Space Approximation proposed by Alabdulmohsin \etal \cite{alabdulmohsin2015efficient} in approximately 50 lines of MATLAB. This algorithm operates iteratively, based on the principles of version space learning. A version space \cite{mitchell1978version} is a hierarchical representation of knowledge. It can also be thought of as the subset of hypotheses consistent with the training examples. In each iteration, the algorithm first approximates a version space, and then synthesizes an instance that reduces this approximated version space quickly. The final query complexity for this algorithm is $O(d\log \frac{1}{\varepsilon})$. 

Figure \ref{fig:aaai15} plots the number of queries needed to extract a halfspace as a function of termination criterion \ie geometric error $\varepsilon$. As discussed earlier, the query complexity is dependent on the dimensionality of the halfspace to be extracted. Across all values of dimensionality $d$, observe that with the exponential decrease in error $\varepsilon$, the increase in query complexity is linear - often by a small factor ($1.3\times - 1.5\times$). The implemented query synthesis algorithm involves solving a convex optimization problem to approximate the version space, an operation that is potentially time consuming. However, based on several runs of our experiment, we noticed that the algorithm always converges in $< 2$ minutes. %(for our choice of operational parameters). %The cost of extraction is $\approx \$ 0.5$ based on current pricing standards.

While the equation solving attack proposed by Tram\`{e}r \etal \cite{TZJRR16} requires fewer queries, it also requires the actual value of the prediction output \ie $\langle w^*, x\rangle$ as auxiliary information. On the other hand, extraction using query synthesis does not rely on any auxiliary information returned by the MLaaS server to increase its efficiency \ie the only input needed for query synthesis-based extraction attacks is {\sign}($\langle w^*, x\rangle $). %Such extraction attacks do not rely on any distributional assumption made on the input space. 

\vspace{2mm}
\noindent{\bf Q2. Resilience to data-independent noise:} An intuitive defense against model extraction might be to flip the sign of the prediction output with independent probability $\rho$ \ie if the output $y \in \{ 1, -1\}$, then $Pr[ y \neq \sign \langle w^*, x \rangle] = \rho < \frac{1}{2}$ (refer subsection \ref{sec:defense}). This setting (\ie, noisy oracles) is extensively studied in the machine learning community. Trivial solutions including repeated sampling to obtain a batch where majority voting (determines the right label) can be employed; if the probability that the outcome of the vote is correct is represented as $1-\alpha$, then the batch size needed for the voting procedure is $k=O(\frac{\log\frac{1}{\alpha}}{|\rho - 0.5|^2}$) \ie there is an increase in query complexity by a (multiplicative) factor $k$, an expensive proposition. While other solutions exist \cite{yan2016active, nowak2009noisy}, we implemented the dimension coupling ($DC^2$) framework proposed by Chen \etal \cite{chen2017near} in approximately 150 lines of MATLAB. The dimension coupling framework reduces a $d-$dimensional learning problem to $d-1$ lower-dimensional sub-problems. It then appropriately aggregates the results to produce a halfspace. This approach is resilient to noise \ie the oracle can flip the label with constant probability (known a priori) $\rho < \frac{1}{2}$, and the algorithm will converge with probability $1-\delta$. The query complexity for this algorithm is $\tilde{O}(d\ (\log \frac{1}{\varepsilon}+\log \frac{1}{\delta}))$.

\begin{figure}
\centering
\subfigure[Adult Income]{\label{fig:evaluation:aaai17_adult_income}\includegraphics[width=0.45\linewidth]{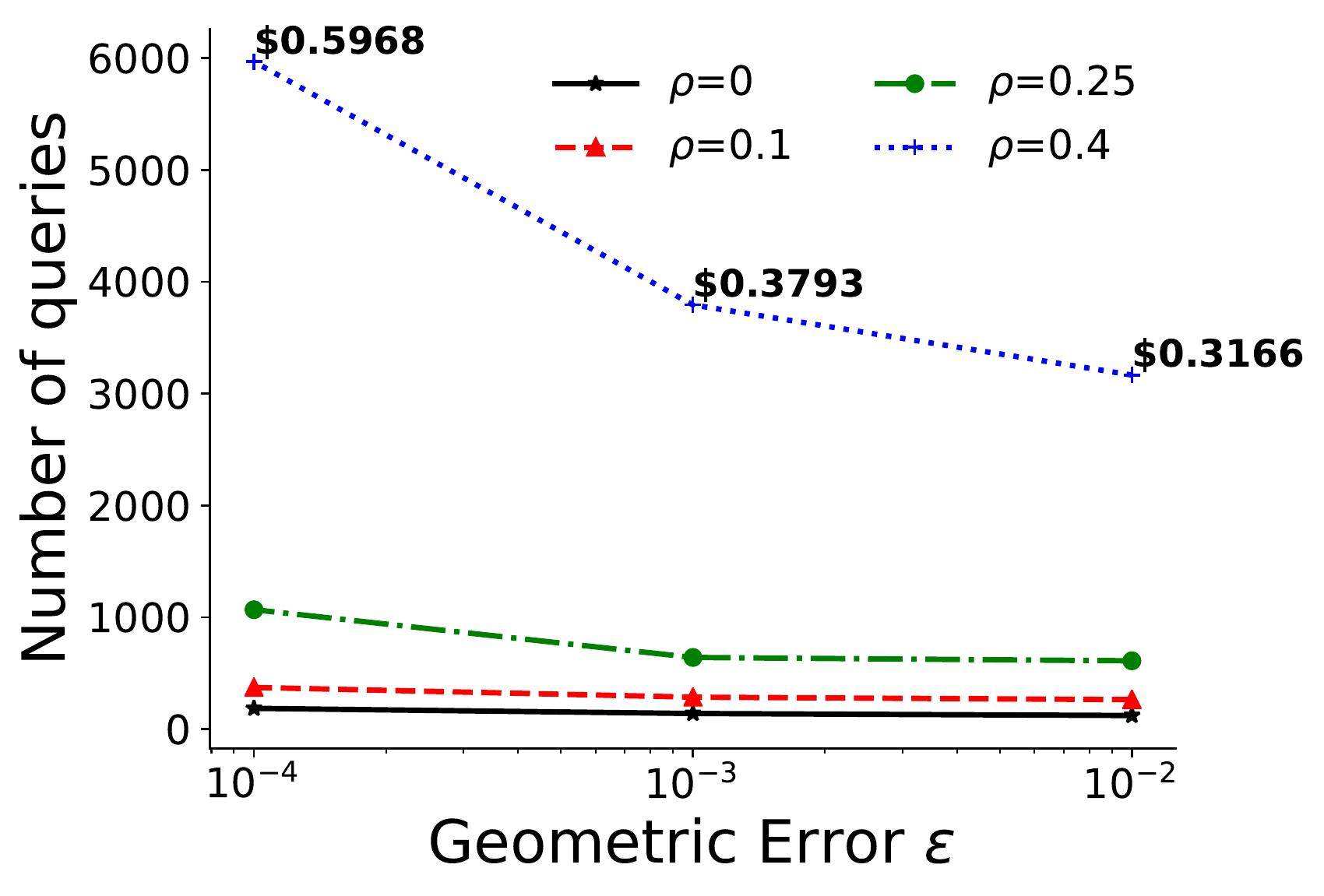}}
\subfigure[Breast Cancer]{\label{fig:evaluation:aaai17_breast_cancer}\includegraphics[width=0.45\linewidth]{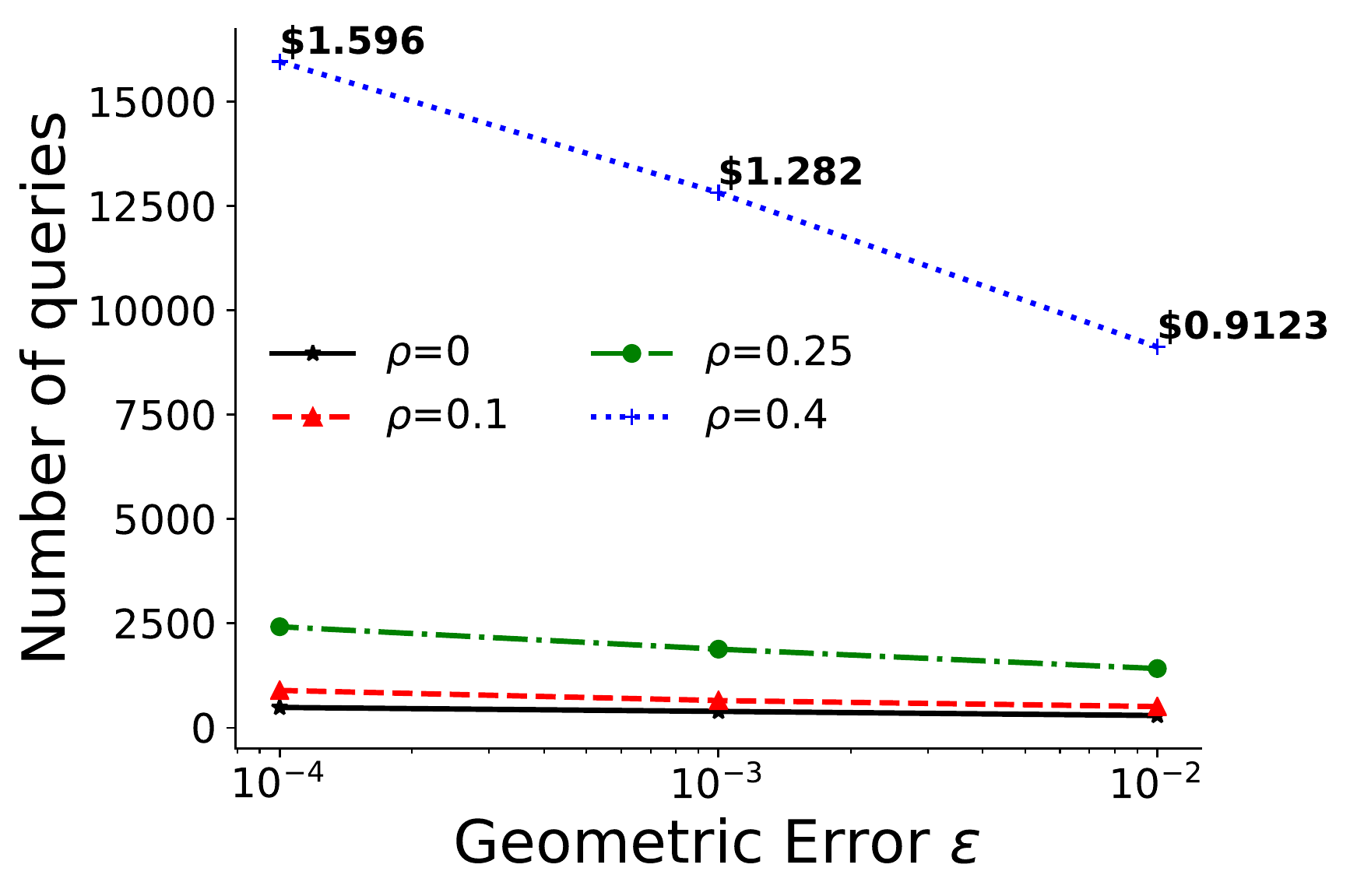}}
\\
\subfigure[Digits]{\label{fig:evaluation:aaai17_digits}\includegraphics[width=0.45\linewidth]{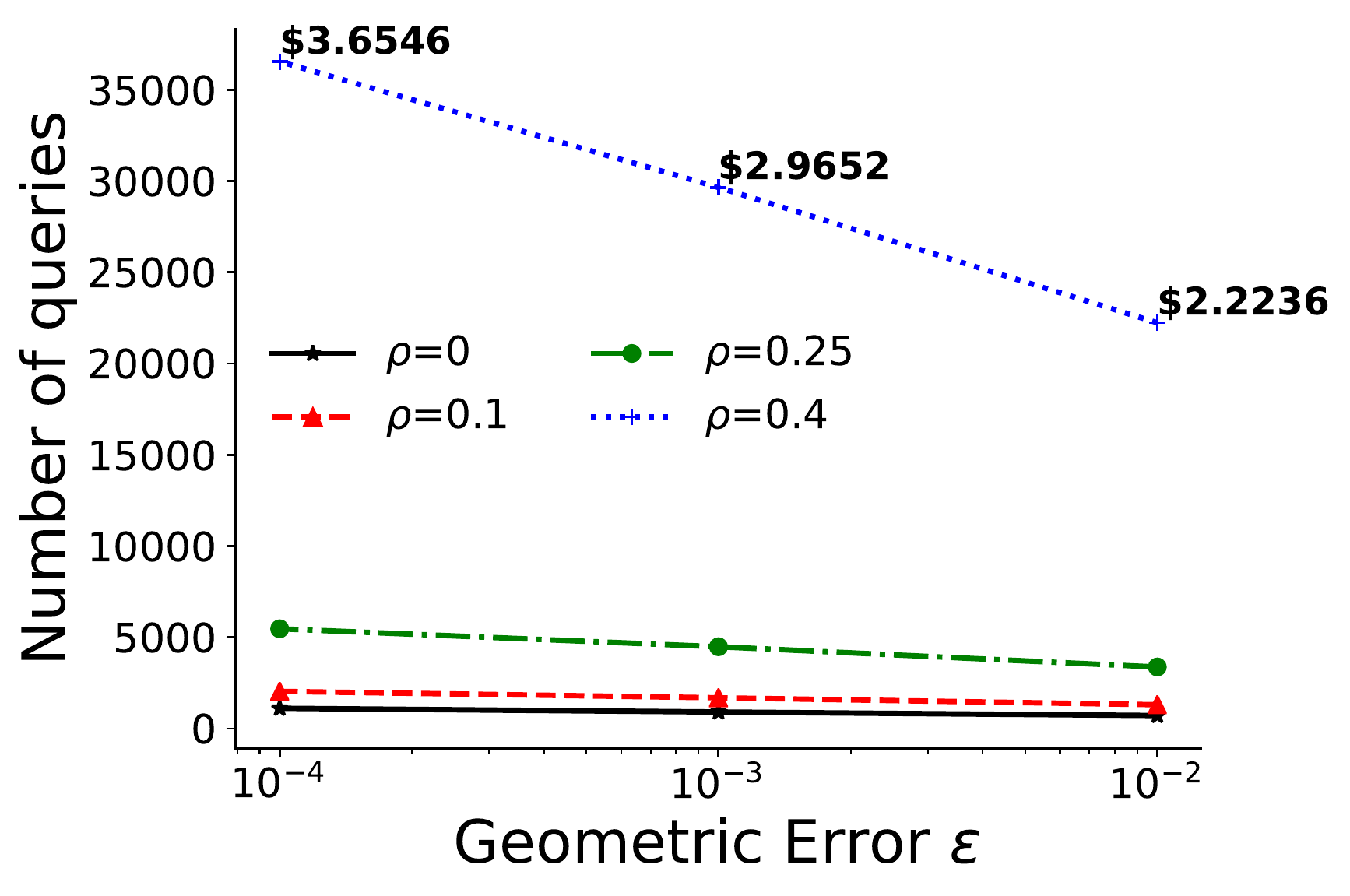}}%
\subfigure[Wine]{\label{fig:evaluation:aaai17_wine}  \includegraphics[width=0.45\linewidth]{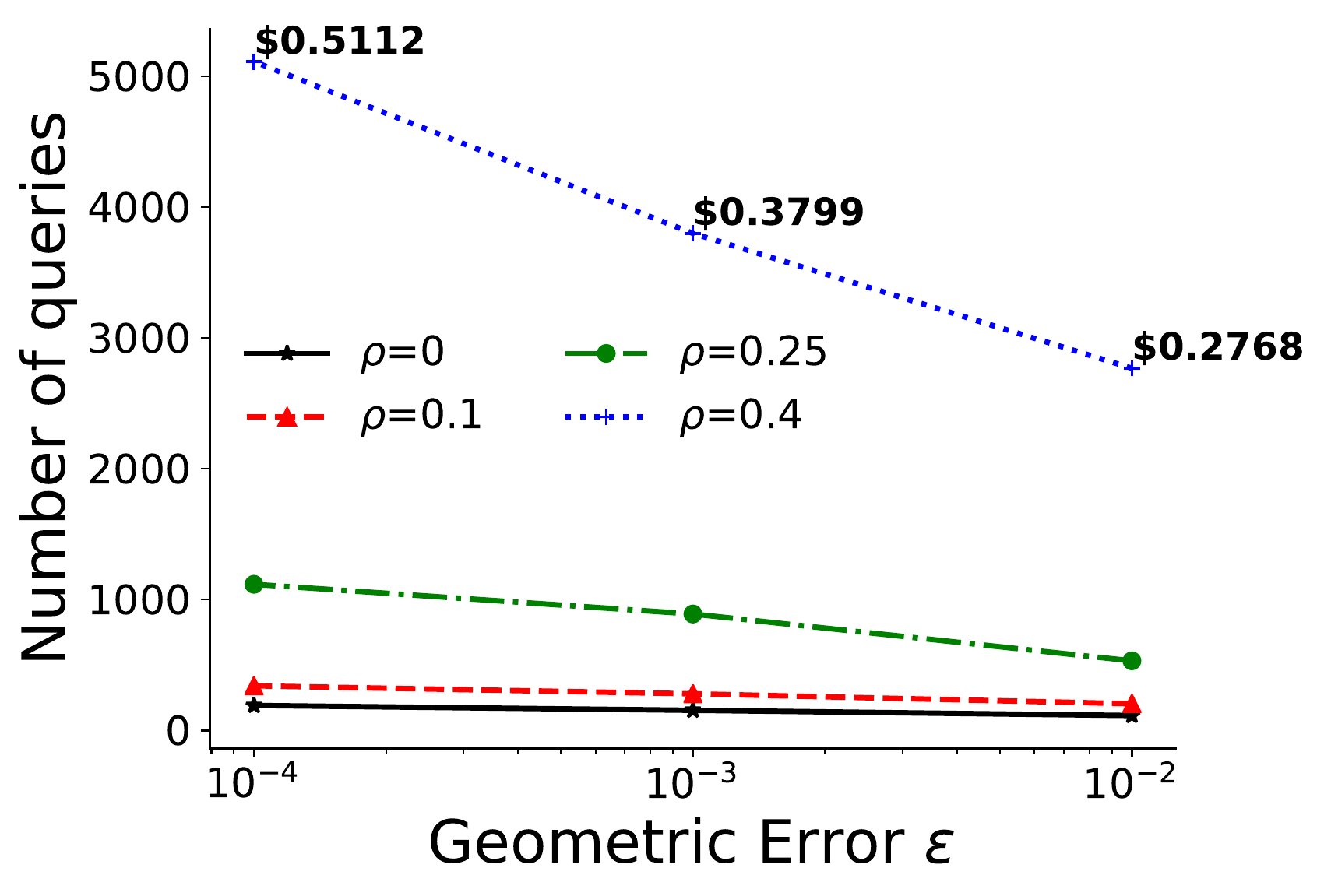}}%
\compactcaption{\footnotesize Number of queries needed for halfspace extraction using the dimension coupling algorithm. Note that the asymptotic query complexity for this algorithm is $\tilde{O}(d(\log\frac{1}{\varepsilon}+\log\frac{1}{\delta}))$. This explains the increase in query complexity as a function of $d$.}
\label{fig:aaai17}
\vspace{-4mm}
\end{figure}

%\Varun{Question for anyone: we see an increase in query complexity with increasing values of $\rho$, but this is not captured in the bounds. Why?} \kc{There is a dependence on $\rho$ but a very complex one -- let me recover this for you. From their Theorem 1, the bound appears to be $C(\rho)( d\ (log \frac{1}{\varepsilon}+log \frac{1}{\delta})$, where $C(\rho)$ is a function of $\rho$ that is approximately $O( \frac{ \rho \log^3(1/\rho)}{\log^2 2(1 - \rho)})$} \Varun{Thanks K. Should we add your note to the paper as well?}
%for $\rho$ close to $\frac{1}{2}$ and $O( \rho \log^2 \frac{1}{\rho} \log \log \frac{1}{\rho})$ for $\rho$ close to $0$.} 

The results of our experiment are presented in Figure \ref{fig:aaai17}. The algorithm is successful in extracting the halfspace for a variety of $\rho$ values. The exact bound is $C(\rho)( d\ (\log \frac{1}{\varepsilon}+\log \frac{1}{\delta})$, where $C(\rho)$ is a function of $\rho$ that is approximately $O( \frac{ \rho \log^3(1/\rho)}{\log^2 2(1 - \rho)})$. Thus, there is a multiplicative increase in the number of queries with increase in $\rho$. This introduces a modest increase in complexity in comparison to the noise-free setting. While the increase in pricing is $\approx 40\times$, this results in a worst case expenditure of $\approx \$3.6$ (see Figure \ref{fig:evaluation:aaai17_digits}). The time (and number of queries) taken for convergence is proportional to $\rho$, ranging from $1-20$ minutes for successful completion. 

%For the worst query complexity (Fig. \ref{fig:evaluation:aaai17_digits}), the cost of extraction is $\approx \$3.7$.

\vspace{2mm}
\noindent{\bf Q3. Resilience to data-dependent noise:} As alluded to in subsection \ref{sec:defense}, another defense against extraction involves learning a family of functions very similar to $w^*$ such that they all provide accurate predictions with high probability. Proposed by Alabdulmohsin \etal \cite{CIKM-14}, data-dependent randomization enables the MLaaS server to sample a random function for each query \ie for each instance $x_i$, the MLaaS server obtains a new $w_i \thicksim \mathcal{N}(\mu, \sigma)$ and responds with $y_i = {\sign}(\langle w_i, x_i \rangle)$. Thus, this approach can be thought of as flipping the sign of the prediction output with probability $\rho_D(w^*, x_i)$ (see subsection \ref{sec:defense}). 

In this algorithm, a separation parameter $C$ determines how close the samples from $\mathcal{N}(\mu, \sigma)$ are; larger the value of $C$, closer each sample is (refer Section 4 in \cite{CIKM-14} for more details). We measure the value of $\rho_D(w^*, x_i)$ as a function of $C$ for those $x_i$ values generated by the dimension coupling algorithm. $\rho_D(w^*, x_i)$ is estimated by (a) obtaining $w_{1}, \cdots, w_{n} \thicksim \mathcal{N}(\mu, \sigma)$, for $n=1000$, and using them to classify $x_i$ to obtain $y_{1}=\sign(\langle w_1, x_i \rangle), \cdots, y_{n}$, and (b) obtaining the percentage of the prediction outputs that is not equal to ${\sign} (\langle w^*, x_i \rangle)$. Our hope was that if the value of ${\tt max}_{\forall x_i}~\rho_D(w^*, x_i) < \frac{1}{2}$, then an adversary similar to $\tilde{A}$ defined in Proposition \ref{prop} could be used to perform extraction.
 
% To observe this effect, we plot the accuracy of a system deploying this defense as a function of the value of $C$, for three different datasets (Figure \ref{fig:distance}). Here, we define accuracy as the percentage of outputs produced by the defense strategy that match the output produced by the ideal halfspace (one learned using passive PAC learning), when the instances used for testing come from the same distribution of data used for training. That is, the dataset was split into training data (used for generating the ideal halfspace and obtaining $\mu$ and $\sigma$), and test data (used for measuring accuracy). We observe that the general trend is as the value of $C$ increases, so does accuracy. However, for lower values of $C$ \ie the values that guarantee more security, the values of accuracy are low. This is conflicting, as no defense strategy must sacrifice utility for security. 

\begin{figure}
	\centering
	\includegraphics[width=.7\linewidth]{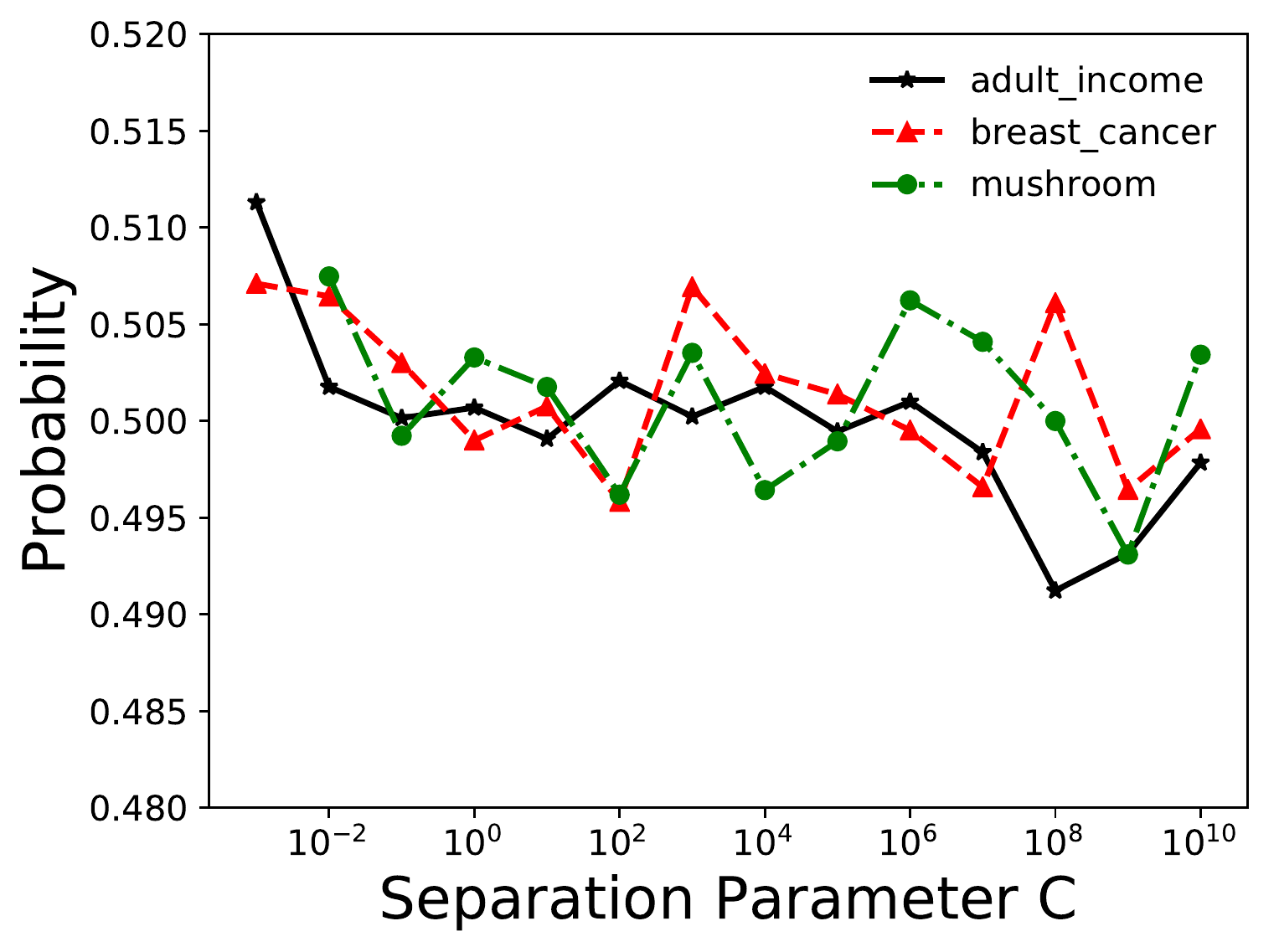}
	\compactcaption{\footnotesize ${\tt average} \rho_D(w^*, x_i) \approx \frac{1}{2} \pm \gamma$; $x_i$ synthesized by the dimension coupling algorithm.}
	\label{fig:avg_rho}
\vspace{-4mm}
\end{figure}

Figure \ref{fig:avg_rho} suggests otherwise; the average value of $\rho_D(w^*, x_i) \approx \frac{1}{2} \pm \gamma$ for some small $\gamma > 0$. Since any adversary will be unable to determine a priori the inputs for which this value is greater than half, neither majority voting, nor the vanilla dimension coupling framework will help extract the halfspace. We believe this is the case for current state-of-the-art algorithms as the instances they synthesize are "close" to the optimal halfspace. To validate this claim, we measured this distance for both the algorithms \cite{alabdulmohsin2015efficient,chen2017near}.  We observed that a majority of the points are very close to the halfspace in both cases (see Figure \ref{fig:distance} in Appendix \ref{app:defense} for more details).

\begin{algorithm}
\begin{algorithmic}[1]
\State{{\bf Input}: variance upper bound $\hat{\sigma}\geq \frac{1}{\sqrt{d}}$, target error $\varepsilon$}
\State{$m \gets \frac{(15\pi)^2}{\varepsilon^2}d\max(1,d\hat{\sigma}^2)\log\frac{2d}{\delta}$, $l \gets \frac{1}{12d\hat{\sigma}}$}
\State{Draw $x_1,x_2,\dots,x_m\in\mathbb{S}^{d-1}$ uniformly at random, and query their labels $y_1,y_2,\dots,y_m$}
\State{$v \gets \sum_{i=1}^m y_i x_i$}
\If{$\left\Vert v\right\Vert \geq l$}
\State{Return $w=\frac{v}{\left\Vert v\right\Vert}$}
\Else
\State{Return \textsc{fail}}
\EndIf
\end{algorithmic}

\caption{\footnotesize \label{alg:passive} Passive Learning Algorithm that breaks \cite{CIKM-14}}
\end{algorithm}

Such forms of data-dependent randomization, however, are not secure against traditional passive learning algorithms. Such an algorithm takes as input an estimated upper bound $\hat{\sigma}$ for $\sigma$. The algorithm first draws $\tilde{O}(\frac{d}{\varepsilon^{2}}\max(1,d\hat{\sigma}^{2}))$ instances from the $d-$dimensional unit sphere $\mathbb{S}^{d-1}$ uniformly at random, and proceeds to have them labeled - by the oracle defined in \cite{CIKM-14} in this case. It then computes the average $v=\sum_{i=1}^{m}y_{i}x_{i}$. $\ensuremath{w=\frac{v}{\left\Vert v\right\Vert }}$, the direction of $v$, is the algorithm's estimate of the classifier $w^*$, and the length of $v$ is used as an indicator of whether the algorithm succeeds:
if this estimated upper bound is correct (i.e. $\sigma\leq\hat{\sigma}$), then with high probability, $\left\Vert w-w^* \right\Vert \leq\varepsilon$; otherwise it outputs \textsc{fail}, indicating the variance bound $\hat{\sigma}$ is incorrect. In such situations, we can reduce $\hat{\sigma}$ and try again. A detailed proof of the algorithm's guarantees is available in Appendix~\ref{app:alg}.

\begin{figure}
	\centering
	\includegraphics[width=.7\linewidth]{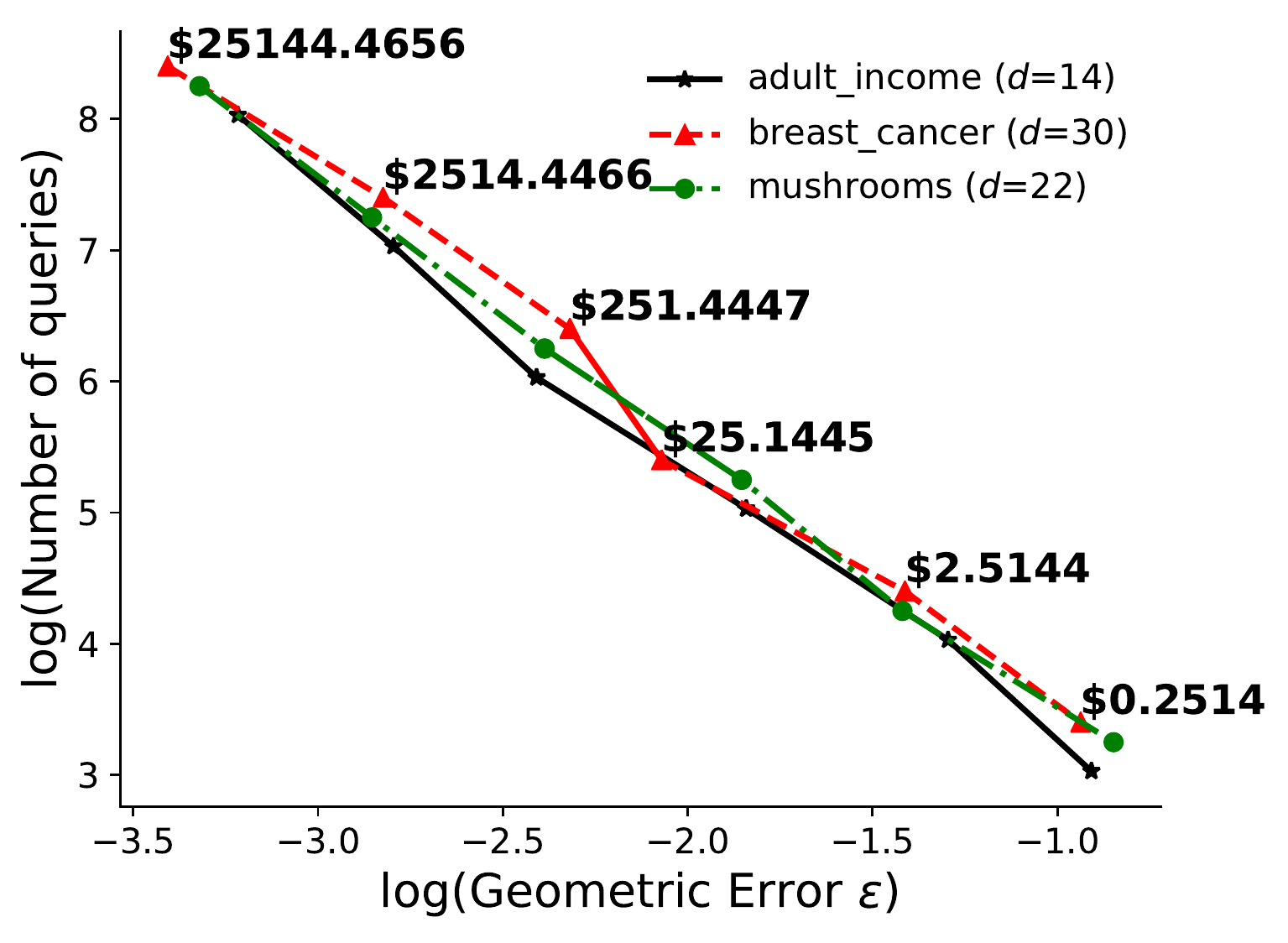}
	\compactcaption{\footnotesize log(Number of queries) needed for halfspace extraction (protected by the defense strategy proposed in \cite{CIKM-14}) using Algorithm \ref{alg:passive}. Note that the asymptotic query complexity for this algorithm is $O(\frac{1}{\varepsilon^2}d\max(1, d\hat{\sigma}^2)\log\frac{2d}{\delta})$. This explains the increase in query complexity as a function of $d$ and $\varepsilon$. The large value of $\frac{Cd}{\varepsilon^2}$ dominates the query complexity in this algorithm. The price is plotted for the attack on the breast cancer dataset.} 
	\label{fig:cikm}
\vspace{-4mm}
\end{figure}

While the asymptotic bounds for Algorithm \ref{alg:passive} are larger than the active learning algorithms discussed thus far, the constant $\mathcal{C}=(15\pi)^2$ can be reduced by a multiplicative factor to reduce the total number of queries used \ie $\frac{\mathcal{C}}{100}$ or $\frac{\mathcal{C}}{1000}$ etc. In Figure \ref{fig:cikm}, we observe that extracting halfspaces with geometric error $\varepsilon \approx 10^{-1}$ requires $\leq 10^4$ queries. While achieving $\varepsilon \approx 10^{-3}$ requires $\approx 10^7$ queries, the algorithm can be executed in parallel enabling faster run-times.

\vspace{2mm}
\noindent{\bf Lowd and Meek Baseline:} The algorithm proposed by Lowd and Meek \cite{LM05} can also be used to extract a halfspace. However, note that this algorithm can only operate in a noise-free setting. This is a severe limitation if a setting where the MLaaS employs defense strategies. From Table \ref{lowdmeek}, one can observe that the number of queries required to extract the halfspace is more than the query synthesis algorithms we implemented. For example, consider the breast cancer dataset. The version space algorithm is able to extract a halfspace at a distance of $\varepsilon \leq 10^{-4}$ with 400 queries (or $\$0.04$). However, the algorithm proposed by Lowd and Meek takes 970 queries for extraction. Additionally, the geometric error of the extracted halfspaces are also higher than those extracted in the query synthesis case. The query complexity of the Lowd and Meek algorithm is $O(d \log (\frac{1}{a\varepsilon}))$, where $a = \min_{i=1,\cdots,d} \frac{|w^*_i|}{\left\Vert w^*\right\Vert}$ ($w^*_i$ is the $i$-th coordinate of the groundtruth classifier $w^*$). This is worse than the $O(d \log (\frac{1}{\varepsilon}))$ query complexity of classical active learning algorithms. While this algorithm is not tailored to minimize the geometric error, we believe that these results further validate our claim that query synthesis active learning is a promising direction to explore.

\begin{table}
	\begin{center}
		\begin{tabular}{  l  c c c }
			\toprule
			{\bf Dataset} & {\bf Queries} & {\bf $\varepsilon$} & {\bf Slowdown} \\ \midrule
% 			Iris & 80 & 0.35374 \\ 
% 			Mushrooms &  926 & 5.4612 \\ 
		    Wine & 189 & 0.071 & 1.67$\times$\\
            Breast Cancer & 940 & 0.162 & 3.19$\times$ \\
            Digits & 1879 & 0.665 & 2.62$\times$\\
		\bottomrule
		\end{tabular}
	\end{center}
\compactcaption{\footnotesize Number of queries and geometric error observed after extracting halfspaces using the line search procedure proposed by Lowd and Meek. Observe that the geometric error in some cases is large. Slowdown indicates the ratio between number of queries taken for the Lowd and Meek procedure and those taken by the DC$^2$ algorithm~\cite{chen2017near} for $\varepsilon=0.01$, and $\rho=0$.}
\vspace{-4mm}
\label{lowdmeek}
\end{table}

\subsection{Non-Linear Models}
\label{subsec:nonlinear}

In this section, we report experimental results for extraction efficacy for non-linear decision boundaries, specifically for kernel SVMs and decision trees. Our experiments are designed to answer the following two questions: (1) Is QS active learning practically useful to extract non-linear models (\ie, kernel SVMs) in an oracle access setting?, and (2) Is active learning useful to extract non-linear models (\ie, decision trees) in scenarios where the ML server does not reveal any auxiliary information? As before, we use the same datasets as in Tram\`{e}r et al.~\cite{TZJRR16}. The continuous variables are made discrete by binning (\ie dividing into groups), and are then one-hot-encoded. Our experiments suggest that:

\begin{enumerate}
\itemsep0em
\item Utilizing the extended adaptive training (EAT) approach (refer subsection~\ref{subsec:kSVM}) is efficient for extracting kernel SVMs. Our approach improves query complexity by 5$\times$-224$\times$.
\item Utilizing the IWAL algorithm (refer subsection~\ref{subsec:iwal}) enables extracting a decision tree in the absence of {\em any} auxiliary information, with a nominal increase in query complexity (14$\times$). 
\end{enumerate}

\noindent{\bf Q1. Kernel SVM:} As discussed in subsection~\ref{subsec:kSVM}, Tram\`{e}r et al. use {\em adaptive retraining} - a procedure to locate points close to the decision boundary - to obtain points used to seed their attack. Using these (labeled) points, they are able obtain the parameters which were used to instantiate the oracle using an extract-and-test approach (more specifically grid search). While techniques in active learning can not be used to expedite the grid search procedure, our experiments suggest that they are able to select more {\em informative} points i.e. points of greater uncertainty, with far fewer queries. With insight from the work of Bordes et al.~\cite{LASVM}, we propose the {\em extended adaptive retraining approach} (refer subsection~\ref{subsec:kSVM}) to obtain uncertain points. Additionally, we measure uncertainty with respect to a model that we train locally\footnote{such a local model is seeded with uniformly random points labeled by the oracle}, eliminating redundant queries to the oracle. To compare the efficiency of our algorithm, we re-execute the adaptive retraining procedure, and present our results in Table~\ref{table:kernelsvm}.

\begin{table}
\small
\centering
\begin{tabular}{c|c|c|c|c}
\hline
\toprule
\multicolumn{1}{c|}{\bf Dataset} & \multicolumn{2}{c|}{\bf Adaptive Retraining} & \multicolumn{2}{c}{\bf EAT}\\
\cline{2-5}
\multicolumn{1}{c|}{} &  Queries &  Accuracy &  Queries & Accuracy \\
\midrule
Mushroom   		& 11301	& 98.5	& 1001	& 94.5 \\
Breast Cancer  		& 1101	& 99.3	& 119	& 96.4 \\
Adult			& 10901	& 96.98	& 48	& 98.2 \\
Diabetes		& 901	& 98.5	& 166	& 94.8 \\
\bottomrule
\end{tabular}
	\compactcaption{\footnotesize \footnotesize Extraction of a kernel SVM model. Comparison of the query complexity and test accuracy (in $\%$) obtained running Tram\`{e}r \etal adaptive retraining vs. extended adaptive retraining.}
\label{table:kernelsvm}
\end{table}

It is clear that our approach is more query efficient in comparison to Tram\`{e}r \etal (between 5$\times$-224$\times$), with comparable test accuracy. These advantages stem from (a) using a more informative metric of uncertainty than the distance from the decision boundary, and (b) querying labels of only those points which the local model is uncertain about. 

\noindent{\bf Q2. Decision Trees:} Tram\`{e}r et al. propose a path finding algorithm to determine the structure of the server-hosted decision tree. They rely on the server's response to incomplete queries, and the addition of node identifiers to the generated outputs to recreate the tree. From our analysis presented in Table~\ref{aux}such flexibility is not readily available in most MLaaS providers. As discussed earlier (refer subsection~\ref{subsec:iwal}), we utilize the IWAL algorithm proposed by Beygelzimer et al.~\cite{iwal:2010} that iteratively refines a learned hypothesis. It is important to note that the IWAL algorithm is more general, and does not rely on the information needed by the path finding algorithm. We present the results of extraction using the IWAL algorithm below in Table~\ref{table:iwal}.

\begin{table}
\small
\centering
\begin{tabular}{c|c|c|c|c}
\hline
\toprule
\multicolumn{1}{c|}{\bf Dataset} & \multicolumn{1}{|c|}{\bf Oracle} & \multicolumn{1}{c|}{\bf Path Finding} & \multicolumn{2}{c}{\bf IWAL}\\%& \multicolumn{2}{c|}{Langford Subsampled} \\%& \multicolumn{2}{c|}{Langford Mixed}\\
\cline{2-5}
	\multicolumn{1}{c|}{} & \multicolumn{1}{|c|}{ Accuracy} &  Queries &  Queries &  Accuracy\\% & num queries & Test Accuracy\\% & num queries & Test Accuracy\\
\midrule
Adult   	& 81.2	& 18323	& 244188& 80.2\\%	& 24426	& 69.1	& 48841	& 76.4\\
Steak  		& 52.1	& 5205	& 1334	& 73.1\\%	& 185	& 61.1	& 269	& 45.1\\
Iris		& 86.8	& 246	& 361	& 89.4\\%	& 35	& 84.1	& 85	& 84.1\\
GSShappiness	& 79	& 18907	& 254892& 79.3\\%	& 25621	& 69.1	& 51019	& 87.6\\
\bottomrule
\end{tabular}
\compactcaption{\footnotesize \footnotesize Extraction of a decision tree model. Comparison of the query complexity and test accuracy (in $\%$) obtained by running path finding (Tram\`{e}r \etal) vs. IWAL algorithm. The test accuracy  (in $\%$) of the server-hosted oracle is presented as a baseline.}
\label{table:iwal}
\vspace{-2mm}
\end{table}

In each iteration, the algorithm learns a new hypothesis, but the efficiency of the approach relies on the hypothesis used preceding the first iteration. To this end, we generate inputs uniformly at random. Note that in such the {\em uniform} scenario, we rely on zero auxiliary information. We can see that while the number of queries required to launch such extraction attacks is greater than in the approach proposed by Tram\`{e}r et al., such an approach obtains comparable test error to the oracle. While the authors rely on certain distributional assumptions to prove a label complexity result, we empirically observe success using the uniform strategy. Such an approach is truly powerful; it makes limited assumptions about the MLaaS provider and any prior knowledge.%The IWAL algorithm provides theoretical guarantees of convergence when (a) the labels are not noisy, and (b) the dataset used comes from the same distribution with which the oracle was trained. However, a practical attack may have noisy labels (i.e. when attack dataset is labeled by a suboptimal oracle), or may not be from the same distribution (i.e. sampled uniformly at random). However, we empirically observe success using the uniform strategy. Such an approach is truly powerful; it makes limited assumptions about the MLaaS provider and any prior knowledge.

\section{Discussion}
\label{sec:discussion}

We begin our discussion by highlighting algorithms an adversary could use if the assumptions made about the operational ecosystem are relaxed. Then, we discuss strategies that can potentially be used to make the process of extraction more difficult, and shortcomings in our approach. 

\subsection{Varying The Adversary's Capabilities}
The operational ecosystem in this work is one where the adversary is able to synthesize data-points de novo to extract a model through oracle access. %We believe that this is the most realistic ecosystem that mirrors real-world model extraction \ie one where the adversary has no prior knowledge about the data used for training the machine learning model. 
In this section, we discuss other algorithms an adversary could use if this assumption is relaxed. We begin by discussing other models an adversary can learn in the query synthesis regime, and move on to discussing algorithms in other approaches.

\vspace{1mm}
\noindent{\em Equivalence queries.} In her seminal work, Angluin \cite{angluin1987learning} proposes a learning algorithm, $L^{*}$, to correctly learn a regular set from any minimally adequate teacher, in polynomial time. For this to work, however, {\em equivalence queries} are also needed along with membership queries. Should MLaaS servers provide responses to such equivalence queries, different extraction attacks could be devised. To learn linear decision boundaries, Wang \etal \cite{wang2015active} first synthesize an instance close to the decision boundary using labeled data, and then select the real instance closest to the synthesized one as a query. Similarly, Awasthi \etal \cite{awasthi2013learning} study learning algorithms that make queries that are close to examples generated from the data distribution. These attacks require the adversary to have access to some subset of the original training data. In other domains, program synthesis using input-output example pairs \cite{gulwani2012synthesis,wang2017interactive,drachsler2017synthesis,peleg2018abstraction} also follows a similar principle.

\vspace{1mm}
If the adversary had access to a subset of the training data, or had prior knowledge of the distribution from which this data was drawn from, it could launch a different set of attacks based on the algorithms discussed below. 

\noindent{\em Stream-based selective sampling.} Atlas \etal \cite{atlas1990training} propose selective sampling as a form of directed search (similar to Mitchell \cite{mitchell1982generalization}) that can greatly increase the ability of a connectionist network (\ie power system security analysis in their paper) to generalize accurately. Dagan \etal \cite{dagan1995committee} propose a method for training probabilistic classifiers by choosing those examples from a stream that are {\em more informative}. Lindenbaum \etal \cite{lindenbaum1999selective} present a lookahead algorithm for selective sampling of examples for nearest neighbor classifiers. The algorithm looks for the example with the highest utility, taking its effect on the resulting classifier into account. Another important application of selective learning was for feature selection \cite{liu2004selective}, an important preprocessing step. Other applications of stream-based selective sampling include sensor scheduling \cite{krishnamurthy2002algorithms}, learning ranking functions for information retrieval \cite{yu2005svm}, and in word sense disambiguation \cite{fujii1998selective}. 

\noindent{\em Pool-based sampling.} Dasgupta \cite{dasgupta2011two} surveys active learning in the non-separable case, with a special focus on statistical learning theory. He claims that in this setting, AL algorithms usually follow one of the following two strategies - (i) Efficient search in the hypothesis spaces (as in the algorithm proposed by Chen \etal \cite{chen2017near}, or by Cohn \etal \cite{cohn1994improving}), or (ii) Exploiting clusters in the data (as in the algorithm proposed by Dasgupta \etal \cite{dasgupta2008general}). The latter option can be used to learn more complex models, such as decision trees. As the ideal halving algorithm is difficult to implement in practice, pool-based approximations are used instead such as uncertainty sampling and the query-by-committee (QBC) algorithm \cite{freund1997selective,tong2001support,brinker2003incorporating}. Unfortunately, such approximation methods are only guaranteed to work well if the number of unlabeled examples (i.e. pool size) grows exponentially fast with each iteration. Otherwise, such heuristics become crude approximations and they can perform quite poorly.

%\section{Discussion}
%\label{defense}

%Through the contents in the earlier sections, we wish to reaffirm our claim that model extraction (if formalized as an active learning problem) in a pay-per-use ecosystem is inevitable. 
%In this section, we discuss strategies that can potentially be used to make the process of extraction more difficult, and shortcomings in our approach. 

\subsection{Complex Models}

%A natural question that arises is as follows: {\it how do our results  apply to complex models, such as deep-neural networks (DNNs)?}
PAC active learning strategies have proven effective in learning DNNs. The work of Sener \etal \cite{sener2018active} selects the most representative points from a sample of the training distribution to learn the DNN. Papernot \etal \cite{papernot2017practical} employ substitute model training - a procedure where a small training subset is strategically augmented and used to train a shadow model that resembles the model being attacked. Note that the prior approaches rely on some additional information, such as a subset of the training data. 

Active learning algorithms considered in this paper work in an
iterative fashion. Let $\mathcal{H}$ be the entire hypothesis
class. At time time $t \geq 0$ let the set of possible hypothesis be
$\mathcal{H}_t \subseteq \mathcal{H}$. Usually an active-learning
algorithm issues a query at time $t$ and updates the possible set of
hypothesis to $\mathcal{H}_{t+1}$, which is a subset of
$\mathcal{H}_t$. Once the size of $\mathcal{H}_t$ is ``small'' the
algorithm stops. Analyzing the effect of a query on possible set of
hypothesis is very complicated in the context of complex models, such
as DNNs. We believe this is a very important and interesting direction
for future work. 

%However, our discussion on defenses is applicable to
%complex models because it essentially treats the classifier as a
%black-box (i.e., we do not use the internal structure of the
%classifier in our analysis).

\subsection{Model Transferability}
%In this section, we discuss strategies that can potentially be used to make the process of extraction more difficult.
%Another orthogonal problem is when the hypothesis space, and nature of the optimal hypothesis is unknown all together. Algorithms are yet to be designed for such tasks, and it remains an open problem.
Most work in active learning has assumed that the correct hypothesis space for the task is already known \ie if the model being learned is for logistic regression, or is a neural network and so on. 
In such situations, observe that the labeled data being used is biased, in that it is implicitly tied to the underlying hypothesis. Thus, it can become problematic if one wishes to re-use the labeled data chosen to learn {\em another, different} hypothesis space. 
This leads us to {\em model transferability}\footnote{A special case of agnostic active learning \cite{balcan2009agnostic}.}, a less studied form of defense where the oracle responds to any query with the prediction output from an entirely different hypothesis class. For example, imagine if a learner tries to learn a halfspace, but the teacher performs prediction using a boolean decision tree. Initial work in this space includes that of Shi \etal \cite{shi2017steal}, where an adversary can steal a linear separator by learning input-output relations using a deep neural network. However, the performance of query synthesis active learning  in such ecosystems is unclear.  

\subsection{Limitations}
\label{sec:limitations}

%%if the hypothesis class is that of homogeneous (i.e. through the origin) linear separators, and the data is distributed uniformly over the unit sphere in $\R^d$, and the labels correspond perfectly to one of the hypotheses (i.e. the separable case) then at most $O(d.\log d/\varepsilon)$ labels are needed to learn a classifier with error less than $\varepsilon$. However, if the hypothesis class is expanded to include non-homogeneous linear separators, the authors notes that there are some target hypotheses for which active learning does not help much. 

We stress that these limitations are not a function of our specific approach, and stem from the theory of active learning. Specifically: (1) As noted by Dasgupta \cite{dasgupta2005coarse}, the label complexity of PAC active learning depends heavily on the specific target hypothesis, and can range from $O(\log \frac{1}{\varepsilon})$ to $\Omega (\frac{1}{\varepsilon})$. Similar results have been obtained by others \cite{hegedHus1995generalized,naghshvar2012noisy}. This suggests that for some hypotheses classes, the query complexity of active learning algorithms is as high as that in the passive setting. (2) Some query synthesis  algorithms  assume that there is some labeled data to bootstrap the system. However, this may not always be true, and randomly generating these labeled points {\em may} adversely impact the performance of the algorithm. (3) For our particular implementation, the algorithms proposed rely on the {\em geometric error} between the optimal and learned halfspaces. Oftentimes, however, there is no direct correlation between this geometric error and the generalization error used to measure the model's {\em goodness}.

\section{Related Work}
\label{sec:related}

Machine learning algorithms and systems are optimized for performance. Little attention is paid to the security and privacy risks of these systems and algorithms. Our work is motivated by the following attacks against machine learning.

\vspace{1mm}
\noindent{\em 1. Causative Attacks:} These attacks are primarily geared at {\em poisoning} the training data used for learning, such that the classifier produced performs erroneously during test time. These include: (a) mislabeling the training data, (b) changing rewards in the case of reinforcement learning, or (c) modifying the sampling mechanism (to add some bias) such that it does not reflect the true underlying distribution in the case of unsupervised learning \cite{papernot2016towards}. The work of Papernot \etal \cite{papernot2016limitations} modify input features resulting in misclassification by Deep Neural Networks. 

\vspace{1mm}
\noindent{\em 2. Evasion Attacks:} Once the algorithm has trained successfully, these forms of attacks provide {\em tailored} inputs such that the output is erroneous. These {\em noisy inputs} often preserves the semantics of the original inputs, are human imperceptible, or are physically realizable. The well studied area of {\em adversarial examples} is an instantiation of such an attack. Moreover, evasion attacks can also be even black-box \ie the attacker needn't know the model. This is because an adversarial example optimized for one model is highly likely to be effective for other models. This concept, known as {\em transferability}, was introduced by Carlini \etal \cite{carlini2017towards}. Notable works in this space include \cite{bhagoji2018black,elsayed2018adversarial,kurakin2018ensemble,kurakin2016adversarial,tramer2017space,papernot2017practical,warde201611,evtimov2017robust}

\vspace{1mm}
\noindent{\em 3. Exploratory Attacks:} These forms of attacks are the primary focus of this work, and  are geared at learning intrinsics about the algorithm used for training. These intrinsics can include learning model parameters, hyperparameters, or training data. Typically, these forms of attacks fall in two categories - {\em model inversion}, or {\em model extraction}. In the first class, Fredrikson \etal \cite{fredrikson2014privacy} show that an attacker can learn sensitive information about the dataset used to train a model, given access to side-channel information about the dataset. In the second class, the work of Tramer \etal \cite{TZJRR16} provides attacks to learn parameters of a model hosted on the cloud, through a query interface. Termed {\em membership inference}, Shokri \etal \cite{shokri2017membership} learn the training data used for machine learning by training their own inference models. Wang \etal \cite{wang2018stealing} propose attacks to learn a model's hyperparameters.

% \subsection{Today's State-of-the-art} 

% Several works in the machine learning community tackle the problem of learning in the presence of noise, specifically for logistic regression \cite{chaudhuri2009privacy,cryptoeprint:2017:707}, and deep learning \cite{shokri2015privacy}. The work of Jha \etal \cite{jha2008towards} present privacy preserving techniques for calculating edit distance and similarity scores. While one school of thought for providing security relies on operating over encrypted inputs (homomorphic encryption) \cite{papadimitriou2016big}, others rely on techniques from secure multi-party computation \cite{mohassel2017secureml} or differential privacy \cite{abadi2016deep}. A common feature among all these works is the trade-off between the ease of operation (efficiency in learning) and security. 

% In the next subsection, we discuss popular active learning algorithms in the context of learning prediction algorithms by repeatedly querying an oracle with knowledge of the prediction function. 

\section{Conclusions}

% [1] Exploring more complicated models (DNNs, Random Forests, ....)
% [2] Strengthening some of the theoretical results (e.g., weaker 
% conditions under which -ve results hold)
% [3] More exploration of defenses

In this paper, we formalize model extraction in the context of
Machine-Learning-as-a-Service (MLaaS) servers that return only
prediction values (\ie, oracle access setting), and we study its
relation with \emph{query synthesis} active learning
(Observation 1). Thus, we are able to implement
efficient attacks to the class
of halfspace models used for binary classification
(Section~\ref{sec:implementation}). While our experiments focus on the
class of halfspace models, we believe that extraction via active
learning can be extended to multiclass and non-linear models such as
deep neural networks, random forests etc.  We also begin exploring
possible defense approaches (subsection \ref{sec:defense}). To the
best of our knowledge, this is the first work to formalize security in
the context of MLaaS systems. We believe this is a fundamental first
step in designing more secure MLaaS systems. Finally, we suggest that
data-dependent randomization (\eg, model randomization as in
\cite{CIKM-14}) is the most promising direction to follow in order to
design effective defenses.

%We also wish to design query synthesis algorithms capable of performing extraction against data-dependent randomization defenses, and design more efficient defense strategies against active learning-based extraction.
%Finally, we wish to strengthen the theoretical results in our work, by relaxing the conditions for which our results hold. 

\section{Acknowledgements}

This material is partially supported by Air Force Grant FA9550-18-1-0166, the National Science Foundation (NSF)
Grants CCF-FMitF-1836978, SaTC-Frontiers-1804648 and CCF-1652140 and ARO grant number W911NF-17-1-0405. Kamalika Chaudhuri and Songbai Yan thank NSF under 1719133 and 1804829 for research support.

\bibliographystyle{plain}
\bibliography{arxiv_main}

\begin{thebibliography}{10}

\bibitem{uci}
\url{https://archive.ics.uci.edu/ml/datasets.html}, 2018.

\bibitem{CIKM-14}
Ibrahim~M. Alabdulmohsin, Xin Gao, and Xiangliang Zhang.
\newblock Adding robustness to support vector machines against adversarial
  reverse engineering.
\newblock In {\em Proceedings of the 23rd {ACM} International Conference on
  Conference on Information and Knowledge Management, {CIKM} 2014, Shanghai,
  China, November 3-7, 2014}, pages 231--240, 2014.

\bibitem{alabdulmohsin2015efficient}
Ibrahim~M Alabdulmohsin, Xin Gao, and Xiangliang Zhang.
\newblock Efficient active learning of halfspaces via query synthesis.
\newblock In {\em AAAI}, pages 2483--2489, 2015.

\bibitem{angluin1987learning}
Dana Angluin.
\newblock Learning regular sets from queries and counterexamples.
\newblock {\em Information and computation}, 75(2):87--106, 1987.

\bibitem{AMSVVF15}
Giuseppe Ateniese, Luigi~V. Mancini, Angelo Spognardi, Antonio Villani,
  Domenico Vitali, and Giovanni Felici.
\newblock Hacking smart machines with smarter ones: How to extract meaningful
  data from machine learning classifiers.
\newblock {\em {IJSN}}, 10(3):137--150, 2015.

\bibitem{atlas1990training}
Les~E Atlas, David~A Cohn, and Richard~E Ladner.
\newblock Training connectionist networks with queries and selective sampling.
\newblock In {\em Advances in neural information processing systems}, pages
  566--573, 1990.

\bibitem{awasthi2013learning}
Pranjal Awasthi, Vitaly Feldman, and Varun Kanade.
\newblock Learning using local membership queries.
\newblock In {\em Conference on Learning Theory}, pages 398--431, 2013.

\bibitem{balcan2009agnostic}
Maria-Florina Balcan, Alina Beygelzimer, and John Langford.
\newblock Agnostic active learning.
\newblock {\em Journal of Computer and System Sciences}, 75(1):78--89, 2009.

\bibitem{BalcanBZ07}
Maria{-}Florina Balcan, Andrei~Z. Broder, and Tong Zhang.
\newblock Margin based active learning.
\newblock In {\em Learning Theory, 20th Annual Conference on Learning Theory,
  {COLT} 2007, San Diego, CA, USA, June 13-15, 2007, Proceedings}, pages
  35--50, 2007.

\bibitem{BalcanL13}
Maria{-}Florina Balcan and Philip~M. Long.
\newblock Active and passive learning of linear separators under log-concave
  distributions.
\newblock In {\em {COLT} 2013 - The 26th Annual Conference on Learning Theory,
  June 12-14, 2013, Princeton University, NJ, {USA}}, pages 288--316, 2013.

\bibitem{iwal:2010}
Alina Beygelzimer, Daniel Hsu, John Langford, and Tong Zhang.
\newblock Agnostic active learning without constraints.
\newblock In {\em 23rd International Conference on Neural Information
  Processing Systems (NIPS)}, 2010.

\bibitem{bhagoji2018black}
Arjun~Nitin Bhagoji, Warren He, Bo~Li, and Dawn Song.
\newblock Black-box attacks on deep neural networks via gradient estimation.
\newblock 2018.

\bibitem{LASVM}
Antoine Bordes, Seyda Ertekin, Jason Weston, and Leon Bottou.
\newblock Fast kernel classifiers with online and active learning.
\newblock {\em Journal of Machine Learning Research (JMLR)}, September 2005.

\bibitem{brendel2017decision}
Wieland Brendel, Jonas Rauber, and Matthias Bethge.
\newblock Decision-based adversarial attacks: Reliable attacks against
  black-box machine learning models.
\newblock {\em arXiv preprint arXiv:1712.04248}, 2017.

\bibitem{brinker2003incorporating}
Klaus Brinker.
\newblock Incorporating diversity in active learning with support vector
  machines.
\newblock In {\em Proceedings of the 20th International Conference on Machine
  Learning (ICML-03)}, pages 59--66, 2003.

\bibitem{carlini2017towards}
Nicholas Carlini and David Wagner.
\newblock Towards evaluating the robustness of neural networks.
\newblock In {\em Security and Privacy (SP), 2017 IEEE Symposium on}, pages
  39--57. IEEE, 2017.

\bibitem{chen2017near}
Lin Chen, Seyed~Hamed Hassani, and Amin Karbasi.
\newblock Near-optimal active learning of halfspaces via query synthesis in the
  noisy setting.
\newblock In {\em AAAI}, pages 1798--1804, 2017.

\bibitem{cohn1994improving}
David Cohn, Les Atlas, and Richard Ladner.
\newblock Improving generalization with active learning.
\newblock {\em Machine learning}, 15(2):201--221, 1994.

\bibitem{HS14}
Hsu D. and Sabato S.
\newblock Heavy-tailed regression with a generalized median-of-means.
\newblock In {\em International Conference on Machine Learning (ICML)}, 2014.

\bibitem{dagan1995committee}
Ido Dagan and Sean~P Engelson.
\newblock Committee-based sampling for training probabilistic classifiers.
\newblock In {\em Proceedings of the Twelfth International Conference on
  Machine Learning}, pages 150--157. The Morgan Kaufmann series in machine
  learning,(San Francisco, CA, USA), 1995.

\bibitem{DHM07}
S.~Dasgupta, D.~Hsu, and C.~Monteleoni.
\newblock A general agnostic active learning algorithm.
\newblock In {\em NIPS}, 2007.

\bibitem{dasgupta2005coarse}
Sanjoy Dasgupta.
\newblock Coarse sample complexity bounds for active learning.
\newblock In {\em Advances in Neural Information Processing Systems 18 [Neural
  Information Processing Systems, {NIPS} 2005, December 5-8, 2005, Vancouver,
  British Columbia, Canada]}, pages 235--242, 2005.

\bibitem{dasgupta2011two}
Sanjoy Dasgupta.
\newblock Two faces of active learning.
\newblock {\em Theoretical computer science}, 412(19):1767--1781, 2011.

\bibitem{dasgupta2008general}
Sanjoy Dasgupta, Daniel~J Hsu, and Claire Monteleoni.
\newblock A general agnostic active learning algorithm.
\newblock In {\em Advances in neural information processing systems}, pages
  353--360, 2008.

\bibitem{drachsler2017synthesis}
Dana Drachsler-Cohen, Sharon Shoham, and Eran Yahav.
\newblock Synthesis with abstract examples.
\newblock In {\em International Conference on Computer Aided Verification},
  pages 254--278. Springer, 2017.

\bibitem{elsayed2018adversarial}
Gamaleldin~F Elsayed, Shreya Shankar, Brian Cheung, Nicolas Papernot, Alex
  Kurakin, Ian Goodfellow, and Jascha Sohl-Dickstein.
\newblock Adversarial examples that fool both human and computer vision.
\newblock {\em arXiv preprint arXiv:1802.08195}, 2018.

\bibitem{CKNS15}
Chaudhuri~K. et~al.
\newblock Convergence rates of active learning for maximum likelihood
  estimation.
\newblock In {\em Advances in Neural Information Processing Systems}, 2015.

\bibitem{evtimov2017robust}
Ivan Evtimov, Kevin Eykholt, Earlence Fernandes, Tadayoshi Kohno, Bo~Li, Atul
  Prakash, Amir Rahmati, and Dawn Song.
\newblock Robust physical-world attacks on deep learning models.
\newblock {\em arXiv preprint arXiv:1707.08945}, 1, 2017.

\bibitem{fredrikson2014privacy}
Matthew Fredrikson, Eric Lantz, Somesh Jha, Simon Lin, David Page, and Thomas
  Ristenpart.
\newblock Privacy in pharmacogenetics: An end-to-end case study of personalized
  warfarin dosing.
\newblock In {\em USENIX Security Symposium}, pages 17--32, 2014.

\bibitem{freund1997selective}
Yoav Freund, H~Sebastian Seung, Eli Shamir, and Naftali Tishby.
\newblock Selective sampling using the query by committee algorithm.
\newblock {\em Machine learning}, 28(2):133--168, 1997.

\bibitem{fujii1998selective}
Atsushi Fujii, Takenobu Tokunaga, Kentaro Inui, and Hozumi Tanaka.
\newblock Selective sampling for example-based word sense disambiguation.
\newblock {\em Computational Linguistics}, 24(4):573--597, 1998.

\bibitem{gulwani2012synthesis}
Sumit Gulwani.
\newblock Synthesis from examples: Interaction models and algorithms.
\newblock In {\em Symbolic and Numeric Algorithms for Scientific Computing
  (SYNASC), 2012 14th International Symposium on}, pages 8--14. IEEE, 2012.

\bibitem{H07}
S.~Hanneke.
\newblock A bound on the label complexity of agnostic active learning.
\newblock In {\em ICML}, 2007.

\bibitem{key-1}
Steve Hanneke.
\newblock Theory of disagreement-based active learning.
\newblock {\em Foundations and Trends in Machine Learning}, 7(2-3):131--309,
  2014.

\bibitem{hegedHus1995generalized}
Tibor Heged{\H{u}}s.
\newblock Generalized teaching dimensions and the query complexity of learning.
\newblock In {\em Proceedings of the eighth annual conference on Computational
  learning theory}, pages 108--117. ACM, 1995.

\bibitem{HJNRT11}
Ling Huang, Anthony~D. Joseph, Blaine Nelson, Benjamin I.~P. Rubinstein, and
  J.~D. Tygar.
\newblock Adversarial machine learning.
\newblock In {\em Proceedings of the 4th {ACM} Workshop on Security and
  Artificial Intelligence, AISec 2011, Chicago, IL, USA, October 21, 2011},
  pages 43--58, 2011.

\bibitem{Kaariainen06}
Matti K{\"{a}}{\"{a}}ri{\"{a}}inen.
\newblock Active learning in the non-realizable case.
\newblock In {\em Algorithmic Learning Theory, 17th International Conference,
  {ALT} 2006, Barcelona, Spain, October 7-10, 2006, Proceedings}, pages 63--77,
  2006.

\bibitem{KarpK07}
Richard~M. Karp and Robert Kleinberg.
\newblock Noisy binary search and its applications.
\newblock In {\em Proceedings of the Eighteenth Annual {ACM-SIAM} Symposium on
  Discrete Algorithms, {SODA} 2007, New Orleans, Louisiana, USA, January 7-9,
  2007}, pages 881--890, 2007.

\bibitem{king2009automation}
Ross~D King, Jem Rowland, Stephen~G Oliver, Michael Young, Wayne Aubrey, Emma
  Byrne, Maria Liakata, Magdalena Markham, Pinar Pir, Larisa~N Soldatova,
  et~al.
\newblock The automation of science.
\newblock {\em Science}, 324(5923):85--89, 2009.

\bibitem{key-3}
Adam~R. Klivans and Pravesh Kothari.
\newblock Embedding hard learning problems into gaussian space.
\newblock In {\em Approximation, Randomization, and Combinatorial Optimization.
  Algorithms and Techniques, {APPROX/RANDOM} 2014, September 4-6, 2014,
  Barcelona, Spain}, pages 793--809, 2014.

\bibitem{krishnamurthy2002algorithms}
Vikram Krishnamurthy.
\newblock Algorithms for optimal scheduling and management of hidden markov
  model sensors.
\newblock {\em IEEE Transactions on Signal Processing}, 50(6):1382--1397, 2002.

\bibitem{kurakin2018ensemble}
Alex Kurakin, Dan Boneh, Florian Tram{\`e}r, Ian Goodfellow, Nicolas Papernot,
  and Patrick McDaniel.
\newblock Ensemble adversarial training: Attacks and defenses.
\newblock 2018.

\bibitem{kurakin2016adversarial}
Alexey Kurakin, Ian Goodfellow, and Samy Bengio.
\newblock Adversarial machine learning at scale.
\newblock {\em arXiv preprint arXiv:1611.01236}, 2016.

\bibitem{Kushilevitz-Mansour}
Eyal Kushilevitz and Yishay Mansour.
\newblock Learning decision trees using the fourier spectrum.
\newblock {\em {SIAM} J. Comput.}, 22(6):1331--1348, 1993.

\bibitem{lindenbaum1999selective}
Michael Lindenbaum, Shaul Markovitch, and Dmitry Rusakov.
\newblock Selective sampling for nearest neighbor classifiers.
\newblock In {\em AAAI/IAAI}, pages 366--371. Citeseer, 1999.

\bibitem{liu2004selective}
Huan Liu, Hiroshi Motoda, and Lei Yu.
\newblock A selective sampling approach to active feature selection.
\newblock {\em Artificial Intelligence}, 159(1-2):49--74, 2004.

\bibitem{LM05}
Daniel Lowd and Christopher Meek.
\newblock Adversarial learning.
\newblock In {\em Proceedings of the Eleventh {ACM} {SIGKDD} International
  Conference on Knowledge Discovery and Data Mining, Chicago, Illinois, USA,
  August 21-24, 2005}, pages 641--647, 2005.

\bibitem{McCallumN98}
Andrew McCallum and Kamal Nigam.
\newblock Employing {EM} and pool-based active learning for text
  classification.
\newblock In {\em Proceedings of the Fifteenth International Conference on
  Machine Learning, Madison, Wisconsin, USA, July 24-27, 1998}, pages 350--358,
  1998.

\bibitem{minh2006mercer}
Ha~Quang Minh, Partha Niyogi, and Yuan Yao.
\newblock Mercer’s theorem, feature maps, and smoothing.
\newblock In {\em International Conference on Computational Learning Theory},
  pages 154--168. Springer, 2006.

\bibitem{mitchell1982generalization}
Tom~M Mitchell.
\newblock Generalization as search.
\newblock {\em Artificial intelligence}, 18(2):203--226, 1982.

\bibitem{mitchell1978version}
Tom~Michael Mitchell.
\newblock Version spaces: an approach to concept learning.
\newblock Technical report, STANFORD UNIV CALIF DEPT OF COMPUTER SCIENCE, 1978.

\bibitem{naghshvar2012noisy}
Mohammad Naghshvar, Tara Javidi, and Kamalika Chaudhuri.
\newblock Noisy bayesian active learning.
\newblock In {\em Communication, Control, and Computing (Allerton), 2012 50th
  Annual Allerton Conference on}, pages 1626--1633. IEEE, 2012.

\bibitem{nowak2009noisy}
Robert Nowak.
\newblock Noisy generalized binary search.
\newblock In {\em Advances in neural information processing systems}, pages
  1366--1374, 2009.

\bibitem{Nowak11}
Robert~D. Nowak.
\newblock The geometry of generalized binary search.
\newblock {\em {IEEE} Trans. Information Theory}, 57(12):7893--7906, 2011.

\bibitem{papernot2017practical}
Nicolas Papernot, Patrick McDaniel, Ian Goodfellow, Somesh Jha, Z~Berkay Celik,
  and Ananthram Swami.
\newblock Practical black-box attacks against machine learning.
\newblock In {\em Proceedings of the 2017 ACM on Asia Conference on Computer
  and Communications Security}, pages 506--519. ACM, 2017.

\bibitem{papernot2016limitations}
Nicolas Papernot, Patrick McDaniel, Somesh Jha, Matt Fredrikson, Z~Berkay
  Celik, and Ananthram Swami.
\newblock The limitations of deep learning in adversarial settings.
\newblock In {\em Security and Privacy (EuroS\&P), 2016 IEEE European Symposium
  on}, pages 372--387. IEEE, 2016.

\bibitem{papernot2016towards}
Nicolas Papernot, Patrick McDaniel, Arunesh Sinha, and Michael Wellman.
\newblock Towards the science of security and privacy in machine learning.
\newblock {\em arXiv preprint arXiv:1611.03814}, 2016.

\bibitem{peleg2018abstraction}
Hila Peleg, Shachar Itzhaky, and Sharon Shoham.
\newblock Abstraction-based interaction model for synthesis.
\newblock In {\em International Conference on Verification, Model Checking, and
  Abstract Interpretation}, pages 382--405. Springer, 2018.

\bibitem{SM14}
Sabato S. and Munos R.
\newblock Active regression by stratification.
\newblock In {\em Advances in Neural Information Processing Systems (NIPS)},
  2014.

\bibitem{sener2018active}
Ozan Sener and Silvio Savarese.
\newblock Active learning for convolutional neural networks: A core-set
  approach.
\newblock 2018.

\bibitem{settlesactive}
B~Settles.
\newblock Active learning literature survey univ. wisconsin-madison, madison,
  wi, 2009.
\newblock Technical report, CS Tech. Rep. 1648.

\bibitem{shi2017steal}
Yi~Shi, Yalin Sagduyu, and Alexander Grushin.
\newblock How to steal a machine learning classifier with deep learning.
\newblock In {\em Technologies for Homeland Security (HST), 2017 IEEE
  International Symposium on}, pages 1--5. IEEE, 2017.

\bibitem{shokri2017membership}
Reza Shokri, Marco Stronati, Congzheng Song, and Vitaly Shmatikov.
\newblock Membership inference attacks against machine learning models.
\newblock In {\em Security and Privacy (SP), 2017 IEEE Symposium on}, pages
  3--18. IEEE, 2017.

\bibitem{SL14}
Nedim Srndic and Pavel Laskov.
\newblock Practical evasion of a learning-based classifier: {A} case study.
\newblock In {\em 2014 {IEEE} Symposium on Security and Privacy, {SP} 2014,
  Berkeley, CA, USA, May 18-21, 2014}, pages 197--211, 2014.

\bibitem{tong2001support}
Simon Tong and Daphne Koller.
\newblock Support vector machine active learning with applications to text
  classification.
\newblock {\em Journal of machine learning research}, 2(Nov):45--66, 2001.

\bibitem{tramer2017space}
Florian Tram{\`e}r, Nicolas Papernot, Ian Goodfellow, Dan Boneh, and Patrick
  McDaniel.
\newblock The space of transferable adversarial examples.
\newblock {\em arXiv preprint arXiv:1704.03453}, 2017.

\bibitem{TZJRR16}
Florian Tram{\`{e}}r, Fan Zhang, Ari Juels, Michael~K. Reiter, and Thomas
  Ristenpart.
\newblock Stealing machine learning models via prediction apis.
\newblock In {\em 25th {USENIX} Security Symposium, {USENIX} Security 16,
  Austin, TX, USA, August 10-12, 2016.}, pages 601--618, 2016.

\bibitem{valiant1984theory}
Leslie~G Valiant.
\newblock A theory of the learnable.
\newblock {\em Communications of the ACM}, 27(11):1134--1142, 1984.

\bibitem{wang2018stealing}
Binghui Wang and Neil~Zhenqiang Gong.
\newblock Stealing hyperparameters in machine learning.
\newblock {\em arXiv preprint arXiv:1802.05351}, 2018.

\bibitem{wang2017interactive}
Chenglong Wang, Alvin Cheung, and Rastislav Bodik.
\newblock Interactive query synthesis from input-output examples.
\newblock In {\em Proceedings of the 2017 ACM International Conference on
  Management of Data}, pages 1631--1634. ACM, 2017.

\bibitem{wang2015active}
Liantao Wang, Xuelei Hu, Bo~Yuan, and Jianfeng Lu.
\newblock Active learning via query synthesis and nearest neighbour search.
\newblock {\em Neurocomputing}, 147:426--434, 2015.

\bibitem{warde201611}
David Warde-Farley and Ian Goodfellow.
\newblock 11 adversarial perturbations of deep neural networks.
\newblock {\em Perturbations, Optimization, and Statistics}, page 311, 2016.

\bibitem{yan2016active}
Songbai Yan, Kamalika Chaudhuri, and Tara Javidi.
\newblock Active learning from imperfect labelers.
\newblock In {\em Advances in Neural Information Processing Systems}, pages
  2128--2136, 2016.

\bibitem{key-5}
Songbai Yan and Chicheng Zhang.
\newblock Revisiting perceptron: Efficient and label-optimal learning of
  halfspaces.
\newblock In {\em Advances in Neural Information Processing Systems 30: Annual
  Conference on Neural Information Processing Systems 2017, 4-9 December 2017,
  Long Beach, CA, {USA}}, pages 1056--1066, 2017.

\bibitem{yu2005svm}
Hwanjo Yu.
\newblock Svm selective sampling for ranking with application to data
  retrieval.
\newblock In {\em Proceedings of the eleventh ACM SIGKDD international
  conference on Knowledge discovery in data mining}, pages 354--363. ACM, 2005.

\bibitem{ZC14}
Chicheng Zhang and Kamalika Chaudhuri.
\newblock Beyond disagreement-based agnostic active learning.
\newblock In {\em Advances in Neural Information Processing Systems}, pages
  442--450, 2014.

\end{thebibliography}

\appendix
\section{Appendix}
\subsection{Proofs}

\subsubsection{Proof of Proposition~\ref{prop}}
\label{app:prop}

\begin{proof}
Let $\tilde{\A}$ be the adversary that does the following:

for $i=1,\dots,q(\varepsilon,\delta)$
\begin{enumerate}
        \item $\tilde{A}$ uses the query strategy  of $\L$ to generate the instance $x_i$;
        \item $\tilde{A}$ queries $x_i$ to $\Ser_D(f^*)$ for $r$ times and defines $y_i$ the most frequent labels among the $r$ answers (we assume $r$ is an even integer).
\end{enumerate}
At the end, $\tilde{A}$ (as the learner $\L$) learns $\hat{f}$ using the points $\{(x_i,y_i)\}_{i=1,\cdots,q(\varepsilon, \delta)}$.
%In the following, we estimate the probability that $\tilde{A}$ is implementing $\varepsilon$-extraction with $q$ queries. 
Let $q=r \cdot q(\varepsilon,\delta)$, then it holds by the union bound that
\begin{align*}
        &\Pr[\Exp^\varepsilon_\F(\Ser_D(f^*),\tilde{A},q)=1]\geq\\ &1-\delta-(1-\Pr[\cap_{i=1}^{q(\varepsilon,\delta)}\{y_i=f^*(x_i)\}])\,.
\end{align*}

Define $X_i^j$ as the binary random variable that is 1 if and only if the answer to the $j$-th query of $x_i$ is correct and $X_i=\sum_{j=1}^{r}X_i^j$, then
\begin{align*}
        \Pr[\cap_{i=1}^{q(\varepsilon,\delta)}\{y_i=f^*(x_i)\}]&\geq \Pr[\cap_{i=1}^{q(\varepsilon,\delta)}\{X_i>r/2\}]\\
        &\geq 1- \sum_{i=1}^{q(\varepsilon,\delta)}\Pr[\{X_i\leq r/2\}]
\end{align*}
where the last step follows from the union bound. Now, observe that $\E[X_i]=r(1-\rho_D(f^*,x_i))> r/2$ and the Chernoff bound can be applied on each term in the right-hand side. In particular, we have that
%$\Pr[\{X_i\leq r/2\}]\leq e^{-r\frac{(\rho_i-\frac{1}{2})^2}{2\rho_i}}\leq e^{-r\frac{(\rho-\frac{1}{2})^2}{2}} $ and it follows that
$\Pr[\{X_i\leq r/2\}]\leq e^{-r\frac{\left(\rho_D(f^*)-\frac{1}{2}\right)^2}{2}} $ and it follows that
$$
\Pr[\Exp^\varepsilon_\F(\Ser_D(f^*),\tilde{A},q)=1]\geq1-\delta-q(\varepsilon,\delta)\,e^{-r\frac{(\rho_D(f^*)-\frac{1}{2})^2}{2}}\,.
$$
By setting $r=\frac{8}{(1-2\rho_D(f^*))^2}\ln\frac{q(\varepsilon,\delta)}{\delta}$ we have $\Pr[\Exp^\varepsilon_\F(\Ser_D(f^*),\tilde{A},q)=1]\geq1-2\delta$.
That is, the adversary $\tilde{A}$ implements an $\varepsilon$-extraction with {\footnotesize Confidence Score}
$1-2\delta$ and complexity $q=\frac{8}{(1-2\rho_D(f^*))^2}q(\varepsilon,\delta)\ln\frac{q(\varepsilon,\delta)}{\delta}$.\\
\end{proof}

\subsubsection{Proof of Algorithm~\ref{alg:passive}}
\label{app:alg}

Here, we discuss the analysis and proofs associated with Algorithm~\ref{alg:passive}.

%\subsubsection{Settings}

Assume unit vector $\mu\in\mathbb{R}^{d}$ is the ground truth. For each
query $x\in\mathbb{R}^{d}$, a vector $w$ is drawn from $N(\mu,\sigma^{2}I)$,
and the label $y=\text{sign}(\left\langle w,x\right\rangle )$ is
returned. The goal of the learner (attacker) is to return $\hat{w}$
such that $\left\Vert \mu-\hat{w}\right\Vert $ is small.

% \subsubsection{Algorithm}

% The attacking algorithm is as shown as Algorithm~\ref{alg:passive}.
% It takes an estimated upper bound $\hat{\sigma}$ of $\sigma$ as
% input. The algorithm first draws $m=\tilde{O}(\frac{d}{\varepsilon^{2}}\max(1,d\hat{\sigma}^{2}))$
% labeled examples $\{(x_{i},y_{i})\}_{i=1}^{m}$ from the unit sphere
% uniformly at random, and then computes the average $v=\sum_{i=1}^{m}y_{i}x_{i}$.
% $\ensuremath{\hat{w}=\frac{v}{\left\Vert v\right\Vert }}$, the direction
% of $v$ is the algorithm's estimate of the classifier $u$, and the
% length of $v$ is used as an indicator of whether the algorithm succeeds:
% if this estimated upper bound is correct (i.e. $\sigma\leq\hat{\sigma}$),
% then with high probability, $\left\Vert \hat{w}-u\right\Vert \leq\varepsilon$;
% otherwise it outputs \textsc{fail}, indicating the variance bound $\hat{\sigma}$
% is incorrect.

% \begin{algorithm}
% \begin{algorithmic}[1]
% \State{Input: variance upper bound $\hat{\sigma}\geq \frac{1}{\sqrt{d}}$, target error $\varepsilon$}
% \State{$m \gets \frac{(15\pi)^2}{\varepsilon^2}d\max(1,d\hat{\sigma}^2)\log\frac{2d}{\delta}$, $l \gets \frac{1}{12d\hat{\sigma}}$}
% \State{Draw $x_1,x_2,\dots,x_m\in\mathbb{S}^{d-1}$ uniformly at random, and query their labels $y_1,y_2,\dots,y_m$}
% \State{$v \gets \sum_{i=1}^m y_i x_i$}
% \If{$\left\Vert v\right\Vert \geq l$}
% \State{Return $\hat{w}=\frac{v}{\left\Vert v\right\Vert}$}
% \Else
% \State{Return \textsc{fail}}
% \EndIf
% \end{algorithmic}

% \compactcaption{\footnotesize \label{alg:passive}Attacking Algorithm}
% \end{algorithm}

%\subsubsection{Analysis}

We have following theoretical guarantees for Algorithm~\ref{alg:passive}.
\begin{prop}
\label{prop:complexity}If $\sigma\leq\hat{\sigma}$, then $\left\Vert \hat{w}-\mu\right\Vert \leq\varepsilon$
with probability at least $1-\delta$.
\end{prop}

\begin{prop}
\label{prop:length-lb}If $\sigma\leq\hat{\sigma}$ and $\hat{\sigma}\geq\frac{1}{\sqrt{d}}$,
then $\left\Vert v\right\Vert \geq\frac{1}{12d\hat{\sigma}}$ with
probability at least $1-\delta$.
\end{prop}

\begin{prop}
\label{prop:length-ub}If $\sigma\geq20\hat{\sigma}\geq\frac{1}{\sqrt{d}}$,
then $\left\Vert v\right\Vert \leq\frac{1}{\sqrt{148}d\hat{\sigma}}$
with probability at least $1-\delta$.
\end{prop}

Propositions~\ref{prop:complexity} and \ref{prop:length-lb} guarantees
that if the estimated upper bound is correct ($\sigma\leq\hat{\sigma}$),
then the algorithm outputs an accurate estimation of $\mu$; Proposition~\ref{prop:length-ub}
guarantees the algorithm declares failure if the estimated upper bound
is too small ($\hat{\sigma}\leq\frac{1}{20}\sigma$).

Intuitively, the average of $y_{i}x_{i}$ ($i=1,\dots,m$) points
to a direction similar to $\mu$ because of the symmetry of distributions
of both $x$ and noise: the projection of $yx$ onto all directions
perpendicular to $\mu$ is distributed symmetrically around 0 and thus
has mean 0. The projection of $yx$ onto $\mu$ has non-negative mean
since the label $y$ is correct (i.e., $yv^{\top}x\geq0$) with probability
at least $\frac{1}{2}$, and the projection is larger if the noise
of $y$ is smaller. Consequently, the scale of the average can be
used as an indicator of the noise level. The Propositions can be formally
proved by applying concentration inequalities on each direction.

%\subsubsection{Proofs}

\begin{notation*}
Denote by $\mathbb{S}^{d-1}$ the unit sphere $\{x\in\R^{d}:x^{\top}x=1\}$.
For any vector $X\in\R^{d}$, denote by $X^{(i)}$ the $i$-th coordinate
of $X$. Define $Z_{i}=Y_{i}X_{i}$ for $i=1,2,\dots,m$.
\end{notation*}
We need following facts.
\begin{fact}
\label{fact:gaussian}$\Pr(w^{\top}x\geq0)=\Pr_{\xi\sim N(0,1)}(\xi\geq-\frac{w^{\star\top}x}{\sigma\left\Vert x\right\Vert })$.
Moreover, for $z\geq0$, $\frac{1}{2}-\frac{z}{\sqrt{2\pi}\sigma}\leq\Pr_{\xi\sim N(0,1)}(\xi\geq\frac{z}{\sigma})\leq\max(\frac{1}{6},\frac{1}{2}-\frac{z}{3\sigma})$.
\end{fact}

\begin{fact}
\label{fact:beta}Let $B(x,y)=\int_{0}^{1}(1-t)^{x-1}t^{y-1}dt$ be
the Beta function. Then $\frac{2}{\sqrt{d-1}}\leq B(\frac{1}{2},\frac{d}{2})\leq\frac{\pi}{\sqrt{d}}$.
\end{fact}

\begin{fact}
\label{fact:exp}If $d\geq2$, then $(1-\frac{1}{d})^{\frac{d}{2}}\geq\frac{1}{2}$.
\end{fact}

\begin{fact}
(Bernstein inequality) If i.i.d. random variables $X_{1},\dots,X_{m}$
satisfy $|X_{i}|\leq b$, $\E[X_{i}]=\mu$, and $\E[X_{i}^{2}]\leq r^{2}$,
then with probability at least $1-\delta$, $|\frac{1}{m}\sum_{i=1}^{m}X_{i}-\mu|\leq\sqrt{\frac{2r^{2}}{m}\log\frac{2}{\delta}}+\frac{2b}{3m}\log\frac{2}{\delta}$.
\end{fact}

\begin{fact}
\label{fact:density}Suppose $(x_{1},\dots,x_{d})$ is drawn from
the uniform distribution over the unit sphere, then $x_{1}$ has a
density function of $p(z)=\frac{(1-z^{2})^{\frac{d-3}{2}}}{B(\frac{d-1}{2},\frac{1}{2})}$.
\end{fact}

Without loss of generality, assume $\mu=(1,0,0,\dots,0)$. 

Following three lemmas give concentration of $v$.
\begin{lem}
\label{lem:ub-k}For any $k=2,3,\dots,d$, with probability at least
$1-\delta$, $|v^{(k)}|\leq2\sqrt{\frac{2}{md}\log\frac{2}{\delta}}$.
\end{lem}

\begin{proof}
For $k=2,3,\dots,d$, $Z_{1}^{(k)},Z_{2}^{(k)},\dots,Z_{m}^{(k)}$
are i.i.d. random variables bounded by 1. By symmetry, $\E[Z_{1}^{(k)}]=0$.
$\E[(Z_{1}^{(k)})^{2}]=\E[(X^{(i)})^{2}]=2\int_{0}^{1}z^{2}\frac{(1-z^{2})^{\frac{d-3}{2}}}{B(\frac{d-1}{2},\frac{1}{2})}dz=\frac{1}{B(\frac{d-1}{2},\frac{1}{2})}\int_{0}^{1}t^{\frac{1}{2}}(1-t)^{\frac{d-3}{2}}dt=\frac{B(\frac{d-1}{2},\frac{3}{2})}{B(\frac{d-1}{2},\frac{1}{2})}=\frac{1}{d}$.
By the Bernstein inequality, with probability at least $1-\delta$,
$|\frac{1}{m}\sum Z_{i}^{(k)}|\leq\sqrt{\frac{2}{md}\log\frac{2}{\delta}}+\frac{2}{3m}\log\frac{2}{\delta}\leq2\sqrt{\frac{2}{md}\log\frac{2}{\delta}}$.
\end{proof}
\begin{lem}
\label{lem:lb-1}With probability at least $1-\delta$, $v^{(1)}\geq\frac{1}{3\pi\sqrt{d}}\min(1,\frac{1}{\sigma\sqrt{d}})-\sqrt{\frac{1}{2m}\log\frac{1}{\delta}}$.
\end{lem}

\begin{proof}
$Z_{1}^{(1)},Z_{2}^{(1)},\dots,Z_{m}^{(1)}$ are i.i.d. random variables
bounded by 1. Their mean can be lower-bounded as follows. 

For for any $0\leq a\leq1$, due to the noise setting, $\E[Y\mid X^{(1)}=a]\geq0$
and $\E[Y\mid X^{(1)}=-a]\leq0$, so $\E[YX^{(1)}\mid X^{(1)}=a]+\E[YX^{(1)}\mid X^{(1)}=-a]\geq0$.
Consequently we have

\begin{align*}
\E[YX^{(1)}] & \geq\E[YX^{(1)}\One[|X^{(1)}|\geq\frac{1}{\sqrt{d}}]]\\
 & =\E[\E[Y\mid X^{(1)}]X^{(1)}\One[|X^{(1)}|\geq\frac{1}{\sqrt{d}}]]\\
 & =\E[(1-2\Pr[Y=-1\mid X^{(1)}])X^{(1)}\One[X^{(1)}\geq\frac{1}{\sqrt{d}}]]\\
 & \geq(1-2\Pr_{\xi\sim N(0,1)}(\xi\geq\frac{1}{\sigma\sqrt{d}}))\E[X^{(1)}\One[X^{(1)}\geq\frac{1}{\sqrt{d}}]]
\end{align*}

Now, $\Pr_{\xi\sim N(0,1)}(\xi\geq\frac{1}{\sigma\sqrt{d}})\leq\max(\frac{1}{6},\frac{1}{2}-\frac{1}{3\sigma\sqrt{d}})$.
Besides, $\E[X^{(1)}\One[X^{(1)}\geq\frac{1}{\sqrt{d}}]]=\int_{\frac{1}{\sqrt{d}}}^{1}\frac{z(1-z^{2})^{\frac{d-3}{2}}}{B(\frac{d-1}{2},\frac{1}{2})}dz=\frac{(1-\frac{1}{d})^{\frac{d-1}{2}}}{(d-1)B(\frac{d-1}{2},\frac{1}{2})}\geq\frac{1}{2(d-1)\frac{\pi}{\sqrt{d-1}}}\geq\frac{1}{2\pi\sqrt{d}}$
where the first inequality follows by Fact~\ref{fact:beta} and \ref{fact:exp}.
Thus, $\E[YX^{(1)}]\geq\frac{1}{3\pi\sqrt{d}}\min(1,\frac{1}{\sigma\sqrt{d}})$. 

By the Chernoff bound, with probability at least $1-\delta$, $v^{(1)}=\frac{1}{m}\sum Z_{i}^{(1)}\geq\frac{1}{3\pi\sqrt{d}}\min(1,\frac{1}{\sigma\sqrt{d}})-\sqrt{\frac{1}{2m}\log\frac{1}{\delta}}$.
\end{proof}
\begin{lem}
\label{lem:ub-1}With probability at least $1-\delta$, $v^{(1)}\leq\frac{2}{\sqrt{2\pi}d\sigma}+\sqrt{\frac{1}{2m}\log\frac{1}{\delta}}$.
\end{lem}

\begin{proof}
We first give an upper bound of $\E[YX^{(1)}]=2\int_{0}^{1}z\frac{(1-z^{2})^{\frac{d-3}{2}}}{B(\frac{d-1}{2},\frac{1}{2})}(1-2\Pr_{\xi\sim N(0,1)}(\xi\geq\frac{z}{\sigma}))dz$.
By Fact~\ref{fact:gaussian}, $1-2\Pr_{\xi\sim N(0,1)}(\xi\geq\frac{z}{\sigma})\leq\frac{2z}{\sqrt{2\pi}\sigma}$,
so we have

\begin{align*}
\E[YX^{(1)}] & \leq2\int_{0}^{1}z\frac{(1-z^{2})^{\frac{d-3}{2}}}{B(\frac{d-1}{2},\frac{1}{2})}\frac{2z}{\sqrt{2\pi}\sigma}dz\\
 & =\frac{4}{\sqrt{2\pi}\sigma B(\frac{d-1}{2},\frac{1}{2})}\int_{0}^{1}z^{2}(1-z^{2})^{\frac{d-3}{2}}dz\\
 & =\frac{2B(\frac{d-1}{2},\frac{3}{2})}{\sqrt{2\pi}\sigma B(\frac{d-1}{2},\frac{1}{2})}\\
 & =\frac{2}{\sqrt{2\pi}d\sigma}
\end{align*}

The conclusion follows by the Chernoff bound.
\end{proof}
Now we present the proofs for the propositions.
\begin{proof}
(of Proposition~\ref{prop:complexity}) Since $\left\Vert \hat{w}-\mu\right\Vert ^{2}=\left\Vert \hat{w}\right\Vert ^{2}+\left\Vert \mu\right\Vert ^{2}-2\mu^{\top}\hat{w}=2(1-\mu^{\top}\hat{w})=2(1-\frac{v^{(1)}}{\left\Vert v\right\Vert })$,
to prove $\left\Vert \hat{w}-\mu\right\Vert \leq\varepsilon$, it suffices
to show $\frac{v^{(1)}}{\left\Vert v\right\Vert }\geq1-\frac{\varepsilon^{2}}{2}$.
Note that $1-(\frac{v^{(1)}}{\left\Vert v\right\Vert })^{2}=\frac{1}{1+\frac{(v^{(1)})^{2}}{\sum_{k=2}^{d}(v^{(k)})^{2}}}$,
and $1-(1-\frac{\varepsilon^{2}}{2})^{2}=\varepsilon^{2}-\frac{\varepsilon^{4}}{4}\geq\frac{\varepsilon^{2}}{2}\geq\frac{1}{1+\frac{2}{\varepsilon^{2}}}$
, so it suffices to show $\frac{(v^{(1)})^{2}}{\sum_{k=2}^{d}(v^{(k)})^{2}}\geq\frac{2}{\varepsilon^{2}}$
with probability at least $1-\delta$.

Now, by Lemma~\ref{lem:ub-k} and \ref{lem:lb-1} and a union bound,
with probability at least $1-\delta$,
\[
\frac{(v^{(1)})^{2}}{\sum_{k=2}^{d}(v^{(k)})^{2}}\geq\frac{(\frac{1}{3\pi\sqrt{d}}\min(1,\frac{1}{\sigma\sqrt{d}})-\sqrt{\frac{1}{2m}\log\frac{d}{\delta}})^{2}}{4(d-1)\frac{2}{md}\log\frac{2d}{\delta}},
\]

which is at least $\frac{2}{\varepsilon^{2}}$ for our setting of $m=\frac{(15\pi)^{2}}{\varepsilon^{2}}d\max(1,d\hat{\sigma}^{2})\log\frac{2d}{\delta}$.
\end{proof}
\begin{proof}
(of Proposition~\ref{prop:length-lb}) By Lemma~\ref{lem:lb-1},
with probability at least $1-\delta$, $v^{(1)}=\frac{1}{m}\sum Z_{i}^{(1)}\geq\frac{1}{3\pi\sqrt{d}}\min(1,\frac{1}{\sigma\sqrt{d}})-\sqrt{\frac{1}{2m}\log\frac{1}{\delta}}\geq\frac{1}{12\sqrt{d}}\min(1,\frac{1}{\hat{\sigma}\sqrt{d}})$,
which implies $\left\Vert v\right\Vert \geq v^{(1)}\geq\frac{1}{12\sqrt{d}}\min(1,\frac{1}{\hat{\sigma}\sqrt{d}})$.
\end{proof}
\begin{proof}
(of Proposition~\ref{prop:length-ub}) By Lemma~\ref{lem:ub-k}
and \ref{lem:ub-1} and a union bound, with probability at least $1-\delta$,

\begin{align*}
\left\Vert v\right\Vert ^{2} & =\sum_{k=1}^{d}(v^{(k)})^{2}\\
 & \leq4(d-1)\frac{2}{md}\log\frac{2d}{\delta}+(\frac{2}{\sqrt{2\pi}d\sigma}+\sqrt{\frac{1}{2m}\log\frac{1}{\delta}})^{2}\\
 & \leq\frac{1}{148d^{2}\hat{\sigma}^{2}}.
\end{align*}
\end{proof}

%\newpage 

\subsection{Noisy Labels for the Continuous Case}
\label{subsec:continuous}

In the continuous, the model extraction problem becomes a regression
problem. In~\cite{HS14} the authors consider passive linear regression
with squared loss and provide an algorithm that achieves nearly
optimal convergence rate $\E [X^\top\hat{w}-Y]-\E[X^\top
  w^\star-Y]=\tilde{O} \left(\frac{C}{n}\right)$ where the constant
$C$ depends on the covariance matrix of $X$ and the error of the
optimal linear model $w^\star$. In~\cite{SM14} the authors point out
that unlike in the classification case, the $O(\frac{1}{n})$ cannot be
improved by active learning, but it provides an algorithm under a
stream-based querying model (in fact, they assume the algorithm can
draw $X$ from any distribution, which can be implemented by rejection
sampling with stream-based querying model) that achieves a learning
rate with a better constant factor $C$. Authors in~\cite{CKNS15}
consider active learning for maximum likelihood estimation (MLE) under the
assumption that the model is well-specified ($P(Y|X)$ is given by a
model in the model class) and that the Fisher information matrix does
not depend on label $y$ (this assumption holds for linear regression
and generalized linear models). It shows that a two-stage algorithm
achieves a nearly optimal convergence rate.

The main difficulties of computationally efficient active learning for
classification arise because of two factors: (1) how to efficiently
find a classifier with the minimum classification error rate; (2) how
to select examples for labeling. For (1), it has been shown that
optimizing the classification error rate (0-1 loss) with noise is hard
in general, and computational efficient solutions with theoretical
guarantees are only known under some assumptions of the hypothesis
space and noise conditions (for
example~\cite{chen2017near,key-3,key-5}). For (2), most existing
active learning algorithms maintain a candidate set of classifiers
either explicitly~\cite{H07} or
implicitly~\cite{chen2017near,DHM07,key-5}, and the noise tolerance is
achieved by repetitive querying as in Proposition~\ref{prop} or a
carefully designed sampling schedule to guarantee that the candidate
set is ``correctly shrunk'' with high
probability~\cite{DHM07,H07,key-5}. For regression, most loss
functions (for example the squared error, negative log likelihood) are
convex, and thus can be optimized efficiently. The labeling strategies
in regression are also different: instead of maintaining candidate
sets, active regression algorithms~\cite{SM14, CKNS15} often first
find a good sampling distribution that optimize some statistics of the
covariance matrix and then draw labeled samples from this
distribution. Such strategies tolerates noise naturally and de-noising
strategies like repetitive querying are not necessary.

\subsection{Additional Results}

\subsubsection{Alternate Stopping Criterion}
\label{app:alternate}

We investigated if measuring the model's stability over $N$ iterations results in acceptable extraction attacks. Here, we define model stability as the oscillation between the approximation learned at iteration $i$ and at iteration $i+1$. Formally, stability can be defined as $\mathcal{S}_i = ||w_i - w_{i+1}||_2$. Our approach checks if $\mathcal{S}_i \leq \tau$ for $i=1,\cdots,N$ and terminates execution if the condition is satisfied. We observe that this approach fails for the algorithm proposed by Chen \etal \cite{chen2017near}, as the approximation produced at each iteration differs greatly from the approximation produced in the preceding iteration. The results for the algorithm proposed by Alabdulmohsin \etal \cite{alabdulmohsin2015efficient} can be found Table \ref{alt}.

\begin{table}[H]
%\centering
\small\addtolength{\tabcolsep}{-5pt}
%\begin{tabular}{|p{1.7cm}|p{1cm}|p{1.4cm}|p{1.4cm}|p{1.4cm}|p{1cm}|}
\begin{tabular}{lccccccc}
\toprule
{} &  \underline{\bf \footnotesize N=10} & {}                 & \underline{\bf \footnotesize N=15} & {}                 & \underline{\bf \footnotesize N=20}      & {}           & {\bf  \footnotesize Baseline}\\ 
{\bf \footnotesize Dataset} 	      	          & {\bf \footnotesize Queries}        & {\ \  \footnotesize $\hat{\varepsilon}$} & {\bf  \footnotesize Queries}        & {\ \  \footnotesize $\hat{\varepsilon}$} & {\bf  \footnotesize Queries}        & {\ \  \footnotesize $\hat{\varepsilon}$} & { \footnotesize $\varepsilon=0.001$}\\ 
\midrule
% & 0.01 & 153 & 0.0307 & 158 & 0.0266 & 163 & 0.0219 & 210\\ 
{\bf \footnotesize Breast Cancer} & 241 & 0.0047 & 247 & 0.0034 & 252 & 0.0031 & 300\\ 
% & 0.0001 & 338 & 0.00322 & 343 & 0.00272 & 348 & 0.00268 & 400\\ \bottomrule

% & 0.01 & 78 & 0.0292 & 83 & 0.0185 & 88 & 0.0141 & 95\\ 
{\bf \footnotesize Adult Income} & 117 & 0.0019 & 122 & 0.0015 & 127 & 0.0012 & 135\\ 
% & 0.0001 & 159 & 0.00027 & 164 & 0.00015 & 169 & 0.00013 & 185\\ \bottomrule

% & 0.01 & 287 & 0.0662 & 296 & 0.0608 & 302 & 0.0609 & 470 \\ 
{\bf \footnotesize Digits} &  493 & 0.0077 & 498 & 0.0075 & 503 & 0.0073 & 700 \\ 
% & 0.0001 & NA & NA & NA & NA & NA & NA & 900 \\ \bottomrule

% & 0.01 & 79 & 0.0328 & 84 & 0.0147 & 89 & 0.0111 & 105\\ 
{\bf \footnotesize Wine} & 120 & 0.0016 & 125 & 0.0014 & 130 & 0.0012 & 135 \\ 
% & 0.0001 & 161 & 0.00013 & 166 & 0.00016 & 168 & NA& 170\\ 
\bottomrule
\end{tabular}
\compactcaption{\footnotesize Model stability results in nominal savings at the expense of a small increase in geometric error ($\hat{\varepsilon}$). The trends are the same for other values of $\varepsilon$.}
\label{alt}
\end{table}

\subsubsection{A Direction For Defense?}
\label{app:defense}

Recall from earlier discussion that QS active learning algorithms are capable of generating points de novo. It is conceivable that these points are not generated from the distribution from which the training data is sampled from. To this end, we verified if these distributions are indeed different using the Hotelling's $T^2$ test, specifically for the algorithms proposed by Alabdulmohsin \etal \cite{alabdulmohsin2015efficient} and Chen \etal \cite{chen2017near} under the null hypothesis that the distributions are the same (refer Table \ref{app:tab-def} and Table \ref{app:tab-def2}). We observe that this QS active learning algorithm indeed produces points that are not from the underlying natural distribution. While discarding points that can not be sampled from the training distribution may seem as a tempting defense strategy, it is conceivable that certain real world tasks may query MLaaS providers with outlier points. Further analysis is required to determine how this strategy may effectively be used to defend against model extraction.

\begin{figure}[H]
	\centering
	\subfigure[Version Space Approximation]{\label{}
	\includegraphics[width=0.45\linewidth]{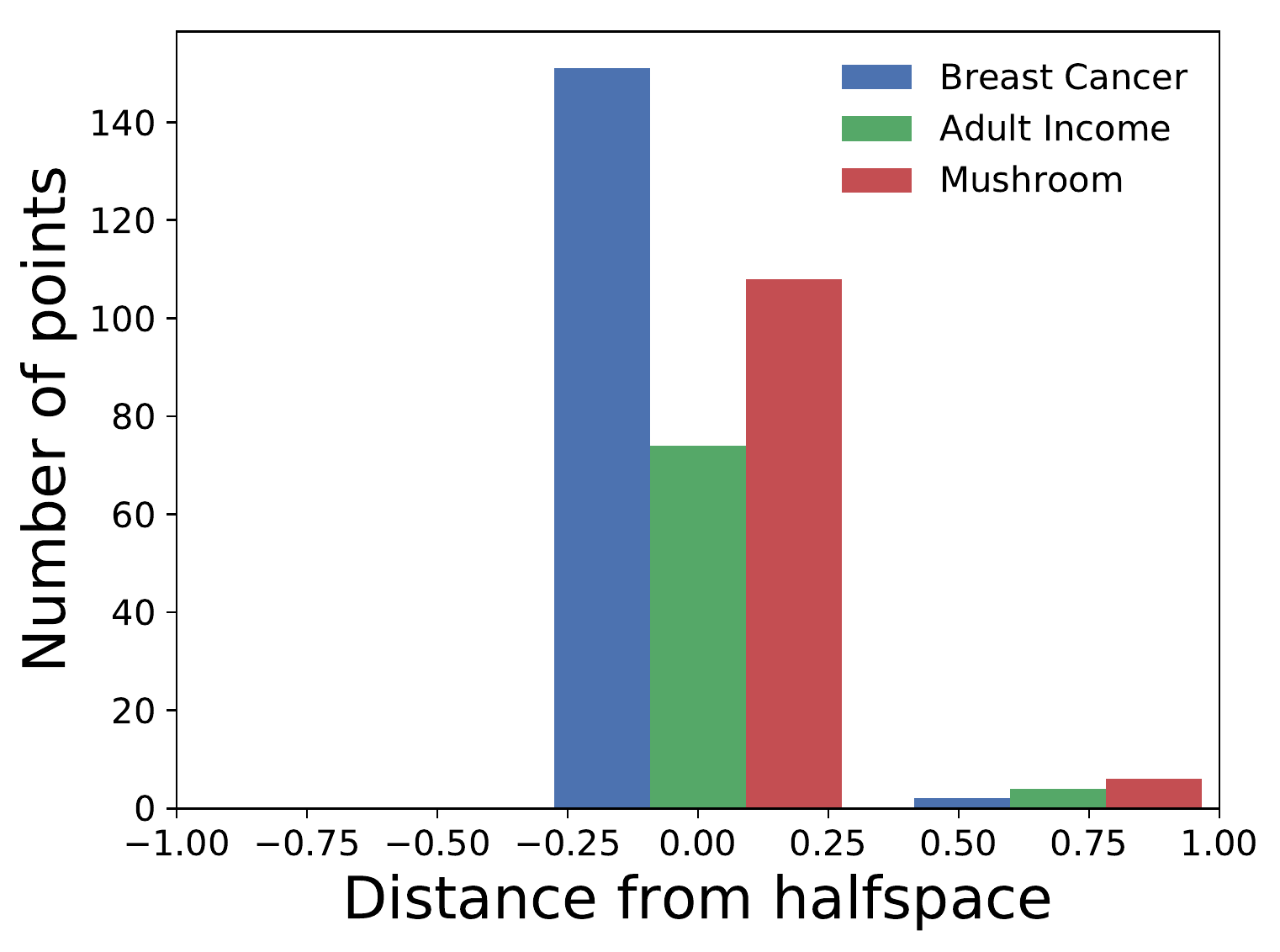}}
	\subfigure[$DC^2$]{\label{}
	\includegraphics[width=0.45\linewidth]{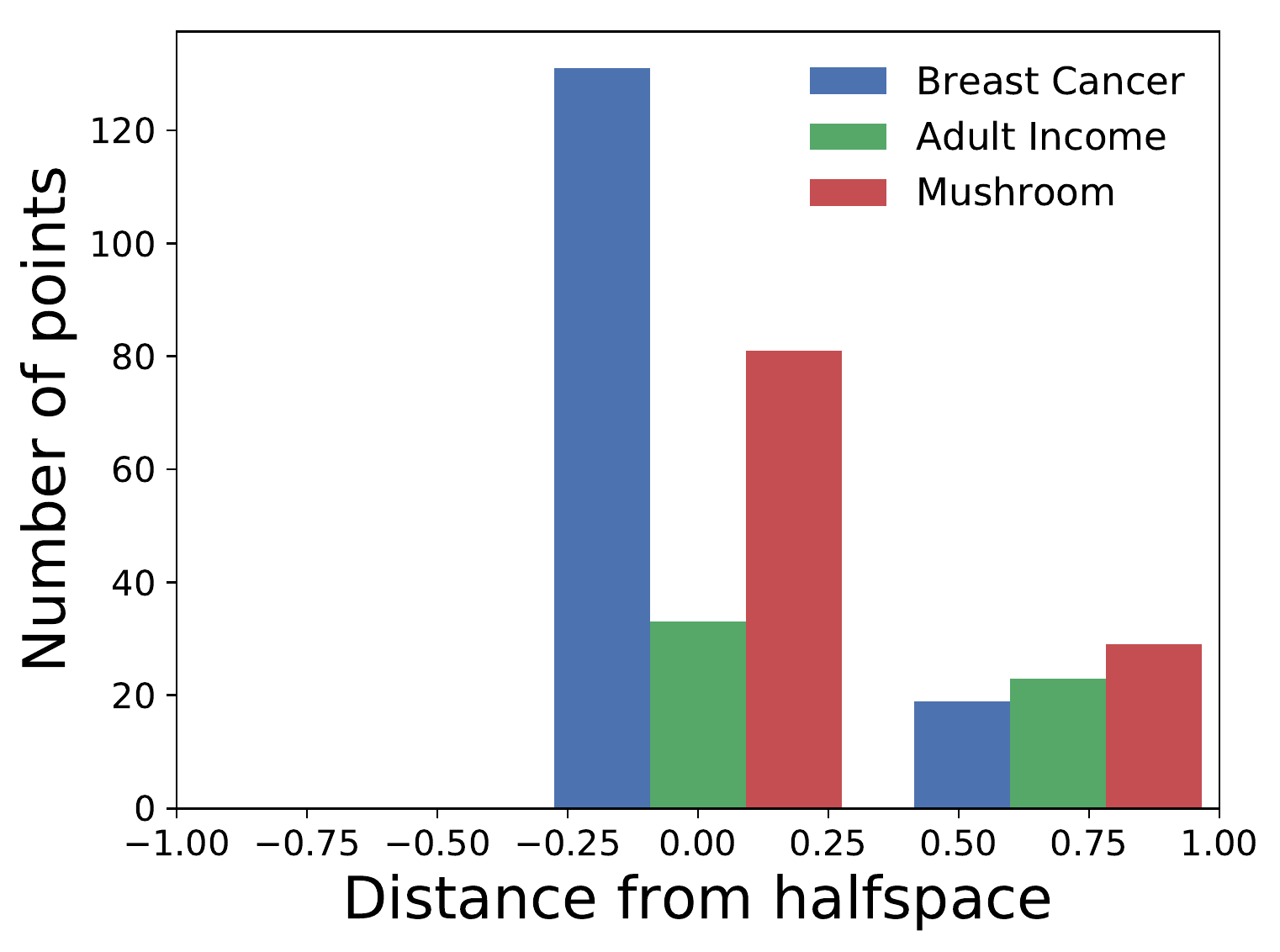}}
	\compactcaption{\footnotesize Distance of the instances synthesized by (a) version space approximation algorithm, and (b) dimension coupling algorithm from optimal halfspace.}
	\label{fig:distance}
\end{figure}

\begin{table}[H]
\small\addtolength{\tabcolsep}{-1pt}
\begin{tabular}{lcccc}
\toprule
{\bf Dataset}	  &{\bf t-value} & {\bf n-1} & {\bf p-value} & {\bf Reject Null?} \\
\midrule
		  {Breast Cancer}  & 14.24 & 198 & 3.70 $\times$ 10 $^{-32}$ & \cmark \\
		  {Adult Income} & 9.71 & 92 & 9.22 $\times$ 10 $^{-16}$ & \cmark\\
		  {Mushroom} & 22.16 & 599 & 5.92 $\times$ 10 $^{-80}$ & \cmark \\
		  \bottomrule
\end{tabular}
\compactcaption{\footnotesize Results of the Hotelling $T^2$ test for multivariate distributions, for n samples. It is observed that the data-points generated by the {\bf version space approximation} algorithm do not lie in the natural distribution underlined by samples from the training data.}
\label{app:tab-def} 
\end{table}

\begin{table}[]
\small\addtolength{\tabcolsep}{-1pt}
\begin{tabular}{lcccc}
\toprule
{\bf Dataset}	  &{\bf t-value} & {\bf n-1} & {\bf p-value} & {\bf Reject Null?} \\
\midrule
		  {Breast Cancer}  & 319.27 & 322 &  0 & \cmark \\
		  {Adult Income} & 467.43 & 133 & 9.64 $\times$ 10 $^{-216}$ & \cmark\\
		  {Mushroom} & 65.74 & 222 & 1.58 $\times$ 10 $^{-147}$ & \cmark \\
		  \bottomrule
\end{tabular}
\compactcaption{\footnotesize Results of the Hotelling $T^2$ test for multivariate distributions, for n samples. It is observed that the data-points generated by the {\bf DC$^2$ algorithm} do not lie in the natural distribution underlined by samples from the training data.}
\label{app:tab-def2} 
\end{table}

\end{document}